\documentclass[twoside,11pt]{article}

\usepackage{jmlr2e}

\usepackage{xcolor}
\usepackage[normalem]{ulem}
\usepackage{amsmath}
\usepackage{mathtools}
\usepackage{bm}
\usepackage{booktabs}
\usepackage{adjustbox}
\usepackage{threeparttable}
\usepackage{ragged2e}
\usepackage{floatrow}
\usepackage{tabularx}
\usepackage{enumitem}
\usepackage{lastpage}
\usepackage{float}

\hypersetup{
  colorlinks=true,
  citecolor=blue,
  linkcolor=black,
  urlcolor=blue
}

\usepackage{algorithm}
\usepackage{algpseudocode}

\usepackage{tikz}
\usetikzlibrary{
  shapes.geometric,
  arrows.meta,
  positioning,
  calc,
  patterns,
  fit,
  backgrounds,
  decorations.pathreplacing
}

\newcolumntype{Y}{>{\raggedright\arraybackslash}X}

\jmlrheading{00}{2026+}{1--\pageref{LastPage}}{0/00; Revised 0/00}{0/00}{0-0000}{Se Yoon Lee, Yonghyun Kwon, and Jae Kwang Kim}

\ShortHeadings{MEC-Cox for ATT Hazard Ratios}{Lee et al.}
\firstpageno{1}

\begin{document}

\title{MEC-Cox: Machine-Learning-Assisted Generalized Entropy Calibration for ATT Marginal Hazard-Ratio Estimation}


\author{\name Se Yoon Lee\email seyoonlee.stat.math@gmail.com \\
\addr Department of Statistics \\
Texas A\&M University \\
College Station, TX, USA
\AND
\name Yonghyun Kwon \email yhkwon@kma.ac.kr \\
\addr Department of Mathematics \\
Korea Military Academy \\
Seoul, Republic of Korea
\AND
\name Jae Kwang Kim \email jkim@iastate.edu \\
\addr Department of Statistics \\
Iowa State University \\
Ames, IA, USA
}

\date{}



\editor{My Editor}

\maketitle

\begin{abstract}
Externally controlled survival trials are increasingly used when concurrent randomized controls are infeasible, particularly in oncology and rare-disease settings with time-to-event endpoints. We target an average-treatment-effect-on-the-treated (ATT)-type marginal hazard-ratio estimand, comparing treatment with counterfactual control in the treated trial population, and estimate it using inverse-probability-weighted (IPW) Cox regression. Valid inference is challenging because IPW Cox regression depends on the weights through both event contributions and risk-set averages, making flexible machine-learning nuisance estimation difficult to incorporate directly. Building on machine-learning-assisted generalized entropy calibration (MEC) by Lee and Kim (2026), we propose MEC-Cox for ATT-weighted IPW Cox regression. The method begins with normalized source-propensity-score odds weights for external controls and then applies Bregman calibration to balance cross-fitted prognostic summaries between external controls and treated trial patients. The calibration basis may include control-survival predictions, Cox linear predictors, penalized-survival-model predictions, or other prognostic-score summaries. MEC-updated weights therefore play a dual role as source-transport and prognostic-score balancing weights. We establish consistency, characterize a calibration-induced efficiency gain, and develop a stacked sandwich variance estimator. Simulations show that MEC-Cox can reduce bias, increase efficiency, and improve coverage through flexible machine-learning-assisted adjustment.
\end{abstract}


\section{Introduction}
Externally controlled single-arm trials, or more broadly externally controlled
comparisons, in which outcomes from patients receiving an investigational
treatment are compared with outcomes from an external-control cohort, are
increasingly used when concurrent randomized controls are difficult to obtain
because of ethical, logistical, or disease-specific constraints
\citep{russek2025supplementing,lambert2023enriching}. In these studies, the
external-control cohort is typically constructed from historical trials,
registries, or real-world data. A central inferential challenge is that the
external-control cohort may differ systematically from the treated trial cohort
in baseline covariates and other outcome-relevant features. Valid inference
therefore requires constructing external-control weights that transport the
external cohort to the treated trial population.

In this paper, we focus on externally controlled survival studies with
time-to-event endpoints. Our target estimand is an average treatment effect on the treated (ATT)-type marginal hazard ratio for time-to-event outcomes, comparing the marginal hazard under treatment with the counterfactual marginal hazard under control in the treated trial population. Marginal structural models provide a natural framework for defining this population-level causal hazard-ratio target \citep{hernan2000marginal,robins2000marginal,cole2008constructing}. In externally controlled single-arm trials, the corresponding ATT pseudo-population is constructed by reweighting external controls with the source propensity-score odds, so that their weighted covariate distribution represents that of the treated trial cohort. The marginal hazard-ratio parameter is then estimated by solving an inverse-probability-weighted (IPW) Cox estimating equation in this weighted sample. This connects externally controlled single-arm trials with weighted Cox regression \citep{cox1975partial,binder1992fitting,lin1989robust} and causal marginal hazard-ratio estimation \citep{hernan2000marginal}.

However, the validity of the IPW Cox estimator depends critically on consistent estimation of the source propensity score; misspecification can lead to severe bias in the weighted Cox estimator \citep{shu2021estimating}. In addition, inference is technically challenging because the weighted Cox partial likelihood depends on the weights through both event contributions and risk-set averages. As a result, naive model-based variance estimators can be biased by ignoring the weighting structure \citep{shu2021variance}. The Lin--Wei robust variance estimator \citep{lin1989robust}, extended to survey-weighted Cox regression by \citet{binder1992fitting}, is convenient and widely used, but treats the estimated weights as fixed and therefore ignores propensity-score estimation uncertainty. Alternative approaches, including bootstrap variance estimation \citep{austin2016variance}, closed-form variance estimators \citep{hajage2018closed}, and analytical corrections for propensity-score estimation \citep{mao2018propensity,shu2021variance}, have been proposed. However, these methods are often developed for parametric propensity-score models and can therefore be sensitive to model misspecification. This limitation motivates methods that can incorporate flexible machine-learning (ML)-based adjustment, such as Bayesian additive regression trees (BART) \citep{chipman2010bart} and deep learning (DL) \citep{lecun2015deep}, for estimating this causal hazard-ratio target.

While recent developments in targeted maximum likelihood estimation (TMLE)
\citep{van2011targeted,van2018targeted} and double machine learning (DML)
\citep{chernozhukov2018double} have shown how flexible ML-based nuisance
estimation can be used for debiased estimation and valid post-ML inference,
especially for mean-type causal estimands, extending these ideas to causal
hazard-ratio estimation based on IPW Cox regression is not straightforward. A
main reason is that these approaches typically require researchers to derive an
efficient influence function (EIF) \citep{tsiatis2006semiparametric,kennedy2016semiparametric,kennedy2024semiparametric}, or a Neyman-orthogonal score \citep{mackey2018orthogonal,foster2023orthogonal,chernozhukov2017double}, for the problem. For mean-type causal estimands, many such EIF-based
or orthogonal-score procedures have been developed
\citep{van2006targeted,gruber2010targeted,semenova2021debiased}, often yielding desirable properties
such as double robustness \citep{bang2005doubly}. In contrast, analogous EIF-based or
orthogonal-score constructions for Cox hazard-ratio estimands are less straightforward. Indeed, targeted-learning methods for
time-to-event outcomes have more commonly focused on survival probabilities,
survival curves, restricted mean survival time, or time-specific intervention
effects \citep{cai2020one,rytgaard2023estimation}, rather than causal
hazard-ratio estimands.

Recently, \citet{lee2026mec} proposed machine-learning-assisted generalized entropy
calibration (MEC) as an alternative approach for leveraging the predictive power
of flexible ML methods in weighted estimating equations. MEC builds on
generalized entropy calibration (GEC) \citep{kwon2025debiased} and, more
broadly, on recent Bregman calibration frameworks \citep{kim2026bregman}. MEC is
a weight-calibration framework that incorporates ML prediction functions into
calibration constraints and can be broadly applied to well-defined weighted
estimating equations; see the Supplemental Materials of \citet{lee2026mec} for
several examples. Unlike TMLE or DML, MEC does not require researchers to derive an
EIF at the implementation stage, although semiparametric efficiency theory can
still be used to study its asymptotic properties, as in \citet{lee2026mec}. This
gives researchers a broadly applicable way to incorporate ML into problems where
the target parameter can be estimated through a weighted estimating equation,
including semi-supervised inference problem \citep{zhang2019semi,angelopoulos2023prediction}, missing-data
problems \citep{robins1995analysis}, and many other settings. IPW Cox regression
is one such example.


In this paper, we apply the MEC framework to IPW Cox regression to estimate the ATT marginal hazard-ratio and refer to the resulting method as MEC-Cox. MEC-Cox starts from normalized ATT transport
weights for external controls and applies a Bregman calibration step to achieve
exact finite-sample balance of cross-fitted prognostic features between the
external-control and treated trial cohorts. The source propensity score
determines transport to the ATT target population, whereas the prognostic
calibration basis targets outcome-relevant imbalance that remains after baseline
transport weighting. This basis can be constructed from control-survival predictions, Cox linear
predictors, penalized survival models, or other fixed-dimensional cross-fitted
prognostic summaries, allowing ML to enter through both source transport and
prognostic adjustment. In this sense, MEC-Cox extends the exact-balancing
perspective in causal inference
\citep{benmichael2021balancing,imai2014covariate,athey2018approximate,hainmueller2012entropy}
to IPW Cox regression for ATT marginal hazard-ratio estimation, but with
balance imposed on survival-prognostic summaries rather than only on raw
baseline covariates. By incorporating prognostic information directly into the
weighting scheme, MEC-Cox can improve efficiency relative to existing IPW Cox
methods
\citep{lin1989robust,binder1992fitting,austin2016variance,shu2021variance}.

The remainder of the paper is organized as follows.
Section~\ref{sec:Estimation of marginal hazard ratios} introduces the data
structure, causal assumptions, ATT marginal hazard-ratio estimand, and baseline
ATT-weighted Cox estimator. Section~\ref{sec:mec_cox} presents the MEC-Cox
estimator, and Section~\ref{sec:variance_mec_cox} develops its stacked sandwich
variance estimator. Section~\ref{sec:simulation_studies} reports simulation
studies, Section~\ref{sec:real_data} presents a real-world application of the MEC-Cox,
and Section~\ref{sec:Discussion} concludes.

\section{Estimation of Marginal Hazard Ratios in the Treated Trial Population}\label{sec:Estimation of marginal hazard ratios}

\subsection{Externally controlled survival setting and causal estimand}\label{subsec:Externally controlled survival setting and causal estimand}
\paragraph{Externally controlled comparison setting.}
External control data can be incorporated into clinical trials in two primary
settings \citep{FDA2023ExternalControl}. In hybrid randomized controlled trial
designs, shown in Panel~(a) of Figure~\ref{fig:trial_designs}, a concurrent
internal control arm is available, and external controls are used to supplement
the randomized comparison. In this setting, the internal control arm serves as an
anchor, enabling empirical assessment of the compatibility between randomized
control patients and external-control patients
\citep{lee2025power,gao2025improving}.

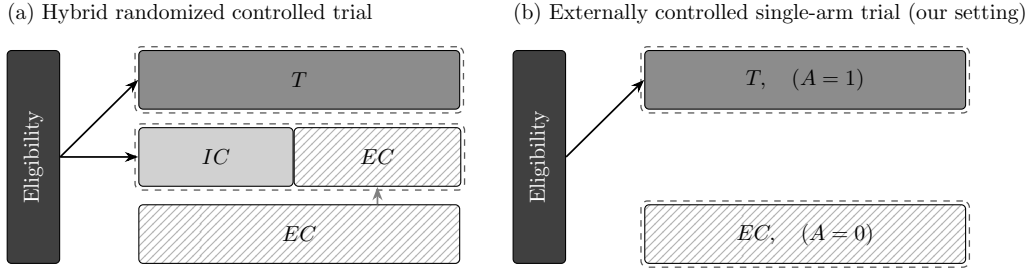
\begin{figure}[h!]
\centering
\resizebox{0.9\linewidth}{!}{
\begin{tikzpicture}[
    font=\small,
    >={Stealth[length=2.2mm]},
    box/.style={
        rounded corners=2pt,
        draw=black,
        line width=0.5pt,
        minimum height=0.95cm,
        inner sep=0pt
    },
    comparison/.style={
        draw=black!65,
        dashed,
        line width=0.6pt,
        rounded corners=3pt,
        inner sep=1.5pt
    }
]

\definecolor{ScreenGray}{gray}{0.25}
\definecolor{TreatGray}{gray}{0.55}
\definecolor{ControlGray}{gray}{0.82}
\definecolor{ECGray}{gray}{0.95}
\definecolor{TextLight}{gray}{1.0}


\node[box, fill=TreatGray, minimum width=5.2cm, anchor=west] (E1) at (1.7,1.45) {};
\node[text=black] at (E1.center) {$T$};

\node[box, fill=ControlGray, minimum width=2.5cm, anchor=west] (C1a) at (1.7,0.20) {};
\node[text=black] at (C1a.center) {$IC$};

\node[box, fill=ECGray, minimum width=2.7cm, anchor=west,
      pattern=north east lines, pattern color=gray!60] 
(C1b) at ($(C1a.east)+(0,0)$) {};
\node[text=black] at (C1b.center) {$EC$};

\node[box, fill=ECGray, minimum width=5.2cm, anchor=west,
      pattern=north east lines, pattern color=gray!60] 
(ECpool) at (1.7,-1.05) {};
\node[text=black] at (ECpool.center) {$EC$};

\node[box, fill=ScreenGray, minimum width=0.85cm, minimum height=3.45cm] (S1) at (0,0.2) {};
\node[rotate=90, text=TextLight] at (S1.center) {Eligibility};

\node[anchor=south west, inner sep=0pt, yshift=0.4cm] at (S1.north west)
{(a) Hybrid randomized controlled trial};

\begin{scope}[on background layer]
\node[comparison, fit=(E1)] {};
\node[comparison, fit=(C1a)(C1b)] {};
\end{scope}

\draw[->, line width=0.8pt] (S1.east) -- (E1.west);
\draw[->, line width=0.8pt] (S1.east) -- (C1a.west);
\draw[->, line width=0.8pt, gray] (C1b.south |- ECpool.north) -- (C1b.south);


\node[box, fill=TreatGray, minimum width=5.2cm, anchor=west] (E2) at (9.9,1.45) {};
\node[text=black] at (E2.center) {$T, \quad(A=1)$};

\node[box, fill=ECGray, minimum width=5.2cm, anchor=west,
      pattern=north east lines, pattern color=gray!60] 
(EC2) at (9.9,-1.05) {};
\node[text=black] at (EC2.center) {$EC, \quad(A=0)$};

\node[box, fill=ScreenGray, minimum width=0.85cm, minimum height=3.45cm] (S2) at (8.2,0.2) {};
\node[rotate=90, text=TextLight] at (S2.center) {Eligibility};

\node[anchor=south west, inner sep=0pt, yshift=0.4cm] at (S2.north west)
{(b) Externally controlled single-arm trial (our setting)};

\begin{scope}[on background layer]
\node[comparison, fit=(E2)] {};
\node[comparison, fit=(EC2)] {};
\end{scope}

\draw[->, line width=0.8pt] (S2.east) -- (E2.west);

\end{tikzpicture}
}
\vspace{-0.4em}
\caption{Comparison of study designs incorporating external control data. Thin dashed outlines indicate the arms included in the treatment comparison, and darker shaded boxes indicate the target comparison population. Here, $T$ denotes the experimental treatment or target cohort, $IC$ denotes the internal randomized control arm, and $EC$ denotes the external-control cohort.}
\label{fig:trial_designs}
\end{figure}

In this paper, we focus on externally controlled single-arm trials, or more
broadly externally controlled comparisons with no concurrent control arm, shown
in Panel~(b) of Figure~\ref{fig:trial_designs}. In this setting, no internal
control group is available, and the counterfactual control outcomes for the
treated or target cohort must be inferred from an external population. Although we use the term externally controlled single-arm trial for concision, the target cohort need not be a prospectively conducted trial cohort. What is essential is that the estimand is defined with respect to the target population, denoted by \(T\) in Figure~\ref{fig:trial_designs}.



\paragraph{Observed data structure.}
Consider an externally controlled single-arm trial in which observations are
pooled from two data sources: a treated single-arm trial cohort and an
external-control cohort. For each individual \(i=1,\ldots,n\), let \(X_i\)
denote an $M$-dimensional vector of measured baseline covariates, and let
\(A_i\in\{0,1\}\) denote the source and treatment indicator, with
\(A_i=1\) for a trial patient receiving the investigational treatment and
\(A_i=0\) for an external-control patient receiving control. This setup corresponds to the right panel of Figure~\ref{fig:trial_designs}; in this setting, source membership and treatment status coincide by design. Let \(T_i\)
denote the event time under the treatment actually received, and let \(C_i\)
denote the censoring time. The observed time and event indicator are $Y_i=\min(T_i,C_i)$ and $\delta_i=I(T_i\le C_i),$ where \(I(\cdot)\) is the indicator function. Equivalently, define the
counting process \(\mathcal{N}_i(t)=I(Y_i\le t,\delta_i=1)\) and the at-risk process
\(\mathcal{Y}_i(t)=I(Y_i\ge t)\). The observed data are
\(O_i=(X_i,A_i,Y_i,\delta_i)\), \(i=1,\ldots,n\), which are assumed to be
independent. We assume that \(C_i\) is independent of \((T_i,X_i)\) conditional on \(A_i\).

\paragraph{Causal assumptions.}
For \(a=0,1\), let \(T_i^a\) denote the potential event time that would be observed under treatment regime \(a\), where \(a=1\) corresponds to the investigational treatment and \(a=0\) corresponds to control. We make the following assumptions:
\begin{enumerate}[leftmargin=1.5em, label={}, itemsep=0.5em]
    \item \textbf{\textit{A1. Consistency.}} The observed event time equals the potential event time under the treatment actually received: $$T_i=A_iT_i^1+(1-A_i)T_i^0.$$

    \item \textbf{\textit{A2. Survival transportability.}} Conditional on measured baseline covariates, the control potential-outcome distribution for trial patients is the same as that for external-control patients: $$T_i^0\mid X_i,A_i=1
    \overset{d}{=}
    T_i^0\mid X_i,A_i=0,$$ where \(\overset{d}{=}\) denotes equality in distribution.

    \item \textbf{\textit{A3. Positivity.}} There exists a
constant \(\epsilon>0\) such that the probability of being an
external control is bounded away from zero: $$1-\pi(x)\geq  \epsilon \quad \text{for all } x\in \mathrm{supp}(X_i\mid A_i=1),$$ where \(\pi(X_i)=\Pr(A_i=1\mid X_i)\) denotes the source propensity score.
\end{enumerate}
Under these assumptions, the external-control cohort can be reweighted to represent the counterfactual control experience of the treated trial population.

\paragraph{Parameter of interest.}
We aim to estimate the ATT marginal log-hazard ratio \(\theta_{ATT}\), which
compares the investigational treatment with external control in the treated trial target population. Specifically, \(\theta_{ATT}\) is defined through the marginal
proportional hazards model \citep{hernan2000marginal,hernan2001marginal,fay2024causal}
\begin{align}
\label{eq:marginal_structural_model_for_HR_ATT}
\lambda_1^a(t)=\lambda_1^0(t)\exp(\theta_{ATT}a),
\qquad a=0,1.
\end{align}
Here, for \(a=0,1\), $S_1^a(t)=\Pr(T_i^a>t\mid A_i=1)$ denotes the marginal survival function under treatment regime \(a\) in the
treated trial target population, and
\[
\lambda_1^a(t)
=
-\frac{\partial}{\partial t}\log S_1^a(t)
=
\lim_{\Delta t\downarrow 0}
\frac{
\Pr(t\le T_i^a<t+\Delta t \mid T_i^a\ge t, A_i=1)
}{\Delta t}
\]
denotes the corresponding marginal hazard function. Under
\eqref{eq:marginal_structural_model_for_HR_ATT}, \(\theta_{ATT}\) is the
marginal log-hazard ratio comparing treatment versus control in the treated
trial population, and the corresponding ATT marginal hazard ratio is
\(HR_{ATT}=\exp(\theta_{ATT})\).

\subsection{Identification by the ATT-weighted Cox score}\label{subsec:Identification by the ATT-weighted Cox score}
The following theorem states that, under the stated causal assumptions, the ATT-weighted Cox estimating equation is centered at the marginal ATT log-hazard ratio $\theta_{ATT}$ in the treated trial target population. Proofs of Theorem~\ref{thm:identification_weighted_cox} and the other theorems in the main text are provided in the Appendix. 

\begin{theorem}\label{thm:identification_weighted_cox}
Suppose Assumptions A1--A3 hold, and suppose that standard regularity
conditions for the weighted Cox estimating equation hold
\citep{andersen1982cox,andersen2012statistical}. We assume the
source-specific independent censoring assumption $$C_i \perp (T_i^0,T_i^1,X_i)\mid A_i.$$ Define the ATT-weighted Cox score
\begin{align}
\label{eq:weighted_Cox_regression_counting_process_form}
U_n^\omega(\theta)
=
\sum_{i=1}^n
\int
\omega_i\{A_i-\overline A_\omega(t;\theta)\}\,d\mathcal{N}_i(t),
\quad
\overline A_\omega(t;\theta)
=
\frac{
\sum_{i=1}^n \omega_i\mathcal{Y}_i(t)\exp(\theta A_i)A_i
}{
\sum_{i=1}^n \omega_i\mathcal{Y}_i(t)\exp(\theta A_i)
},
\end{align}
where \(\omega_i=A_i+(1-A_i)q(X_i)\),
\(q(X_i)=\pi(X_i)/\{1-\pi(X_i)\}\), and
\(\pi(X_i)=\Pr(A_i=1\mid X_i)\). If the marginal proportional hazards model
\eqref{eq:marginal_structural_model_for_HR_ATT} holds, then the population
limit \(U(\theta)\) of \(n^{-1}U_n^\omega(\theta)\) satisfies
\(U(\theta_{ATT})=0\). If, additionally, \(U(\theta)\) has a unique root and
\(U_n^\omega(\theta)\) converges uniformly to \(U(\theta)\), then any solution
\(\widehat\theta\) to \(U_n^\omega(\widehat\theta)=0\) satisfies
\(\widehat\theta \overset{p}{\longrightarrow} \theta_{ATT}\).
\end{theorem}

This result provides the population-level justification for using ATT odds weights in the Cox partial-likelihood score \citep{cox1975partial,lee2026asymptotic}. In the following paragraph, we describe the corresponding sample estimator obtained by replacing the unknown source propensity score \(\pi(X_i)\) with an estimate \(\widehat\pi(X_i)\).

\subsection{Existing ATT-IPW Cox estimator and variance estimation methods}\label{subsec:Existing ATT-IPW Cox estimator and variance estimation methods}
\paragraph{IPW estimator of \(\theta_{ATT}\).}
In practice, the source propensity score \(\pi(X_i)=\Pr(A_i=1\mid X_i)\) is unknown and must be estimated. Let \(\widehat\pi(X_i)\) denote an estimate of \(\pi(X_i)\), obtained, for example, from a logistic regression model for \(A_i\) given \(X_i\). We define the estimated ATT odds weight for external controls by $\widehat q(X_i)
=\widehat\pi(X_i)/(1-\widehat\pi(X_i)),$ and the corresponding estimated ATT estimating-equation weight by $\widehat\omega_i
=
A_i+(1-A_i)\widehat q(X_i).$ 


Let $\mathcal R_i=\{\ell:Y_\ell\ge Y_i\}$ denote the risk set just before the observed event time \(Y_i\). The IPW estimator \(\widehat\theta_{\mathrm{IPW}}\) is defined as the solution to the weighted Cox partial-likelihood score equation \citep{lin1989robust,cox1975partial,binder1992fitting}
\begin{align}
\label{eq:ipw_score_equation}
U_n^{\mathrm{IPW}}(\theta)
=
\sum_{i=1}^n
\widehat\omega_i\delta_i
\left[
A_i
-
\frac{
\sum_{\ell\in\mathcal R_i}
\widehat\omega_\ell \exp(\theta A_\ell)A_\ell
}{
\sum_{\ell\in\mathcal R_i}
\widehat\omega_\ell \exp(\theta A_\ell)
}
\right]
=
0.    
\end{align}

The resulting estimator \(\widehat\theta_{\mathrm{IPW}}\) estimates the ATT
marginal log-hazard ratio \(\theta_{ATT}\), and the corresponding ATT marginal
hazard-ratio estimator is $\widehat{HR}_{ATT}
=
\exp(\widehat\theta_{\mathrm{IPW}}).$ The consistency of \(\widehat\theta_{\mathrm{IPW}}\) for \(\theta_{ATT}\)
requires consistent estimation of the source propensity score \(\pi(X_i)\), together with the conditions
stated in Theorem~\ref{thm:identification_weighted_cox} \citep{hernan2000marginal,hernan2001marginal}.


\paragraph{Variance estimation for the IPW Cox estimator.}
A Wald-type \(95\%\) confidence interval for the ATT marginal log-hazard ratio $\theta_{ATT}$ can be constructed as \(\widehat\theta_{\mathrm{IPW}}\pm z_{\alpha/2}\,\widehat{\mathrm{se}}(\widehat\theta_{\mathrm{IPW}})\), where \(\alpha=0.05\), \(z_{\alpha/2}=\Phi^{-1}(1-\alpha/2)=\Phi^{-1}(0.975)=1.96\), \(\Phi\) denotes the standard normal cumulative distribution function, and \(\widehat{\mathrm{se}}(\widehat\theta_{\mathrm{IPW}})\) denotes the estimated standard error of \(\widehat\theta_{\mathrm{IPW}}\). The corresponding confidence interval for \(HR_{ATT}\) is obtained by exponentiating the endpoints.

In this paper, we focus on the following three representative approaches as the main comparators for our proposed method:

\begin{enumerate}[leftmargin=1.5em, itemsep=0.5em] 
    \item \textbf{\textit{Naive likelihood-based variance estimator.}} The naive likelihood-based variance estimator treats the estimated weights as fixed and regards the weighted pseudo-population as if it were an ordinary independent sample. This ignores the weighting structure and the uncertainty associated with weight estimation, and is therefore generally biased in weighted Cox regression. See Subsection~3.1.1 of \citet{shu2021variance} for details.
    \item \textbf{\textit{Robust sandwich estimator.}} \citet{lin1989robust} proposed a robust
sandwich variance estimator for the Cox model, and \citet{binder1992fitting} extended this robust variance idea to weighted Cox regression for survey data. See Subsection 3.1.2 of \citet{shu2021variance} for a summary. The
Lin--Wei/Binder (hereafter ``LW") robust sandwich variance estimator treats the estimated weights
\(\widehat\omega_i\) as fixed constants in the weighted Cox score. In particular,
if \(\widehat\theta_{\mathrm{IPW}}\) solves the score equation
\eqref{eq:ipw_score_equation}, then the robust sandwich variance is constructed
from the empirical Cox score contributions computed conditional on the weights.
This approach is simple and widely implemented by practitioners, for example
through \texttt{coxph(..., weights = ..., robust = TRUE)} in the
\texttt{survival} package. However, its key limitation in the present setting is
that it ignores the uncertainty from estimating the source propensity score
\(\pi(X_i)\), because \(\widehat\omega_i\) is treated as known. 

\item \textbf{\textit{Corrected sandwich estimator.}} \citet{shu2021variance} addressed this limitation by treating the weights as estimated quantities. Their corrected sandwich estimator stacks the weighted Cox estimating equation with the propensity-score estimating equation. Specifically, they assume a logistic propensity-score model, $\pi_\gamma(X_i)
=
\Pr(A_i=1\mid X_i;\gamma)
=
1/\{1+\exp(-\gamma^\top X_i)\},$ with score equation $\sum_{i=1}^n
\{A_i-\pi_\gamma(X_i)\}X_i
=
0.$ They then consider the joint parameter \((\theta,\gamma)\) and form a stacked estimating system combining the Cox score for \(\theta\) with the score equation for \(\gamma\). This propagates propensity-score estimation uncertainty into the variance of \(\widehat\theta_{\mathrm{IPW}}\). However, this correction relies on a correctly specified finite-dimensional
logistic propensity-score model and is therefore not robust to propensity-score
model misspecification, for example when the true propensity score is nonlinear
in the covariates.
\end{enumerate}


We address two limitations of existing IPW Cox methods
\citep{lin1989robust,binder1992fitting,shu2021variance}: fixed-weight robust
sandwich estimators ignore uncertainty from estimated weights, whereas corrected
sandwich estimators can be sensitive to logistic propensity-score model
misspecification. We address the former using stacked estimating equations, as in
\citet{shu2021variance}, to account for uncertainty induced by estimating the propensity-score-based weights, and the latter by incorporating the MEC framework
\citep{lee2026mec} into the IPW Cox estimating equation.

Other existing approaches in the literature include bootstrap procedures, which can account for
propensity-score estimation uncertainty by resampling subjects and refitting the
propensity-score model and weighted Cox estimator \citep{austin2016variance}.
However, bootstrap inference can be computationally prohibitive when flexible
machine-learning propensity-score models, survival prediction models, or
calibration steps must be repeatedly refitted. Linearization-based closed-form
variance estimators have also been proposed \citep{hajage2018closed}. However,
as with the variance correction of \citet{shu2021variance}, these approaches
are primarily developed under parametric propensity-score models, typically
logistic regression, and do not directly address inference when flexible ML nuisance estimation is used. 





\section{Machine-Learning-Assisted Generalized Entropy Calibration for Weighted Cox Regression}
\label{sec:mec_cox}




\subsection{A brief overview of the MEC framework and relevant literature}
\label{subsec:mec-framework}

\paragraph{Overview.} We briefly summarize the MEC framework for a weighted estimating equation.
For further details, see the Appendix of \citet{lee2026mec}.
Figure~\ref{fig:mec_conceptual} illustrates the general MEC workflow.
Suppose that the target parameter \(\theta_0\) is characterized as the unique
root of a population estimating equation \(U(\theta_0)=0\). Further suppose that
an initial weighted sample estimating equation \(U_n^\omega(\theta)=0\) provides
a valid sample analogue, in the sense that the normalized weighted sample
estimating equation converges to \(U(\theta)\). The weights
\(\omega\) may encode sampling, missingness, treatment assignment, or transport,
depending on the specific application.

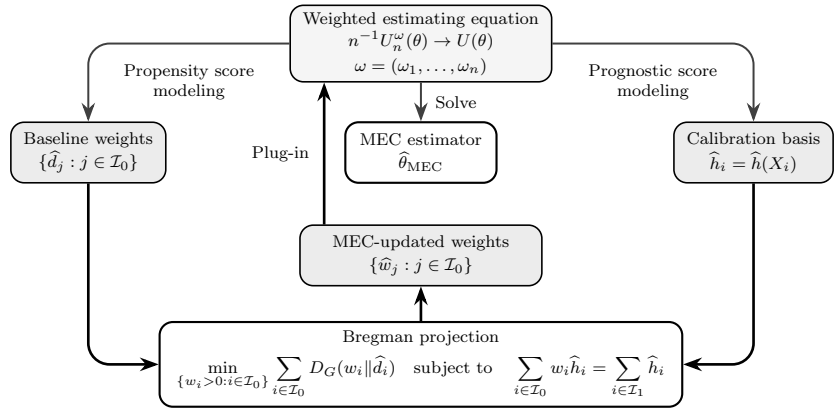
\begin{figure}[h!]
\centering
\begin{tikzpicture}[
    scale=0.8,
    transform shape,
    >=Stealth,
    font=\scriptsize,
    boxmain/.style={
        draw=black, fill=gray!8,
        rounded corners=5pt, line width=0.6pt, align=center,
        inner xsep=6pt, inner ysep=3pt, minimum height=1.0cm
    },
    boxweight/.style={
        draw=black, fill=gray!15,
        rounded corners=5pt, line width=0.6pt, align=center,
        inner xsep=6pt, inner ysep=3pt, minimum height=1.0cm
    },
    boxest/.style={
        draw=black, fill=white,
        rounded corners=5pt, line width=0.8pt, align=center,
        inner xsep=6pt, inner ysep=3pt, minimum height=1.0cm
    },
    boxbasis/.style={
        draw=black, fill=gray!15,
        rounded corners=5pt, line width=0.6pt, align=center,
        inner xsep=6pt, inner ysep=3pt, minimum height=1.0cm
    },
    boxbregman/.style={
        draw=black, fill=white, text=black,
        rounded corners=5pt, line width=0.8pt, align=center,
        inner xsep=8pt, inner ysep=4pt, minimum height=1.1cm
    },
    flow/.style={
        draw=black!75,
        line width=0.75pt,
        -{Stealth[length=2.5mm,width=1.8mm]}
    },
    flowthick/.style={
        draw=black,
        line width=1.0pt,
        -{Stealth[length=2.5mm,width=1.8mm]}
    },
    lab/.style={
        font=\scriptsize,
        align=center,
        text=black
    }
]


\node[boxmain, minimum width=4.0cm] (top) at (0, 5.1) {%
    Weighted estimating equation\\[0.15em]
    $\displaystyle n^{-1}U_n^\omega(\theta)\to U(\theta)$\\[0.25em]
    $\displaystyle \omega=(\omega_1,\dots,\omega_n)$
};

\node[boxest, minimum width=2.5cm] (center) at (0, 3.3) {%
    MEC estimator\\[0.15em]
    $\widehat\theta_{\mathrm{MEC}}$
};

\node[boxweight, minimum width=2.5cm] (left) at (-5.5, 3.3) {%
    Baseline weights\\[0.15em]
    $\{\widehat d_j:j\in \mathcal I_0\}$
};

\node[boxbasis, minimum width=2.5cm] (right) at (5.5, 3.3) {%
    Calibration basis\\[0.25em]
    $\widehat h_i=\widehat h(X_i)$
};

\node[boxweight, minimum width=4.0cm] (updated) at (0, 1.6) {%
    MEC-updated weights\\[0.15em]
    $\{\widehat w_j:j\in \mathcal I_0\}$
};

\node[boxbregman, minimum width=6.5cm] (bregman) at (0, -0.2) {%
    Bregman projection\\[0.4em]
    $\displaystyle \min_{\{w_i>0:i\in\mathcal I_0\}} \sum_{i\in\mathcal I_0}D_G(w_i\|\widehat d_i) \quad \text{subject to} \quad \sum_{i\in\mathcal I_0}w_i\widehat h_i = \sum_{i\in\mathcal I_1}\widehat h_i$
};


\draw[flow] (top.south) -- (center.north)
    node[midway, right=1mm, lab] {Solve};

\draw[flow, rounded corners=8pt] (top.west) -- 
    node[below=2mm, lab] {Propensity score\\modeling} 
    (top.west -| left.north) -- (left.north);

\draw[flow, rounded corners=8pt] (top.east) -- 
    node[below=2mm, lab] {Prognostic score\\modeling} 
    (top.east -| right.north) -- (right.north);

\draw[flowthick, rounded corners=8pt] (left.south) -- 
    (left.south |- bregman.west) -- (bregman.west);

\draw[flowthick, rounded corners=8pt] (right.south) -- 
    (right.south |- bregman.east) -- (bregman.east);

\draw[flowthick] (bregman.north) -- (updated.south);

\draw[flowthick] ([xshift=-1.6cm]updated.north) -- ([xshift=-1.6cm]top.south)
    node[midway, left=1mm, lab] {Plug-in};

\end{tikzpicture}
\caption{Conceptual diagram for constructing the MEC estimator.}
\label{fig:mec_conceptual}
\end{figure}

MEC starts from normalized baseline weights
\(\{\widehat d_i:i\in \mathcal{I}_0\}\) for a source set
\(\mathcal{I}_0\), which define a valid initial weighted estimator for the
target parameter. Its distinctive feature is the use of a cross-fitted
calibration basis
\(\widehat h_i=\widehat h(X_i)=(1,\widehat b_1^{(-)}(X_i),\ldots,
\widehat b_{p-1}^{(-)}(X_i))\), where the components
\(\widehat b_\ell^{(-)}(\cdot)\) are out-of-fold ML predictions that determine
the directions along which balance is enforced. In causal settings, these steps
may correspond to propensity-score modeling and prognostic modeling,
respectively, allowing MEC to incorporate flexible ML methods. Given a strictly
convex Bregman generator \(G\), MEC updates the baseline weights by a Bregman
projection subject to calibration constraints, so that the weighted source set
\(\mathcal I_0\) is matched to the target set \(\mathcal I_1\) on the
calibration basis. The MEC-updated weights
\(\{\widehat w_i:i\in\mathcal{I}_0\}\) are then inserted into the original
weighted estimating equation, yielding the MEC estimator
\(\widehat\theta_{\rm MEC}\).

\paragraph{MEC as a prognostic-score balancing-weight method.}
In causal settings, the MEC framework is closely related to the broad literature
on balancing weights \citep{zubizarreta2023handbook}. Since the introduction of
classical calibration estimators in survey sampling
\citep{deville1992calibration}, related ideas have been extensively developed
in causal inference, including entropy balancing \citep{hainmueller2012entropy},
stable balancing weights \citep{zubizarreta2015stable}, covariate balancing
propensity scores \citep{imai2014covariate}, empirical balancing calibration
weighting \citep{chan2016globally}, overlap and balancing weights
\citep{li2018balancing}, approximate residual balancing
\citep{athey2018approximate}, augmented minimax linear estimation
\citep{hirshberg2021augmented}, and regression-implied weighting
\citep{chattopadhyay2023implied}. These methods share the principle that
weighted estimation can be improved by balancing suitable functions of baseline
covariates. When such functions are used to define calibration constraints, we refer to them as calibration basis functions, or collectively as the calibration basis. In this paper, we denote the population calibration basis by \(h(X)\) and its estimated version by \(\widehat h(X)\).

A central question in balancing-weight methods is therefore which functions of
the covariates should be balanced \citep{cohn2023balancing}. Existing approaches
commonly impose balance on raw covariates, low-order moments, interactions,
basis expansions, kernel-induced features, or other analyst-specified function
classes. MEC \citep{lee2026mec} takes a different approach by constructing the
calibration basis from cross-fitted ML predictions, or prognostic summaries in
causal settings, so that the balancing functions represent outcome-relevant
directions rather than only pre-specified covariate features. This perspective is
aligned with \citet{benmichael2021balancing}, who emphasize that the relevant
notion of balance is not merely balance in raw covariates themselves, but balance
in functions of the covariates that are predictive of the outcome or the
estimating equation, such as the conditional mean outcome function or the regression-error function. 

In causal settings, MEC can therefore be viewed as a
\emph{prognostic-score balancing-weight method}: it preserves the baseline
weighting mechanism used for identification, such as source-transport weighting,
while refining the weights through a principled calibration step that balances
the source and target samples in outcome-predictive directions. In this sense,
the MEC-updated weights play a dual role as both source-transport weights and
prognostic-score balancing weights.

\subsection{Applying MEC to weighted Cox regression}
\label{subsec:Applying MEC to weighted Cox regression}
We apply MEC \citep{lee2026mec} to the weighted
Cox estimating equation
\eqref{eq:weighted_Cox_regression_counting_process_form} to estimate the
marginal hazard-ratio parameter \(\theta_{ATT}\) defined in
\eqref{eq:marginal_structural_model_for_HR_ATT}. Throughout this section, we
assume the conditions of Theorem~\ref{thm:identification_weighted_cox}, so that
the ATT-weighted Cox estimating equation is centered at the marginal ATT
log-hazard ratio \(\theta_{ATT}\). Thus, identification has already been
established, and we focus on the construction and estimation of the MEC-Cox
estimator. The consistency and variance estimation of the MEC-Cox estimator are detailed in Subsection~\ref{subsec:conditions_consistency_mec_cox_optimal_calibration_basis} and Section~\ref{sec:variance_mec_cox}, respectively.

Let \(\mathcal I_1=\{i:A_i=1\}\) denote the treated trial index set, representing the target population, and let \(\mathcal I_0=\{i:A_i=0\}\) denote the external-control index set, representing the source (study) population. Let
\(|\mathcal I_1|=n_1\), \(|\mathcal I_0|=n_0\), and \(n=n_1+n_0\). We use stratified \(K\)-fold cross-fitting on the pooled sample of \(n\) patients. Specifically, we partition the full index set \(\{1,\ldots,n\}\) into \(K\) folds, $\mathcal J^{(1)},\ldots,\mathcal J^{(K)}.$ The fold assignment is performed separately within the trial cohort and the external-control cohort, so that each validation fold contains both source groups in approximately the same proportion as the full sample. For each fold \(k=1,\ldots,K\), let $\mathcal J^{(-k)}
=
\{1,\ldots,n\}\setminus \mathcal J^{(k)}$ denote the corresponding training set.

We first describe the construction of the baseline weights for MEC. Using observations in \(\mathcal J^{(-k)}\), we estimate the source propensity score $\pi(X_i)=\Pr(A_i=1\mid X_i)$ by a flexible ML method (or simply, by a parametric logistic regression), and denote the resulting estimator by \(\widehat\pi^{(-k)}(X)\). For each validation subject \(i\in\mathcal J^{(k)}\), define the cross-fitted propensity-score estimate $\widehat\pi_i
=
\widehat\pi^{(-k)}(X_i),$ possibly after truncation to avoid extreme values. For external-control subjects \(i\in\mathcal I_0\), define the cross-fitted ATT odds weight by $\widehat q_i
=\widehat\pi_i/(1-\widehat\pi_i).$ In our implementation, these baseline external-control weights are normalized to have total mass equal to the number of treated trial patients. 

We define the normalized baseline ATT weights, in a Hájek-type form, as
\begin{align}
\label{eq:baseline_weights}
\widehat d_i
=
\frac{n_1\widehat q_i}
{\sum_{j\in\mathcal I_0}\widehat q_j},
\qquad i\in\mathcal I_0.
\end{align}
These normalized ATT odds weights serve as the baseline weights for MEC.

Next, we describe the construction of the calibration basis. Let \(0<t_1<\cdots<t_{p-1}\) denote prespecified landmark times. For example, one may set these
landmark times as empirical quantiles of the observed event times among external controls. We estimate the external-control survival regression $S_0^0(t\mid X)
=
\Pr(T_i^0>t\mid X_i=X,A_i=0)
=
\Pr(T_i>t\mid X_i=X,A_i=0),$ where the second equality follows from consistency (A1) among external-control patients. Under the survival transportability assumption (A2), this external-control survival regression provides a prognostic summary of the counterfactual control survival experience for trial patients with the same covariates. 

For each fold \(k\), we fit a flexible control-survival learner using only the external-control subjects in the training set, $\mathcal I_0\cap\mathcal J^{(-k)},$ and denote the resulting estimator by $\widehat S_0^{(-k)}(t\mid X).$ One may use flexible survival learners, such as random survival forests (RSF)
\citep{ishwaran2008random}, when the external-control sample size is sufficiently
large. Simpler learners, such as Cox regression, may be preferable when the
external-control sample size is moderate or when computational speed is important.

For every validation subject \(i\in\mathcal J^{(k)}\), including both trial patients and external-control patients, define the $p$-dimensional cross-fitted survival-feature vector
\begin{align}
\label{eq:predictor_basis_survival}
\widehat h_i
=
\widehat h_i(X_i; t_1, \ldots, t_{p-1})
=
\left(
1,
\widehat S_0^{(-k)}(t_1\mid X_i),
\ldots,
\widehat S_0^{(-k)}(t_{p-1}\mid X_i)
\right) \in \mathbb{R}^{p}.
\end{align}

The intercept component of \(\widehat h_i\) enforces normalization of the
calibrated external-control weights. The landmark predicted control-survival
probabilities in \eqref{eq:predictor_basis_survival} can be viewed as a
vector-valued control-prognostic score
\citep{hansen2008prognostic,leacy2014joint}; see
Subsection~\ref{subsec:Choices of calibration basis and Bregman generator for MEC-Cox} for details. These features
summarize baseline covariates through their predicted association with the
counterfactual control survival outcome. In MEC-Cox, they guide the calibration step so that the weighted external controls resemble the treated trial cohort in predicted control prognosis.

In Subsection~\ref{subsec:conditions_consistency_mec_cox_optimal_calibration_basis},
we establish that, for consistency of the resulting MEC-Cox estimator, this
specific survival-probability basis \eqref{eq:predictor_basis_survival} can be replaced by any fixed-dimensional,
cross-fitted prognostic score or calibration basis, provided that it satisfies
the \(L_2\)-stochastic boundedness condition and the calibration regularity
condition. In Subsection~\ref{subsec:Choices of calibration basis and Bregman generator for MEC-Cox}, we further discuss
an ideal oracle basis and practical choices of the calibration basis.

We are now ready to perform the core MEC step, which updates the baseline
external-control weights in \eqref{eq:baseline_weights} through 
Bregman projection under calibration constraints constructed from
the survival-feature basis in \eqref{eq:predictor_basis_survival}. Let \(G\) be a strictly convex and twice continuously differentiable generator, with derivative \(g=G'\). See Table~\ref{tab:bregman_generators} for common choices of \(G\). For \(w_i>0\), define the Bregman divergence from the normalized baseline weight \(\widehat d_i\) by
\[
D_G(w_i\|\widehat d_i)
=
G(w_i)-G(\widehat d_i)-g(\widehat d_i)(w_i-\widehat d_i).
\]

\begin{table}[htbp]
\centering
\caption{Representative Bregman generators.}
\label{tab:bregman_generators}
\small
\setlength{\tabcolsep}{4pt}
\renewcommand{\arraystretch}{1.12}

\begin{tabular*}{\textwidth}{@{\extracolsep{\fill}}llll@{}}
\toprule
\textsc{Entropy} & \(G(u)\) & \(g(u)=G'(u)\) & \(D_G(u\|v)\) \\
\midrule
Quadratic
&
\(\frac{1}{2}u^2\)
&
\(u\)
&
\(\frac{1}{2}(u-v)^2\)
\\

Kullback--Leibler
&
\(u\log u\)
&
\(\log u+1\)
&
\(u\log(u/v)-u+v\)
\\

Empirical likelihood
&
\(-\log u\)
&
\(-u^{-1}\)
&
\(-\log(u/v)+u/v-1\)
\\

Squared Hellinger
&
\((\sqrt{u}-1)^2\)
&
\(1-u^{-1/2}\)
&
\(\sqrt{v}\{1-\sqrt{u/v}\}^2\)
\\

Inverse
&
\(\frac{1}{2u}\)
&
\(-\frac{1}{2}u^{-2}\)
&
\(\frac{1}{2}\left(\frac{1}{u}-\frac{1}{v}\right)
+\frac{u-v}{2v^2}\)
\\

R\'enyi \((\alpha>0)\)
&
\(\frac{u^{\alpha+1}}{\alpha+1}\)
&
\(u^\alpha\)
&
\(\frac{u^{\alpha+1}-v^{\alpha+1}}{\alpha+1}
-v^\alpha(u-v)\)
\\
\bottomrule
\end{tabular*}
\end{table}


The MEC-updated external-control weights are obtained by solving
\begin{align}
\label{eq:Bregman_projection}
\widehat w
=
\arg\min_{\{w_i>0:i\in\mathcal I_0\}}
\sum_{i\in\mathcal I_0}
D_G(w_i\|\widehat d_i)
\end{align}
subject to the calibration constraint $$\sum_{i\in\mathcal I_0} w_i\widehat h_i
=
\sum_{i\in\mathcal I_1}\widehat h_i.$$

The solution of \eqref{eq:Bregman_projection} admits the dual representation
\begin{align}
\label{eq:MEC_updated_weight}
\widehat w_i
=
w_i(\widehat\lambda)
=
g^{-1}\{g(\widehat d_i)+\widehat\lambda^\top \widehat h_i\},
\qquad i\in\mathcal I_0,
\end{align}
where \(\widehat\lambda\) is chosen to satisfy the dual calibration equation:
\begin{align}
\label{eq:dual_calibration_equation}
F(\lambda)
=
\sum_{i\in\mathcal I_0}
\widehat h_i
g^{-1}\{g(\widehat d_i)+\lambda^\top\widehat h_i\}
-
\sum_{i\in\mathcal I_1}\widehat h_i
=
0.
\end{align}

The final MEC-Cox estimating-equation weight is then defined by
\begin{align}
\label{eq:cox_estimating_equation_weights}
\widetilde\omega_i
=
\begin{cases}
1, & i\in\mathcal I_1 \quad \text{(trial cohort)},\\
\widehat w_i, & i\in\mathcal I_0 \quad \text{(external-control cohort)}.
\end{cases}
\end{align}
Here, `\(\widetilde{\omega}\)' denotes the final estimating-equation weight, whereas `\(\widehat w\)' denotes the MEC-updated weight. Equivalently, since \(A_i=1\) for treated trial patients and \(A_i=0\) for external-control patients, we can write $\widetilde\omega_i = A_i + (1-A_i)\widehat w_i.$


Finally, the MEC-Cox estimator \(\widehat\theta_{\mathrm{MEC}}\) is defined as the solution to
\begin{align}
\label{eq:Cox_score_MEC_theta}
U_n^{\mathrm{MEC}}(\theta)
=
\sum_{i=1}^n
\int
\widetilde\omega_i
\{A_i- \overline A_{\rm MEC, \omega}(t;\theta)\}
\,d\mathcal{N}_i(t)
=
0,    
\end{align}
where $\overline A_{\rm MEC, \omega}(t;\theta) = \{\sum_{i=1}^n
\widetilde\omega_i\mathcal{Y}_i(t)\exp(\theta A_i)A_i\}/\{ 
\sum_{i=1}^n
\widetilde\omega_i\mathcal{Y}_i(t)\exp(\theta A_i)\}
$.
The corresponding MEC-Cox hazard-ratio estimator is $\widehat{HR}_{\mathrm{MEC}}
=
\exp(\widehat\theta_{\mathrm{MEC}}).$
\subsection{Conditions for consistency of the MEC-Cox estimator}
\label{subsec:conditions_consistency_mec_cox_optimal_calibration_basis}
We have primarily illustrated MEC-Cox using predicted control survival
probabilities as the calibration basis in
\eqref{eq:predictor_basis_survival}. The following theorem establishes the
consistency of the MEC-Cox estimator \(\widehat\theta_{\mathrm{MEC}}\) for the
marginal log-hazard ratio \(\theta_{ATT}\) in
\eqref{eq:marginal_structural_model_for_HR_ATT}, under the causal and regularity
assumptions stated in Theorem~\ref{thm:identification_weighted_cox}, allowing
for a more general choice of calibration basis.

\begin{theorem}\label{thm:Conditions_for_consistency_MEC_Cox}
Suppose the assumptions of Theorem~\ref{thm:identification_weighted_cox}
hold. For \(i\in\mathcal J^{(k)}\), let
\(\widehat\pi_i=\widehat\pi^{(-k)}(X_i)\) and
\(\widehat h_i=\widehat h^{(-k)}(X_i)\), where
\(\widehat\pi^{(-k)}(\cdot)\) is the source-propensity-score estimator trained
on the pooled training-fold observations
\(\{(A_j,X_j):j\in\mathcal J^{(-k)}\}\), and
\(\widehat h^{(-k)}(\cdot)\) is a cross-fitted calibration-basis map trained
using the external-control observations
\(\{(X_j,Y_j,\delta_j):j\in\mathcal I_0\cap\mathcal J^{(-k)}\}\) and evaluated
on the validation fold \(\mathcal J^{(k)}\). Specifically, we write
\begin{align}
\label{eq:generic_basis_for_MEC_Cox}
\widehat h_i
=
\widehat h^{(-k)}(X_i)
=
\left(
1,\widehat b^{(-k)}_1(X_i),\ldots,
\widehat b^{(-k)}_{p-1}(X_i)
\right),
\qquad i\in\mathcal J^{(k)},
\end{align}
where \(\widehat b^{(-k)}_\ell(\cdot)\), \(\ell=1,\ldots,p-1\), denotes a
data-adaptive basis function trained using the external-control training subset
\(\mathcal I_0\cap\mathcal J^{(-k)}\). Let
\(\widehat\pi^{(-)}(\cdot)\) and \(\widehat h^{(-)}(\cdot)\) denote the
corresponding generic out-of-fold propensity-score estimator and
calibration-basis map, respectively. Assume the following additional conditions:
\begin{itemize}[leftmargin=1.5em, itemsep=0.35em]

    \item[] \textbf{\textit{C1. \(L_2\)-consistency of the cross-fitted propensity-score estimator.}} Under the external-control covariate distribution,
    \[
    \mathbb E\left[
    \{\widehat\pi^{(-)}(X)-\pi(X)\}^2 \mid A=0
    \right]
    =
    o_p(1).
    \]

    \item[] \textbf{\textit{C2. \(L_2\)-stochastic boundedness of the cross-fitted calibration basis.}} 
    The fixed-dimensional calibration basis
    \(\widehat h^{(-)}(X)\in\mathbb R^p\) satisfies, for each fixed \(p\) and
    \(r=0,1\),
    \[
    \mathbb E\left[
    \|\widehat h^{(-)}(X)\|^2
    \mid A=r
    \right]
    =
    O_p(1),
    \]
    where \(\|\cdot\|\) denotes the Euclidean norm.

    \item[] \textbf{\textit{C3. Calibration regularity.}} The Bregman calibration problem is well posed in the following sense:
    (i) the dual calibration equation \(F(\lambda)=0\) in
    \eqref{eq:dual_calibration_equation} admits a local solution
    \(\widehat\lambda\);
    (ii) \(F(\lambda)\) is locally smooth around \(\lambda=0\), and the
    corresponding normalized Jacobian has a nonsingular limiting matrix; and
    (iii) the inverse gradient map \(g^{-1}\) is locally smooth on the
    corresponding neighborhood of the baseline dual values
    \(\{g(\widehat d_i):i\in\mathcal I_0\}\).
\end{itemize}
Then the MEC-Cox estimator \(\widehat\theta_{\mathrm{MEC}}\) is consistent for
\(\theta_{ATT}\), the marginal log-hazard ratio in the treated trial target
population.
\end{theorem}

Theorem~\ref{thm:Conditions_for_consistency_MEC_Cox} shows that
\(\widehat\theta_{\mathrm{MEC}}\) is consistent under the same causal and
Cox-score regularity conditions that justify the ATT-weighted Cox estimator,
together with \(L_2\)-consistency of the cross-fitted source propensity-score
estimator, \(L_2\)-stochastic boundedness of the cross-fitted calibration
basis, and regularity of the Bregman calibration problem. Cross-fitting plays
a key role in this argument. Conditional on the training fold, the
calibration-basis map \(\widehat h^{(-)}(\cdot)\) can be treated as a fixed
covariate-dependent function when it is evaluated on the validation fold. This
allows the proof to avoid Donsker-type entropy conditions or other empirical
process complexity restrictions on the data-adaptive procedure used to construct
the calibration basis, in the same spirit as cross-fitting in TMLE and DML literature \citep{chernozhukov2018double,zheng2011cross}.

It is important to note that MEC-Cox remains an IPW-type Cox estimator based on
the weighted Cox score \eqref{eq:weighted_Cox_regression_counting_process_form}.
Its outcome-adaptive component enters only through the MEC step described in
Subsection~\ref{subsec:Applying MEC to weighted Cox regression}. The MEC step
uses the calibration basis \(\{\widehat h_i:i=1,\ldots,n\}\) \eqref{eq:generic_basis_for_MEC_Cox} to produce
MEC-updated external-control weights
\(\{w_i(\widehat\lambda):i\in\mathcal I_0\}\) \eqref{eq:MEC_updated_weight}, which differ from the baseline
external-control weights \(\{\widehat d_i:i\in\mathcal I_0\}\) \eqref{eq:baseline_weights} only through a
finite-sample Bregman perturbation. Under the regularity conditions stated in
Theorem~\ref{thm:Conditions_for_consistency_MEC_Cox}, this perturbation is
asymptotically negligible in empirical \(L_2\) norm, and the normalized ATT-IPW
weights are asymptotically feasible for the MEC calibration constraint: $(1/n)\sum_{i\in\mathcal I_0}
\{w_i(\widehat\lambda)-\widehat d_i\}^2
= o_p(1),$ and $(1/n_1)
\{
\sum_{i\in\mathcal I_0}\widehat d_i\widehat h_i
-
\sum_{i\in\mathcal I_1}\widehat h_i
\}
= o_p(1),$ where the second \(o_p(1)\) statement is understood componentwise.



\subsection{Choices of calibration basis and Bregman generator for MEC-Cox}
\label{subsec:Choices of calibration basis and Bregman generator for MEC-Cox}

This subsection discusses the choices of calibration basis and Bregman generator for MEC-Cox. We first describe the oracle prognostic score basis as an efficiency target, then discuss practical cross-fitted calibration bases, and motivate the Kullback-Leibler (KL) generator as a canonical default for ATT transport.

\paragraph{An oracle prognostic score basis.}
Following \citet{hansen2008prognostic}, a control-prognostic score is a
covariate summary \(\Psi(X)\) such that
\begin{align}
\label{eq:prognostic_score_condition}
T^0\perp X\mid \Psi(X).
\end{align}
Thus, conditional on \(\Psi(X)\), the remaining variation in \(X\) is not
predictive of the counterfactual control event time \(T^0\). This differs from
the source propensity score \(\pi(X)=\Pr(A=1\mid X)\), which is a balancing
score for source or treatment assignment satisfying \(X\perp A\mid \pi(X)\)
\citep{rosenbaum1983central}. In MEC-Cox, the source propensity score, or
equivalently the ATT odds weight \(q(X)=\pi(X)/\{1-\pi(X)\}\), determines the
transport weights, whereas the control-prognostic score \(\Psi(X)\) summarizes
outcome-relevant variation. Figure~\ref{fig:DAG_MEC_Cox} summarizes these
distinct roles.

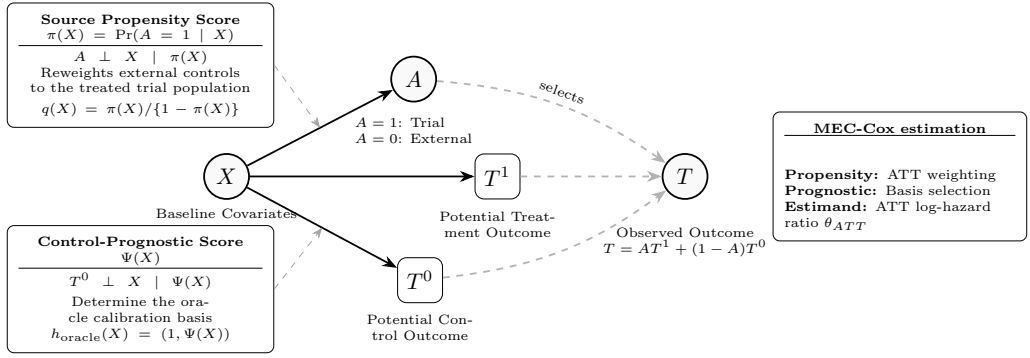
\begin{figure}[h!]
    \centering
    \resizebox{0.9\textwidth}{!}{
    \begin{tikzpicture}[
        node distance=1.5cm and 2.2cm,
        obs/.style={draw, thick, circle, minimum size=0.7cm, inner sep=1pt, font=\small\itshape, fill=gray!5},
        pot/.style={draw, thin, rectangle, rounded corners=4pt, minimum size=0.7cm, inner sep=4pt, font=\small\itshape, fill=white},
        box/.style={draw, thin, rectangle, rounded corners=3pt, text width=3.8cm, font=\tiny, inner sep=5pt, align=center},
        arrow/.style={-{Stealth[scale=1.0]}, thick},
        dashed_arrow/.style={-{Stealth[scale=1.0]}, thick, dashed, gray!60}
    ]

        \node[obs] (X) {$X$};
        \node[below=0.00cm of X, font=\tiny] {Baseline Covariates};

        \node[obs, above right=1.0cm and 2.4cm of X] (A) {$A$};
        \node[below=0.1cm of A, font=\tiny, align=left] {
            $A=1$: Trial \\
            $A=0$: External
        };

        \node[pot, right=3.5cm of X] (T1) {$T^1$};
        \node[below=0.05cm of T1, font=\tiny, text width=2.5cm, align=center] {Potential Treatment Outcome};

        \node[pot, below right=1.0cm and 2.4cm of X] (T0) {$T^0$};
        \node[below=0.05cm of T0, font=\tiny, text width=2.5cm, align=center] {Potential Control Outcome};

        \node[obs, right=2.2cm of T1] (T) {$T$};
        \node[below=0.3cm of T, font=\tiny, align=center] {
            Observed Outcome \\
            $T = AT^1 + (1-A)T^0$
        };

        \draw[arrow] (X) -- (A) coordinate[midway] (XA_mid);
        \draw[arrow] (X) -- (T1);
        \draw[arrow] (X) -- (T0) coordinate[midway] (XT0_mid);

        \draw[dashed_arrow] (A) to [bend left=15] node[above, sloped, black, font=\tiny] {selects} (T);
        \draw[dashed_arrow] (T1) -- (T);
        \draw[dashed_arrow] (T0) to [bend right=15] (T);

        \node[box, above left=0.5cm and -1cm of X] (PropScore) {
            \textbf{Source Propensity Score} \\
            $\pi(X)=\Pr(A=1\mid X)$ \\
            \vspace{1pt} \hrule \vspace{2pt}
            $A \perp X \mid \pi(X)$ \\
            Reweights external controls \\
            to the treated trial population \\
            \vspace{2pt}
            $q(X)=\pi(X)/\{1-\pi(X)\}$
        };

        \node[box, below left=0.5cm and -1cm of X] (ProgScore) {
            \textbf{Control-Prognostic Score} \\
            $\Psi(X)$ \\
            \vspace{1pt} \hrule \vspace{2pt}
            $T^0 \perp X \mid \Psi(X)$ \\
            \vspace{2pt}
            Determine the oracle calibration basis\\
            $h_{\rm oracle}(X)=(1,\Psi(X))$
        };

        \node[box, right=1cm of T, text width=3.6cm] (Interpretation) {
            \textbf{MEC-Cox estimation} \\
            \vspace{1pt} \hrule \vspace{3pt}
            \begin{flushleft}
            \textbf{Propensity:} ATT weighting \\
            \textbf{Prognostic:} Basis selection \\
            \textbf{Estimand:} ATT log-hazard ratio \(\theta_{ATT}\)
            \end{flushleft}
        };

        \draw[dashed, -{Stealth}, gray!70] (PropScore.east) -- (XA_mid);
        \draw[dashed, -{Stealth}, gray!70] (ProgScore.east) -- (XT0_mid);

    \end{tikzpicture}
    }
    \caption{Conceptual roles of the source propensity score and the
control-prognostic score in MEC-Cox.}
    \label{fig:DAG_MEC_Cox}
\end{figure}

This distinction motivates the oracle control-prognostic score basis for MEC-Cox,
\begin{align}
\label{eq:oracle_basis}
h_{\rm oracle}(X)=(1,\Psi(X)).
\end{align}
Proposition~\ref{prop:ideal_prognostic_basis} in the Appendix shows that, under the conditions of
Theorem~\ref{thm:Conditions_for_consistency_MEC_Cox} and the existence of a
control-prognostic score satisfying \eqref{eq:prognostic_score_condition}, the
following inequality holds among the admissible MEC-Cox calibration bases
\(\mathcal H\):
\[
\mathcal V_{\rm res}^{0}(h)
\ge
\mathcal V_{\rm res}^{0}(h_{\rm oracle}),
\qquad h\in\mathcal H.
\]
Here, \(\mathcal V_{\rm res}^{0}(h)\) denotes the residual variation of the
weight-normalized external-control Cox contribution after projection onto the
basis \(h(X)\), $\mathcal V_{\rm res}^{0}(h)
=
\operatorname{Var}[
\eta^0-\mathbb E\{\eta^0\mid h(X),A=0\}\mid A=0
],$ $\eta^0=\eta/q(X),$ where \(\eta\) denotes a generic limiting external-control LW Cox contribution, obtained as the limiting analogue of \(\widehat\eta_i\) defined in
\eqref{eq:MEC_Cox_contribution_at_solution_optimized_dual_parameter} of Section \ref{sec:variance_mec_cox} with $A_i=0$; see the proof of Proposition~\ref{prop:ideal_prognostic_basis} for its closed-form expression. Thus, after the ATT
transport weight has been factored out, the oracle basis
\eqref{eq:oracle_basis} leaves the smallest residual variation in the
weight-normalized external-control Cox contribution. Equivalently, it captures
the largest baseline-covariate-explained component of the control-survival
contribution relevant to the MEC-Cox estimating equation.

\paragraph{Practical choice of calibration basis.}
In practice, the true control-prognostic score \(\Psi(X)\) is unknown, and the
oracle basis \(h_{\rm oracle}(X)=(1,\Psi(X))\) cannot be used directly.
Thus, the oracle basis should be viewed as an ideal efficiency-motivating target,
rather than as a requirement for constructing MEC-Cox. Because \(\Psi(X)\) is an
unknown and potentially complex prognostic summary, estimating it consistently
can be challenging. This naturally raises the question of whether the MEC-Cox
estimator remains consistent when the chosen calibration basis does not
consistently estimate the oracle prognostic score or any other population
nuisance function.

Theorem~\ref{thm:Conditions_for_consistency_MEC_Cox} guarantees consistency of
the MEC-Cox estimator when the cross-fitted source-propensity-score estimator is
\(L_2\)-consistent as in C1, and the MEC constraint is constructed using a
fixed-dimensional, cross-fitted calibration basis of the form
\eqref{eq:generic_basis_for_MEC_Cox} that satisfies the \(L_2\)-stochastic
boundedness condition in C2 and the calibration regularity condition in C3.
Thus, the specific form of \(\widehat b^{(-k)}_\ell(\cdot)\),
\(\ell=1,\ldots,p-1\), is not essential for consistency of
\(\widehat\theta_{\mathrm{MEC}}\) for \(\theta_{\mathrm{ATT}}\). The
survival-probability landmark basis in \eqref{eq:predictor_basis_survival} is therefore only one possible choice. 

Since the boundedness condition C2 in
Theorem~\ref{thm:Conditions_for_consistency_MEC_Cox} is fairly mild, it allows
substantial flexibility in the choice of calibration basis. For example, when computational
speed is a primary concern, the calibration basis may be constructed from
simple out-of-fold prognostic summaries, such as estimated hazards, cumulative
hazards, or risk scores. One convenient choice is the cross-fitted Cox linear
predictor, which gives the two-dimensional calibration basis $\widehat h_i
=
(
1,\widehat\zeta^{(-k)}(X_i)
)
\in\mathbb R^2$ with $\widehat\zeta^{(-k)}(X_i)
=
X_i^\top\widehat\beta_0^{(-k)},$ ($i\in\mathcal J^{(k)} $). Here, \(\widehat\zeta^{(-k)}(X_i)\) is the out-of-fold linear predictor trained
using the external-control training sample
\(\mathcal I_0\cap\mathcal J^{(-k)}\) under a Cox working model
\citep{cox1972regression}. In high-dimensional settings,
\(\widehat\beta_0^{(-k)}\) may be obtained from a penalized Cox model, such as
the lasso-Cox estimator \citep{tibshirani1997lasso}.

\paragraph{Canonical generator for ATT transport.}
The choice of Bregman generator determines the geometry by which the baseline
external-control weights \(\{\widehat d_i:i\in I_0\}\) in
\eqref{eq:baseline_weights} are perturbed. In the ATT setting, these baseline
weights are already normalized source-to-target transport weights, because they
are constructed from the estimated ATT odds \(\widehat q_i\). Therefore, the
natural calibration update should preserve this transport structure and modify it
only through a relative, prognostic-balance correction.

For the KL generator \(G(u)=u\log u\), the dual MEC solution in
\eqref{eq:MEC_updated_weight} gives the offset-plus-fluctuation representation
\begin{align}
\label{eq:KL_fluctuation}
\log \widehat w_i=
\log\left(
\frac{n_1}{\sum_{j\in I_0}\widehat q_j}
\right)
+
\log \widehat q_i
+
\widehat\lambda^\top\widehat h_i,
\qquad i\in I_0.
\end{align}
Here, the first two terms form the baseline ATT transport weight on the log
scale, while the last term is the MEC-induced basis-balancing fluctuation.
Thus, KL is canonical for ATT transport because it preserves the baseline
log-odds transport structure and updates it by a log-linear fluctuation chosen
to balance the calibration basis. When \(\widehat h_i\) is constructed from
control-prognostic features \eqref{eq:oracle_basis}, this fluctuation can be
interpreted as a prognostic-balance correction to the original ATT transport
weights. Other generators, such as empirical likelihood, squared Hellinger, and
Rényi, are valid Bregman choices, but they perturb the baseline weights on
reciprocal or power scales rather than on the log-odds transport scale. Hence,
KL provides a natural default generator for MEC-Cox in the ATT transport setting.

This MEC-updating representation in \eqref{eq:KL_fluctuation} also clarifies
the connection with targeting updates in TMLE
\citep{gruber2010targeted,van2011targeted,van2018targeted}. In both cases, an
initial estimate is updated through a finite-dimensional fluctuation on a
natural link scale. The key difference is that TMLE chooses the fluctuation to
solve an EIF equation, whereas MEC
chooses \(\widehat\lambda\) to satisfy the calibration equation. Hence, the KL
update is TMLE-like in form, but it targets calibration balance rather than
directly solving the EIF equation.

\section{Variance Estimation for the MEC-Cox Estimator}
\label{sec:variance_mec_cox}
\subsection{Stacked estimating-equation formulation}
\label{subsec:variance_stacked_equation}

We now describe variance estimation for the MEC-Cox estimator. The main difficulty is that the final Cox estimating-equation weights \(\widetilde\omega_i = A_i + (1-A_i)\widehat w_i\) ($i=1,\cdots, n$) in \eqref{eq:cox_estimating_equation_weights} are not fixed, because the external-control weights $\widehat w_i
=w_i(\widehat\lambda)$ ($i\in\mathcal I_0$) in \eqref{eq:MEC_updated_weight} are obtained through the MEC update and therefore depend on the estimated dual parameter \(\widehat\lambda\). Thus, variance estimation should account for the first-order uncertainty induced by estimating the calibration parameter.

We consider the stacked estimating system obtained by combining the weighted Cox estimating equation for \(\theta\) with the dual calibration equation for \(\lambda\), as follows. For exposition, we rewrite the MEC-weighted Cox score
\(U_n^{\mathrm{MEC}}(\theta)\) in \eqref{eq:Cox_score_MEC_theta} using notation
that makes its dependence on \(\lambda\) explicit:
\[
U_{\rm MEC,\theta}(\theta,\lambda)
:=
\sum_{i=1}^n
\int
\widetilde\omega_i(\lambda)
\{A_i-\overline A_{\rm MEC,\omega}(t;\theta,\lambda)\}
\,d\mathcal{N}_i(t),
\]
where $\overline A_{\rm MEC,\omega}(t;\theta,\lambda)
=\{\sum_{i=1}^n
\widetilde\omega_i(\lambda)\mathcal{Y}_i(t)\exp(\theta A_i)A_i\}/ \{
\sum_{i=1}^n
\widetilde\omega_i(\lambda)\mathcal{Y}_i(t)\exp(\theta A_i)\}.$ Note that the MEC-Cox estimator \(\widehat\theta_{\rm MEC}\) satisfies $U_{\rm MEC,\theta}(\widehat\theta_{\rm MEC},\widehat\lambda)=0$. The dual calibration equation for \(\lambda\) in \eqref{eq:dual_calibration_equation} can be rewritten as $F(\lambda)
=
\sum_{i=1}^n \rho_i(\lambda)
=
0$, where the subject-level calibration contribution is 
\begin{align}
\nonumber
\rho_i:\mathcal D_{\lambda}\subseteq \mathbb R^p \to \mathbb R^p,
\qquad
\rho_i(\lambda)
=
(1-A_i)w_i(\lambda)\widehat h_i
-
A_i\widehat h_i.
\end{align}
Here, \(\mathcal D_{\lambda}\) denotes the set of \(\lambda\) values for which
\(w_i(\lambda)\) is well defined. Collect the Cox parameter and dual calibration parameter into the joint vector
\begin{align}
\label{eq:stacked_parameter}
\xi=(\theta,\lambda)\in\mathbb R^{p+1},
\qquad
\widehat\xi=(\widehat\theta_{\rm MEC},\widehat\lambda).    
\end{align}
We define the stacked estimating system
\begin{align}
\label{eq:stacked_estimating_system}
\Psi_n(\xi)
=
\sum_{i=1}^n
\Psi_i(\xi)
=
\sum_{i=1}^n
\begin{pmatrix}
\eta_i(\theta,\lambda)\\
\rho_i(\lambda)
\end{pmatrix}
=
0,    
\end{align}
where \(\eta_i(\theta,\lambda)\) is the LW empirical Cox
contribution \citep{lin1989robust,binder1992fitting}, evaluated at a generic value of the Cox parameter \(\theta\) and conditional on a fixed dual parameter \(\lambda\) (see Subsection \ref{subsec:lin_wei_binder_empirical_contribution} in the Appendix for derivation):
\begin{align}
\nonumber
\eta_i(\theta,\lambda)
&=
\int
\widetilde\omega_i(\lambda)
\{A_i-\overline A_{\rm MEC,\omega}(t;\theta,\lambda)\}
\,d\mathcal{N}_i(t)\\
\label{eq:Lin_wei_Binder_contribution}
&\quad -
\int
\frac{
\widetilde\omega_i(\lambda)\mathcal{Y}_i(t)\exp(\theta A_i)
\{A_i-\overline A_{\rm MEC,\omega}(t;\theta,\lambda)\}
}{
S^{(0)}_{\lambda}(t;\theta)
}
\,d\mathcal{N}_{\omega,\lambda}(t),
\end{align}
where $S^{(r)}_{\lambda}(t;\theta)
=
\sum_{i=1}^n
\widetilde\omega_i(\lambda)
\mathcal{Y}_i(t)\exp(\theta A_i)A_i^r, $ $(r=0,1)$ and $d\mathcal{N}_{\omega,\lambda}(t)
=
\sum_{j=1}^n
\widetilde\omega_j(\lambda)\,d\mathcal{N}_j(t).$ We write
\begin{align}
    \label{eq:MEC_Cox_contribution_at_solution_optimized_dual_parameter}
    \widehat\eta_i
=
\eta_i(\widehat\theta_{\rm MEC},\widehat\lambda) \in \mathbb{R},
\qquad
\widehat\rho_i
=
\rho_i(\widehat\lambda) \in \mathbb{R}^p ,
\end{align}
and use these quantities in the stacked sandwich variance estimator.


\subsection{First-order linearization and sandwich variance}
\label{subsec:first_order_linearization_mec}

The stacked sandwich variance follows from a first-order Taylor expansion of the stacked estimating equation \(\Psi_n(\xi)\) in \eqref{eq:stacked_estimating_system}. Since \(\widehat\xi=(\widehat\theta_{\rm MEC},\widehat\lambda)\) in \eqref{eq:stacked_parameter} solves \(\Psi_n(\widehat\xi)=0\) \eqref{eq:stacked_estimating_system}, a first-order expansion around the population target \(\xi_0=(\theta_0,\lambda_0)\) gives
\begin{align}
\label{eq:taylor_expansion_joint}
0
=
\Psi_n(\widehat\xi)
=
\Psi_n(\xi_0)
+
D(\xi_0)(\widehat\xi-\xi_0)
+
R_n,
\,\,
D(\xi_0)
=
\left.
\frac{\partial \Psi_n(\xi)}
{\partial \xi^\top}
\right|_{\xi=\xi_0}\in \mathbb R^{(p+1)\times(p+1)},
\end{align}
where \(R_n\) is a higher-order Taylor remainder satisfying $\left\|D(\xi_0)^{-1}R_n\right\|=o_p(n^{-1/2})$ under standard smoothness and regularity conditions \citep{van2000asymptotic}. Here, although \(D(\xi_0)\) depends on the sample through \(\Psi_n\) \eqref{eq:stacked_estimating_system}, we suppress this dependence for notational simplicity. 

The population target \(\xi_0=(\theta_0,\lambda_0)\) in \eqref{eq:taylor_expansion_joint} satisfies
\(\theta_0=\theta_{ATT}\) and \(\lambda_0=0\) under
Theorem~\ref{thm:Conditions_for_consistency_MEC_Cox}. The latter follows from
Condition C3 and the fact that \(\widehat\lambda=o_p(1)\); see the proof of
Theorem~\ref{thm:Conditions_for_consistency_MEC_Cox} in the Appendix. Thus, we have
\begin{align}
\label{eq:approx_form_joint}
\widehat\xi-\xi_0
&=
-
D(\xi_0)^{-1}\Psi_n(\xi_0)
=
\sum_{i=1}^n
\phi_i + o_p(n^{-1/2}),
\end{align}
where $\phi_i
=
-
D(\xi_0)^{-1}
\Psi_i(\xi_0)$ denotes the subject-level influence contribution. We estimate \(D(\xi_0)\) \eqref{eq:taylor_expansion_joint} by the empirical derivative matrix
\begin{align}
\label{eq:D_hat}
\widehat D
=
\left.
\frac{\partial}{\partial \xi^\top}
\sum_{i=1}^n
\begin{pmatrix}
\eta_i(\theta,\lambda)\\
\rho_i(\lambda)
\end{pmatrix}
\right|_{\xi=\widehat\xi}
=
\begin{pmatrix}
\widehat D_{\theta\theta} & \widehat D_{\theta\lambda}\\
0 & \widehat D_{\lambda\lambda}
\end{pmatrix}
\in \mathbb R^{(p+1)\times(p+1)}.
\end{align}
Here, because \(\rho_i(\lambda)\) does not depend on \(\theta\), the matrix $\widehat D$ has the block upper-triangular form with
\[
\widehat D_{\theta\theta}
=
\left.
\frac{\partial U_{\rm MEC,\theta}(\theta,\widehat\lambda)}
{\partial\theta}
\right|_{\theta=\widehat\theta_{\rm MEC}}
\in\mathbb R,
\quad
\widehat D_{\theta\lambda}
=
\left.
\frac{\partial U_{\rm MEC,\theta}(\widehat\theta_{\rm MEC},\lambda)}
{\partial\lambda^\top}
\right|_{\lambda=\widehat\lambda}
\in\mathbb R^{1\times p},
\]
and
\[
\widehat D_{\lambda\lambda}
=
\left.
\frac{\partial F(\lambda)}
{\partial\lambda^\top}
\right|_{\lambda=\widehat\lambda}
=
H_0^\top
\operatorname{diag}
\left[
\frac{1}{g'\{\widehat w_i\}}:i\in\mathcal I_0
\right]
H_0
\in\mathbb R^{p\times p},
\]
where the matrix \(H_0\) is defined in Section \ref{subsec:dual_newton_solver}.
In implementation, \(\widehat D_{\theta\theta}\) and \(\widehat D_{\theta\lambda}\) may be computed by numerical differentiation of the weighted Cox score, whereas \(\widehat D_{\lambda\lambda}\) is obtained directly from the dual Newton solver (see Algorithm \ref{alg:mec_cox} in the Appendix). From \eqref{eq:approx_form_joint}, the estimated subject-level contribution to
\(\widehat\xi-\xi_0\) is
\begin{align}
\label{eq:estimated_subject_level_contribution}
\widehat\phi_i
=
-\widehat D^{-1}\widehat\Psi_i,
\qquad
\widehat\Psi_i
=
\Psi_i(\widehat\xi)
=
\begin{pmatrix}
\eta_i(\widehat\theta_{\rm MEC},\widehat\lambda)\\
\rho_i(\widehat\lambda)
\end{pmatrix}
=
\begin{pmatrix}
\widehat\eta_i\\
\widehat\rho_i
\end{pmatrix}.
\end{align}

Thus, the empirical stacked sandwich covariance estimator for \(\widehat\xi=(\widehat\theta_{\rm MEC},\widehat\lambda)\) \eqref{eq:stacked_parameter} is
\begin{align}
\nonumber
\widehat V(\widehat\xi)
&=
\sum_{i=1}^n
\widehat\phi_i\widehat\phi_i^\top
=
\sum_{i=1}^n
\left\{
-\widehat D^{-1}\widehat\Psi_i
\right\}
\left\{
-\widehat D^{-1}\widehat\Psi_i
\right\}^{\top}\\
\nonumber
&=
\widehat D^{-1}
\left\{
\sum_{i=1}^n
\widehat\Psi_i\widehat\Psi_i^\top
\right\}
(\widehat D^{-1})^\top
=
\widehat D^{-1}\widehat B(\widehat D^{-1})^\top
,
\end{align}
where $\widehat B
=
\sum_{i=1}^n
\widehat\Psi_i\widehat\Psi_i^\top.$ Therefore, the variance estimator for \(\widehat\theta_{\rm MEC}\) is
\begin{align}
\label{eq:MEC_variance}
\widehat{\operatorname{Var}}(\widehat\theta_{\rm MEC})
=
e_1^\top
\widehat V(\widehat\xi)
e_1,\qquad
e_1=(1,0,\ldots,0)^\top\in\mathbb R^{p+1}.
\end{align}
A Wald-type confidence interval for the ATT marginal log-hazard ratio $\theta_{ATT}$ \eqref{eq:marginal_structural_model_for_HR_ATT} is
\[
\widehat\theta_{\rm MEC}
\pm
z_{\alpha/2}\widehat{\operatorname{se}}(\widehat\theta_{\rm MEC}),\quad
\widehat{\operatorname{se}}(\widehat\theta_{\rm MEC})
=
\left\{
e_1^\top
\widehat V(\widehat\xi)
e_1
\right\}^{1/2},
\]
and the corresponding confidence interval for the ATT marginal hazard ratio \(HR_{ATT}\)
is obtained by exponentiating the resulting interval: $\exp [
\widehat\theta_{\rm MEC}
\pm
z_{\alpha/2}
\widehat{\operatorname{se}}(\widehat\theta_{\rm MEC})
].$

This variance estimator in \eqref{eq:MEC_variance} is a calibration analogue of the corrected sandwich estimator for IPW Cox regression proposed by \citet{shu2021variance}. More specifically, in the IPW setting, the authors stack the weighted Cox score with the propensity-score estimating equation obtained from a parametric logistic regression model to account for uncertainty from weight estimation. In MEC-Cox, the corresponding weight-estimation step is the Bregman calibration step; therefore, we stack the weighted Cox score with the dual calibration equation \eqref{eq:dual_calibration_equation}. The resulting sandwich estimator accounts for the first-order effect of estimating \(\lambda\), while treating the cross-fitted nuisance functions used to construct \(\widehat d_i\) and \(\widehat h_i\) as fixed in this sandwich calculation.

\subsection{MEC-induced efficiency gain relative to fixed-weight robust variance under a projection condition}
\label{subsec:fixed_weight_projection}

We next give a projection condition under which the MEC-Cox stacked sandwich
variance is no larger than the fixed-weight Lin--Wei robust variance. The key
idea is that calibration removes the component of Cox score variation explained
by the calibration contribution. By the block upper-triangular form of \(\widehat D\) \eqref{eq:D_hat}, the first component of the stacked influence contribution \eqref{eq:estimated_subject_level_contribution} can be written as $\widehat\phi_i^{\rm MEC}
=
-\widehat D_{\theta\theta}^{-1}
\{
\widehat\eta_i
-
\widehat C_{\rm lin}\widehat\rho_i
\}$ with $\widehat C_{\rm lin}
=
\widehat D_{\theta\lambda}
\widehat D_{\lambda\lambda}^{-1}.$ Therefore, the MEC-Cox sandwich variance in \eqref{eq:MEC_variance} can be re-expressed as
\begin{align}
\label{eq:MEC_variance_decomposition_form}
\widehat{\operatorname{Var}}_{\rm MEC}
(\widehat\theta_{\rm MEC})
=
\sum_{i=1}^n
\left(
\widehat\phi_i^{\rm MEC}
\right)^2
=
\widehat D_{\theta\theta}^{-1}
\left\{
\sum_{i=1}^n
\left(
\widehat\eta_i
-
\widehat C_{\rm lin}\widehat\rho_i
\right)^2
\right\}
(\widehat D_{\theta\theta}^{-1})^\top.
\end{align}
In contrast, if the calibrated weights are treated as fixed, so that the
calibration equation is ignored, then the LW variance uses only the Cox
score contribution:
\begin{align}
\label{eq:LW_variance_decomposition_form}
\widehat{\operatorname{Var}}_{\rm LW}
(\widehat\theta_{\rm MEC})
=
\widehat D_{\theta\theta}^{-1}
\left\{
\sum_{i=1}^n
\widehat\eta_i^2
\right\}
(\widehat D_{\theta\theta}^{-1})^\top.
\end{align}
Thus, the difference between the two variance estimators is governed by whether
\(\widehat C_{\rm lin}\widehat\rho_i\) removes a genuine calibration-explained component of
\(\widehat\eta_i\). 

The following theorem makes this statement precise.

\begin{theorem} \label{thm:mec_projection_gain}
Let \(\widehat\eta_i\) denote the empirical Lin--Wei Cox contribution in
\eqref{eq:Lin_wei_Binder_contribution}, and let \(\widehat\rho_i\) denote the
empirical contribution from the dual calibration equation in
\eqref{eq:MEC_Cox_contribution_at_solution_optimized_dual_parameter}. Define
\(\widehat B_{\theta\lambda}=\sum_{i=1}^n\widehat\eta_i\widehat\rho_i^\top\) and
\(\widehat B_{\lambda\lambda}=\sum_{i=1}^n\widehat\rho_i\widehat\rho_i^\top\),
and suppose that \(\widehat B_{\lambda\lambda}\) is nonsingular. The empirical
least-squares projection coefficient of the Cox contribution onto the calibration
contribution is
\begin{align}
\label{eq:projection_condition}
\widehat C_{\rm proj}
=
\arg\min_{C\in\mathbb R^{1\times p}}
\sum_{i=1}^n
\bigl(\widehat\eta_i-C\widehat\rho_i\bigr)^2
=
\widehat B_{\theta\lambda}\widehat B_{\lambda\lambda}^{-1}.
\end{align}
Suppose that the stacked sandwich linearization uses the same correction
coefficient, namely
\(\widehat C_{\rm lin}
=
\widehat D_{\theta\lambda}\widehat D_{\lambda\lambda}^{-1}
=
\widehat C_{\rm proj}\).

\noindent Then the MEC-Cox sandwich variance estimator \eqref{eq:MEC_variance_decomposition_form} is no larger than the fixed-weight LW sandwich variance estimator \eqref{eq:LW_variance_decomposition_form}: $\widehat{\operatorname{Var}}_{\rm MEC}(\widehat\theta_{\rm MEC})
\le
\widehat{\operatorname{Var}}_{\rm LW}(\widehat\theta_{\rm MEC}).$ Moreover, the inequality is strict whenever the calibration contribution explains a nonzero component of the Cox contribution.
\end{theorem}

In \eqref{eq:projection_condition}, \(\widehat C_{\rm proj}\widehat\rho_i\)
represents the part of the empirical Cox contribution explained by the
calibration contribution, while
\(\widehat\eta_i-\widehat C_{\rm proj}\widehat\rho_i\) is the remaining residual
contribution. Theorem~\ref{thm:mec_projection_gain} shows that, when the stacked
linearization correction coincides with this empirical projection, the MEC-Cox
sandwich variance is no larger than the fixed-weight LW robust variance
\citep{lin1989robust,binder1992fitting}. Thus, the potential variance reduction
can be interpreted as calibration-induced residualization of Cox score variation
along the calibration directions.

\section{Simulation Studies}\label{sec:simulation_studies}
\subsection{Simulation setting}\label{subsec:simulation_setting}
\paragraph{Synthetic data.} We conduct Monte Carlo simulation studies to evaluate the finite-sample
performance of MEC-Cox for estimating the marginal ATT log-hazard ratio
\(\theta_{ATT}\) in hypothetical externally controlled single-arm trial settings. The detailed simulation procedure is described in Section~\ref{subsec:Simulation procedure}
of the Appendix. We follow the super-population approach of
\citet{austin2016variance} to compute the true value of the causal estimand \(\theta_{ATT}\); see Step~5 in
Section~\ref{subsec:Simulation procedure}.
Each
simulated dataset consists of a treated trial cohort \((A_i=1)\) and an
external-control cohort \((A_i=0)\). Let
\(X_i=(X_{i1},\ldots,X_{iM})\in\mathbb R^M\) denote the \(M\)-dimensional baseline
covariate vector, generated from a multivariate standard normal distribution.

Source membership is generated according to
\begin{align}
\label{eq:true_ps_model_sim}
\operatorname{logit}\{\pi(X_i)\}
=
-0.2+\ell_\pi(X_i)+\kappa_\pi r_\pi(X_i),    
\end{align}
where
\(\ell_\pi(X_i)
=0.75X_{i1}+0.75X_{i2}+0.65X_{i3}+0.65X_{i4}+0.55X_{i5}\)
is the linear component, and $r_\pi(X_i)
=0.70\sin(1.25X_{i1})+0.45(X_{i2}^2-1)
-0.55\{I(X_{i3}>0)-0.5\}  +0.35X_{i4}X_{i5}
+0.25\{\cos(X_{i1}+X_{i2})-\exp(-1)\}$ is the nonlinear component. This mechanism induces
covariate imbalance between the treated trial and external-control cohorts.

Event times are generated from a Weibull proportional hazards model,
\begin{align}
\label{eq:true_hazard_funciton_sim}
\lambda^a(t\mid X_i)
=
\eta\lambda_0 t^{\eta-1}
\exp\{m_0(X_i)+a\beta\},
\quad
m_0(X_i)=\ell_m(X_i)+\kappa_m \cdot r_m(X_i),
\quad
a=0,1,    
\end{align}
where \(a=0\) denotes control and \(a=1\) denotes treatment. The linear
prognostic component is \(\ell_m(X_i)=X_i^\top b\), with $b=
(\log(1.75),$ $ \log(1.75), $  $ \log(1.60), $ $ \log(1.60),$ $\log(1.50),$ $
\log(1.25),$ $\ldots,$ $\log(1.25),$ $0,\ldots,0)^\top,$ and the nonlinear component is $r_m(X_i)
=
0.45\sin(X_{i2})+0.35(X_{i3}^2-1)
+0.30\{I(X_{i4}>0)-0.5\} 
 +0.25X_{i1}X_{i5}
+0.20\{\cos(X_{i2}+X_{i5})-\exp(-1)\}.$ 

Under the construction \eqref{eq:true_ps_model_sim} and \eqref{eq:true_hazard_funciton_sim}, \(X_{i1},\ldots,X_{i5}\) enter both the
source-selection and outcome models, $X_{i6},$ $\ldots,$ $X_{i10}$ enter only the
outcome model, and \(X_{i11},\ldots,X_{iM}\), when present, are noise variables. In simulation, we use \(\kappa_\pi\ge 0\) in \eqref{eq:true_ps_model_sim} and \(\kappa_m\ge 0\) in \eqref{eq:true_hazard_funciton_sim} to control the degree of
nonlinearity in the source-selection and outcome models, respectively; the
linear setting corresponds to \(\kappa_\pi=\kappa_m=0\).

Across all
scenarios, we set \(\lambda_0=0.00008\), \(\eta=2\), and
\(\beta=\log(0.70)\), corresponding to a constant conditional hazard ratio of
\(0.70\). Independent censoring times are generated as
\(C_i\sim \operatorname{Exp}(0.0008)\), so that \(E(C_i)=1250\). The observed
survival data are \(Y_i=\min(T_i,C_i)\) and \(\delta_i=I(T_i\le C_i)\). The target parameter is the marginal ATT log-hazard ratio \(\theta_{ATT}\), not
the conditional log-hazard ratio \(\beta\). 

\paragraph{Competing methods.}
MEC-Cox is the proposed method and is implemented as described in
Subsection~\ref{subsec:Applying MEC to weighted Cox regression}, with variance
estimation carried out as detailed in Section~\ref{sec:variance_mec_cox}. We use
the KL generator as the default Bregman generator. Both the source
propensity-score estimator and the calibration-basis estimator are constructed
using \(K=10\)-fold cross-fitting. The nuisance setup for MEC-Cox is
scenario-specific and is described below.

We compare MEC-Cox with standard ATT-IPW Cox estimators based on
parametric logistic-regression source-propensity-score weights. The ATT-IPW
comparators share the same point estimator and differ only in variance
estimation: the naive model-based variance, denoted by ``Naive''
\citep{austin2016variance,shu2021variance}; the Lin--Wei/Binder robust
sandwich variance, denoted by ``Robust sandwich''
\citep{lin1989robust,binder1992fitting}; and the Shu corrected sandwich
variance, denoted by ``Corrected sandwich'' \citep{shu2021variance}. These variance estimators are briefly reviewed in Subsection~\ref{subsec:Existing ATT-IPW Cox estimator and variance estimation methods}, and their known finite-sample properties in the literature are summarized in
Subsection~\ref{subsec:Comparison_methods} of the Appendix.

\paragraph{Simulation scenarios and MEC-Cox nuisance setup.} We consider two representative simulation scenarios and specify the nuisance setup used to implement MEC-Cox as follows:
\begin{enumerate}[leftmargin=1.5em, itemsep=0.35em]
    \item[] \textbf{Scenario 1: Linear source-selection and outcome models with varying external-control sample sizes.}
The source propensity-score model and the outcome model are both correctly
specified linear models, corresponding to \(\kappa_\pi=\kappa_m=0\), with
\(M=50\) baseline covariates. We vary the relative size of the external-control
cohort through \(n_1:n_0\in\{1:2,1:3,1:4\}\).

For MEC-Cox, the baseline ATT transport weights are obtained from cross-fitted
logistic regression, and the calibration basis is the cross-fitted landmark
survival basis
\begin{align}
\label{eq:5_dimen_landmark_basis}
\widehat h_i
=
\left(1,\widehat S_0^{(-k)}(t_1\mid X_i),\ldots,
\widehat S_0^{(-k)}(t_5\mid X_i)\right),
\end{align}
where the control-survival probabilities are estimated from a Cox model fit to
external controls \citep{cox1972regression,cox1975partial}. The landmark times
are \(t_\ell=\widehat Q_{0,\delta=1}(\tau_\ell)\), \(\ell=1,\ldots,5\), with
\((\tau_1,\ldots,\tau_5)=(0.10,0.30,0.50,0.70,0.90)\), where
\(\widehat Q_{0,\delta=1}(\tau)\) denotes the empirical \(\tau\)-quantile of
\(Y_i\) among external-control subjects with \(\delta_i=1\).
  
\item[] \textbf{Scenario 2: Source-selection and outcome models with increasing nonlinearity.}
The degrees of nonlinearity in both the source-selection and outcome models are
varied, with \(M=10\) and the treated-to-external-control sample-size ratio
fixed at \(n_1:n_0=1:4\). We consider three settings:
\((\kappa_\pi,\kappa_m)=(0,0)\), representing no nonlinearity;
\((\kappa_\pi,\kappa_m)=(1,2)\), representing mild nonlinearity; and
\((\kappa_\pi,\kappa_m)=(2,5)\), representing severe nonlinearity.

For MEC-Cox, the baseline ATT transport weights are obtained from cross-fitted
BART estimation of the source propensity score \citep{chipman2010bart}, and the
KL generator is used for Bregman calibration. The calibration basis uses the
same cross-fitted landmark-survival form as in
\eqref{eq:5_dimen_landmark_basis}. We consider two MEC-Cox variants, in which
the landmark control-survival probabilities are estimated using either a linear
Cox model or an RSF fitted to the external-control data
\citep{ishwaran2008random}.
\end{enumerate}

The two scenarios target complementary settings. Scenario~1 evaluates
finite-sample efficiency under a correctly specified parametric source
propensity-score model, so differences mainly reflect gains from prognostic
calibration. Scenario~2 considers nonlinear source-selection and outcome models,
where standard ATT-IPW Cox estimators may be biased by logistic propensity-score
misspecification; in this setting, flexible ML-based nuisance estimation in
MEC-Cox can become increasingly beneficial as nonlinearity increases.

\paragraph{Performance metrics.} For each scenario, we use \(R=1000\) replications and summarize performance using empirical coverage of the nominal 95\% Wald confidence interval, Monte Carlo bias, and root mean squared error (RMSE), all on the log-hazard-ratio scale. Further details on the data-generating mechanisms, nuisance-parameter estimation, cross-fitting setting, calibration-basis construction, and additional simulation experiments are provided in the Appendix.

\subsection{Scenario 1: Linear models with varying external-control sample size}
\label{subsec:simulation_linear_models}

Figure~\ref{fig:scenario_1_main_OR_linear_PS_linear} summarizes the results for
Scenario~1, in which all methods are implemented under correctly specified
linear source-selection and outcome models with \(M=50\) baseline covariates.
Panels~(a)--(c), (d)--(f), and (g)--(i) correspond to
\(n_1:n_0=1:2\), \(1:3\), and \(1:4\), respectively. Within each row, the
panels report empirical coverage, Monte Carlo bias, and RMSE. The proposed method is labeled MEC-Cox (KL--GLM(PS)--Cox(OR)).

\begin{figure}[h!]
    \centering
    \includegraphics[width=0.9\linewidth]{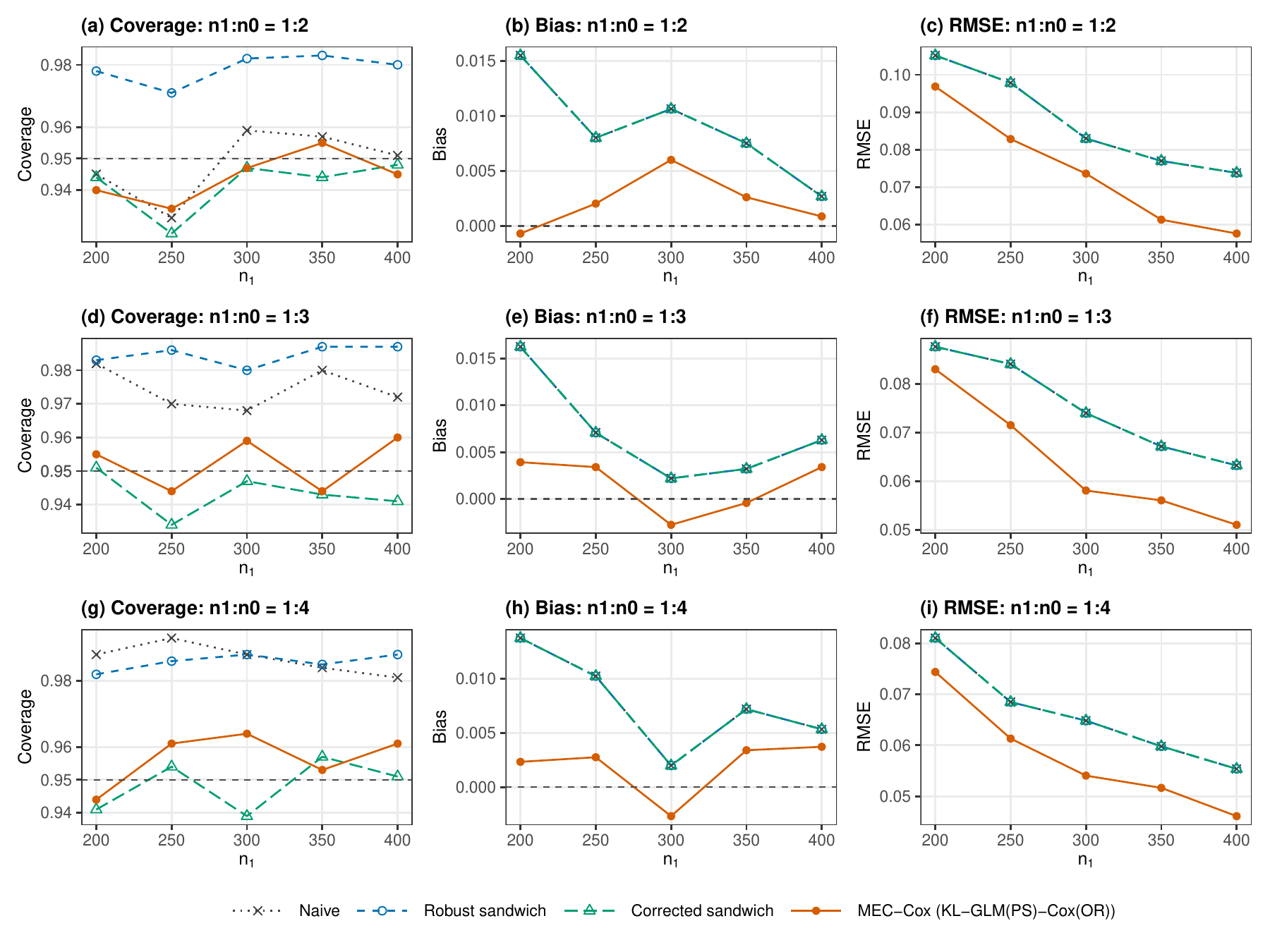}
    \caption{Simulation results for Scenario 1 with \(\kappa_\pi=\kappa_m=0\) and \(M=50\). Panels (a)--(c), (d)--(f), and (g)--(i) correspond to \(n_1:n_0=1:2\), \(1:3\), and \(1:4\), respectively. MEC-Cox uses logistic regression for source propensity-score estimation and the landmark survival calibration basis in \eqref{eq:5_dimen_landmark_basis}, constructed from a linear Cox model fitted to the external-control data.}
    \label{fig:scenario_1_main_OR_linear_PS_linear}
\end{figure}

Across all three external-control sample-size ratios, MEC-Cox achieves smaller
bias and RMSE than the standard ATT-IPW Cox estimators, while maintaining
coverage close to the nominal level. Because the source-selection and outcome
models are correctly specified in this scenario, all estimators are expected to
be consistent. Thus, the main performance difference reflects the efficiency
gain from the additional prognostic calibration step rather than correction of
model misspecification. The improvement of MEC-Cox is especially visible in
terms of RMSE, indicating that balancing the Cox-based landmark survival
calibration basis can improve finite-sample efficiency even when the baseline
logistic source propensity-score model is correctly specified. These results
support the use of MEC-Cox as an efficiency-enhancing refinement of standard
ATT-IPW Cox estimators.

\subsection{Scenario 2: Nonlinear models with increasing nonlinearity}
\label{subsec:simulation_nonlinear_models}

Figure~\ref{fig:scenario_2_main_OR_nonlinear_PS_nonlinear} summarizes the
results for Scenario 2, in which the degrees of nonlinearity in both the
source-selection and outcome models are varied. Panels (a)--(c), (d)--(f), and
(g)--(i) correspond to the no-, mild-, and severe-nonlinearity settings,
respectively, with \(M=10\) and \(n_1:n_0=1:4\). The proposed method is labeled MEC-Cox (KL--BART(PS)--Cox(OR)) or MEC-Cox (KL--BART(PS)--RSF(OR)), which differ only in the outcome-regression learner used to construct the calibration basis.

\begin{figure}[h!]
    \centering
    \includegraphics[width=0.90\linewidth]{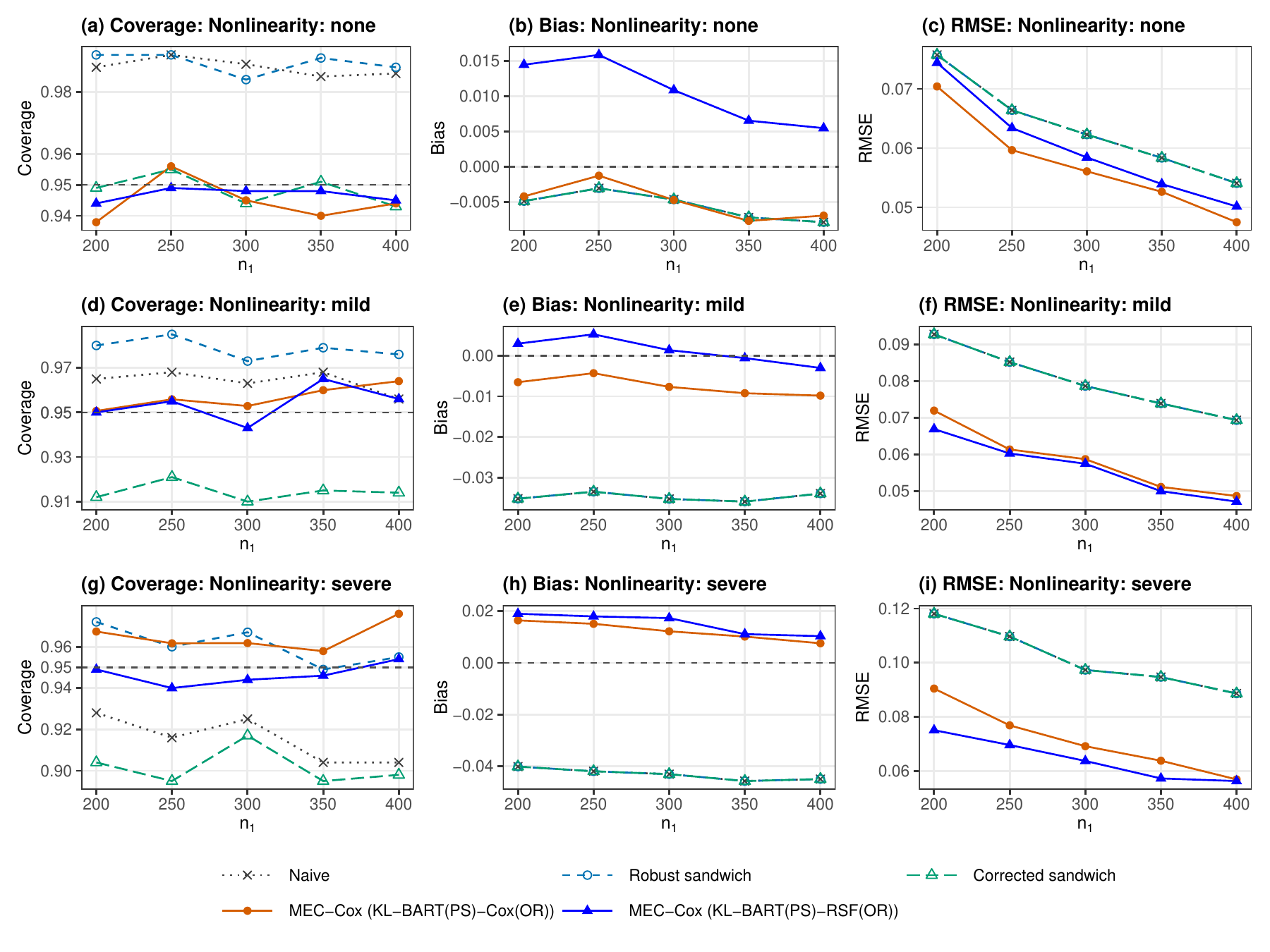}
    \caption{Simulation results for Scenario 2 with \(M=10\) and \(n_1:n_0=1:4\). Panels (a)--(c), (d)--(f), and (g)--(i) correspond to the no-, mild-, and severe-nonlinearity settings, respectively, with \((\kappa_\pi,\kappa_m)=(0,0)\), \((\kappa_\pi,\kappa_m)=(1,2)\), and \((\kappa_\pi,\kappa_m)=(2,5)\). MEC-Cox uses BART for source propensity-score estimation and either linear-Cox- or RSF-based landmark survival calibration bases.}
    \label{fig:scenario_2_main_OR_nonlinear_PS_nonlinear}    
\end{figure}
Across all three settings, MEC-Cox achieves smaller bias and RMSE than the
standard ATT-IPW Cox estimators. When the true prognostic model is linear
\((\kappa_m=0)\), MEC-Cox with the linear-Cox-based landmark survival
calibration basis performs best, as expected (Panel~(c)). In this setting, the
linear Cox model correctly captures the main prognostic structure, whereas RSF
introduces unnecessary flexibility and may be less efficient in finite samples.
However, as the degree of nonlinearity increases (Panels~(f) and~(i), with
\(\kappa_m=2\) and \(\kappa_m=5\), respectively), MEC-Cox with the RSF-based
landmark survival calibration basis performs best. This is also expected,
because the flexibility of RSF allows the calibration basis to capture nonlinear
prognostic information that the linear-Cox-based calibration basis may miss. In
contrast, the standard ATT-IPW Cox estimators do not use the additional
prognostic calibration step and rely on logistic-regression source propensity
score estimation, which is misspecified under nonlinear source-selection
mechanisms. Consequently, these estimators exhibit larger bias and larger
finite-sample error under nonlinear settings. These results suggest that MEC-Cox can effectively leverage ML-based prognostic
summaries to improve finite-sample performance when linear nuisance
specifications for source selection and prognosis are inadequate.

This experiment also supports the theoretical implication of
Theorem~\ref{thm:Conditions_for_consistency_MEC_Cox}. In the nonlinear settings
shown in the second and third rows of
Figure~\ref{fig:scenario_2_main_OR_nonlinear_PS_nonlinear}, the
linear-Cox-based landmark survival calibration basis is misspecified as a
prognostic model. However, the theorem does not require correct specification of
the calibration basis for consistency of MEC-Cox. Rather, for the calibration
basis, the key consistency requirement is \(L_2\)-stochastic boundedness, which
is a mild condition. Therefore, even when the linear-Cox-based landmark survival
calibration basis is \emph{not} correctly specified, MEC-Cox remains a valid
estimator of the ATT marginal log-hazard ratio \(\theta_{ATT}\), provided that
the source propensity score is consistently estimated and the calibration
regularity conditions hold. Thus, the role of the calibration basis is primarily related to efficiency rather than validity.

In practical applications, because the degree of nonlinearity in the
source-selection and prognostic mechanisms is unknown, it is often prudent to
use flexible ML methods, such as BART \citep{chipman2010bart} or DL \citep{lecun2015deep}, for source propensity-score estimation. The choice of calibration basis can then be guided by the available external-control sample size \(n_0\) and the number of baseline covariates
\(M\), since flexible survival learners require sufficient information to train
stable prognostic summaries and to avoid overfitting in high-dimensional
settings. Flexible survival learners such as RSF may be preferable when the
external-control sample size $n_0$ is sufficiently large relative to \(M\), whereas a
simpler Cox-based basis may be more stable for small to moderate
external-control sample sizes or when \(M\) is too large to train flexible
survival learners reliably. In the latter case, particularly when \(M\) is large
(for example, \(M=200\)) and a sparse linear prognostic structure is expected,
we recommend using a lasso-Cox linear-predictor calibration basis rather than a
landmark survival basis; see the additional simulation study in the Appendix.



\section{Real-data Application}
\label{sec:real_data}
We illustrate MEC-Cox using two publicly available breast-cancer datasets from the \texttt{survival} package in \texttt{R}. These datasets have been widely used as benchmark datasets in survival-analysis research \citep{kvamme2019time,chen2020deep}. The treated target set \(\mathcal I_1=\{i:A_i=1\}\) consisted of German Breast Cancer Study Group (GBSG) patients who received hormonal therapy, whereas the external-control set \(\mathcal I_0=\{i:A_i=0\}\) consisted of Rotterdam tumor-bank patients who did not receive hormonal therapy. This restriction gives a clinically interpretable comparison between hormonal therapy and no hormonal therapy in the GBSG hormonal-therapy target population. The final analysis cohort contained \(n=2889\) patients, with \(n_1=|\mathcal I_1|=246\) and \(n_0=|\mathcal I_0|=2643\). The endpoint was recurrence-free survival (RFS), administratively censored at 5 years, where \(Y_i\) and \(\delta_i\) denote the observed follow-up time and RFS event indicator, respectively.

The covariate vector \(X_i\in\mathbb R^M\) was constructed from seven baseline variables: age, menopausal status, tumor size, tumor grade, number of positive lymph nodes, progesterone receptor (PGR) level, and estrogen receptor (ER) level. Tumor size was categorized as \(\leq 20\) mm, 20--50 mm, or \(>50\) mm, and nodes, PGR, and ER were transformed using \(\log(1+x)\). Categorical variables were encoded as indicator variables in the fitted models; under this representation, \(M=11\). Only covariates without missing values in the combined analysis cohort were used. The primary fitted models used main-effect terms only; selected clinically motivated nonlinear and interaction terms, particularly involving tumor grade and lymph-node burden, were considered in sensitivity analyses. The target causal estimand was the ATT marginal hazard ratio \(HR_{ATT}=\exp(\theta_{ATT})\), comparing hormonal therapy versus no hormonal therapy in the GBSG hormonal-therapy target population.

We compare the unweighted Cox analysis, the standard ATT-IPW Cox estimator, and
MEC-Cox. For the ATT-IPW Cox estimator, we report the same weighted Cox point
estimate with the three variance estimators used in the simulation studies in
Section~\ref{sec:simulation_studies}. For MEC-Cox, we consider two
implementations based on the KL generator and the \(K=10\) cross-fitted
landmark-survival basis in \eqref{eq:5_dimen_landmark_basis}, evaluated at 20
landmark times; thus, including the intercept, the calibration-basis dimension
is \(p=21\). One implementation uses logistic regression for the source propensity score and
Cox regression for the prognostic basis, whereas the other uses DL for the
source propensity score and RSF for the prognostic basis. We denote these two
implementations by MEC-Cox (GLM(PS)--Cox(OR)) and MEC-Cox
(DL(PS)--RSF(OR)), respectively.

Tables~\ref{tab:realdata_hr_results}
and~\ref{tab:realdata_balance_weight_diagnostics} summarize the analysis
results. First, the unweighted Cox analysis yielded an estimated hazard ratio of
0.963, with a 95\% confidence interval including 1. This unadjusted comparison
would therefore suggest no clear difference between the GBSG hormonal-therapy
group and the Rotterdam external-control group. However, this comparison is not
interpretable as an ATT causal estimate because the two cohorts exhibited
substantial baseline imbalance, as shown in
Table~\ref{tab:realdata_balance_weight_diagnostics}. The mean and maximum
absolute standardized mean differences (SMDs) were 0.303 and 1.158,
respectively, indicating strong raw covariate imbalance. This motivates the use of ATT weighting to transport the external-control cohort toward the GBSG hormonal-therapy target population before estimating the marginal hazard ratio $\theta_{ATT}$. 

\begin{table}[h!]
\centering
\caption{Marginal hazard-ratio estimates in the real-data illustration.}
\label{tab:realdata_hr_results}
\small
\setlength{\tabcolsep}{3.5pt}
\renewcommand{\arraystretch}{1.12}

\begin{tabular*}{\textwidth}{@{\extracolsep{\fill}}lrrrr@{}}
\toprule
\textsc{Method}
&
\(\widehat\theta\)
&
\(\widehat{\mathrm{se}}(\widehat\theta)\)
&
\(\widehat{HR}\)
&
95\% CI of HR
\\
\midrule
Unweighted Cox
&
\(-0.037\)
&
0.112
&
0.963
&
\((0.773,\ 1.200)\)
\\

Naive
&
\(-0.564\)
&
0.136
&
0.569
&
\((0.436,\ 0.743)\)
\\

Robust sandwich
&
\(-0.564\)
&
0.139
&
0.569
&
\((0.433,\ 0.747)\)
\\

Corrected sandwich
&
\(-0.564\)
&
0.130
&
0.569
&
\((0.441,\ 0.734)\)
\\

MEC-Cox (GLM(PS)-Cox(OR))
&
\(-0.498\)
&
0.124
&
0.608
&
\((0.476,\ 0.775)\)
\\

MEC-Cox (DL(PS)-RSF(OR))
&
\(-0.438\)
&
0.117
&
0.646
&
\((0.513,\ 0.813)\)
\\

\bottomrule
\end{tabular*}

\vspace{0.25em}
\begin{minipage}{\textwidth}
\footnotesize
\textit{Note.}
Notation \(\widehat\theta\) denotes the estimated log-hazard ratio and
\(\widehat{\mathrm{HR}}=\exp(\widehat\theta)\). The unweighted Cox analysis is included
as an unadjusted descriptive comparison and is not interpreted as an estimator of the
ATT marginal hazard ratio. The Naive, Robust sandwich, Corrected sandwich, and
MEC-Cox rows are weighted Cox analyses targeting the ATT marginal hazard ratio in the
GBSG hormonal-therapy target population. The Naive, Robust sandwich, and Corrected
sandwich rows share the same ATT-IPW Cox point estimator and differ only in variance
estimation. The MEC-Cox rows use the KL generator.
\end{minipage}
\end{table}

\begin{table}[h!]
\centering
\caption{Covariate balance and external-control weight diagnostics.}
\label{tab:realdata_balance_weight_diagnostics}
\scriptsize
\setlength{\tabcolsep}{2.5pt}
\renewcommand{\arraystretch}{1.12}

\begin{tabular*}{\textwidth}{@{\extracolsep{\fill}}llrrrr@{}}
\toprule
\textsc{Covariate}
&
\textsc{Type}
&
Unweighted
&
ATT-IPW
&
\begin{tabular}{@{}c@{}}
MEC-Cox \\
(GLM(PS)-Cox(OR))
\end{tabular}
&
\begin{tabular}{@{}c@{}}
MEC-Cox \\
(DL(PS)-RSF(OR))
\end{tabular}
\\
\midrule
age
&
C
&
0.268
&
\(-0.032\)
&
\(-0.017\)
&
\(-0.035\)
\\

meno: postmenopausal
&
B
&
0.241
&
0.008
&
0.028
&
0.053
\\

size: \(\leq 20\) mm
&
B
&
\(-0.213\)
&
0.042
&
0.027
&
\(-0.005\)
\\

size: 20--50 mm
&
B
&
0.247
&
\(-0.045\)
&
\(-0.028\)
&
0.006
\\

size: \(>50\) mm
&
B
&
\(-0.034\)
&
0.003
&
0.002
&
\(-0.001\)
\\

grade 1
&
B
&
0.134
&
0.134
&
0.134
&
0.134
\\

grade 2
&
B
&
0.385
&
\(-0.105\)
&
\(-0.104\)
&
\(-0.082\)
\\

grade 3
&
B
&
\(-0.519\)
&
\(-0.029\)
&
\(-0.030\)
&
\(-0.052\)
\\

\(\log(1+\mathrm{nodes})\)
&
C
&
1.158
&
\(-0.270\)
&
\(-0.021\)
&
0.148
\\

\(\log(1+\mathrm{PGR})\)
&
C
&
\(-0.048\)
&
0.046
&
0.003
&
\(-0.003\)
\\

\(\log(1+\mathrm{ER})\)
&
C
&
\(-0.085\)
&
0.107
&
0.047
&
\(-0.053\)
\\

\midrule
Mean \(|\mathrm{SMD}|\)
&
--
&
0.303
&
0.075
&
0.040
&
0.052
\\

Max \(|\mathrm{SMD}|\)
&
--
&
1.158
&
0.270
&
0.134
&
0.148
\\

No. \(|\mathrm{SMD}|>0.10\)
&
--
&
8
&
4
&
2
&
2
\\

External-control ESS
&
--
&
--
&
169.476
&
262.077
&
354.709
\\

External-control weight CV
&
--
&
--
&
3.821
&
3.015
&
2.540
\\

\bottomrule
\end{tabular*}

\vspace{0.25em}
\begin{minipage}{\textwidth}
\footnotesize
\emph{Note.} Type C denotes a continuous covariate and type B denotes a binary
covariate. For continuous covariates, SMD denotes the standardized mean
difference using the treated-cohort standard deviation as the denominator; for
binary covariates, it denotes the corresponding difference in proportions. The
summary rows report absolute balance statistics. ESS and weight CV are computed among external-control patients only, using the
corresponding method-specific external-control weights. Specifically,
\(\mathrm{ESS}=(\sum_{i\in \mathcal{I}_0} a_i)^2/\sum_{i\in \mathcal{I}_0} a_i^2\) and
\(\mathrm{CV}=\operatorname{sd}(a_i:i\in \mathcal{I}_0)/\bar a_{\mathcal{I}_0}\), where \(a_i\)
denotes the method-specific external-control weight. Smaller absolute balance statistics
indicate better raw-covariate balance. Larger ESS and smaller weight CV indicate
more stable external-control weights. The unweighted analysis has no
external-control weighting diagnostics.
\end{minipage}
\end{table}

After ATT weighting, the standard ATT-IPW Cox estimator yielded an estimated
hazard ratio of \(\widehat{HR}_{ATT}=0.569\). The naive,
robust sandwich, and corrected sandwich rows share the same point estimator,
\(\widehat{\theta}_{IPW}=-0.564\), and differ only in variance
estimation. Because the upper bound of the 95\% confidence interval is below 1, this
ATT-weighted comparison would therefore suggest a lower marginal hazard in the
GBSG hormonal-therapy target population than in its transported external-control
counterpart. Although ATT-IPW substantially improved raw covariate balance relative to the
unweighted analysis, some imbalance remained, with a maximum absolute SMD of
0.270. The ATT-IPW external-control weights were also relatively variable, with
an external-control effective sample size (ESS) of 169.476 and a weight
coefficient of variation (CV) of 3.821.

The MEC-Cox implementations produced estimates in the same qualitative direction
as the ATT-IPW estimator, while improving weight stability, reducing the
estimated standard errors, and improving covariate balance. MEC-Cox
(GLM(PS)--Cox(OR)) yielded an estimated hazard ratio of
\(\widehat{HR}_{MEC}=0.608\), with a standard error of
\(\widehat{\rm se}(\widehat{\theta}_{MEC})=0.124\). MEC-Cox
(DL(PS)--RSF(OR)) yielded an estimated hazard ratio of
\(\widehat{HR}_{MEC}=0.646\), with a standard error of
\(\widehat{\rm se}(\widehat{\theta}_{MEC})=0.117\). Compared with ATT-IPW, both
MEC-Cox implementations improved external-control weight stability: the
external-control ESS increased from 169.476 under ATT-IPW to 262.077 and
354.709 under the two MEC-Cox implementations, while the weight CV decreased
from 3.821 to 3.015 and 2.540, respectively. The covariate-balance diagnostics
were also favorable, with mean absolute SMDs of 0.040 and 0.052 and maximum
absolute SMDs of 0.134 and 0.148.

Overall, this real-data example illustrates MEC-Cox as a practical
prognostic-calibration refinement of standard ATT-IPW Cox regression. MEC-Cox
preserves the ATT weighting target and yields the same qualitative conclusion as
standard ATT-IPW, while producing more stable external-control weights and
smaller estimated standard errors.


\section{Discussion}\label{sec:Discussion}
We extended the MEC framework of \citet{lee2026mec} to IPW-type weighted Cox regression and proposed MEC-Cox for ATT marginal hazard-ratio estimation in externally controlled survival studies. The proposed method addresses an important gap in the literature on marginal hazard-ratio estimation \citep{hernan2001marginal}. Existing robust, survey-based, corrected, and linearization-based variance estimators account for ATT weighting and, in some cases, propensity-score estimation uncertainty \citep{lin1989robust,binder1992fitting,shu2021variance,hajage2018closed,shu2021estimating}. However, these methods were mainly developed under parametric propensity-score models and can therefore be sensitive to model misspecification, potentially leading to distorted inference. They also do not directly use prognostic summaries for calibration, although such summaries can improve efficiency when properly constructed, as illustrated in Section~\ref{sec:simulation_studies}. EIF-based methods, including TMLE \citep{van2006targeted,van2011targeted} and DML \citep{chernozhukov2018double}, provide a principled way to incorporate flexible ML nuisance estimation, but their direct application to IPW-type weighted Cox regression is not straightforward because the Cox score depends nonlinearly on the weights through risk-set averages. The MEC framework of \citet{lee2026mec} helps bridge these directions by preserving the weighted Cox estimating equation while incorporating cross-fitted ML prognostic summaries through calibration. Relatedly, weight calibration has been used to improve efficiency for pure-risk estimation under additive hazards models in nested case-control designs \citep{shin2022weight}; our use differs by targeting prognostic balance for ATT-weighted Cox regression in externally controlled comparisons.

The present work also suggests several directions for future research. The first two are theoretical. First, whether MEC-Cox attains the semiparametric efficiency bound \citep{kennedy2016semiparametric,kennedy2024semiparametric} for the ATT marginal hazard-ratio estimand is beyond the scope of this paper. For mean-type estimands, MEC can have stronger theoretical guarantees; for example, in semi-supervised mean estimation, \citet{lee2026mec} showed that MEC can attain the semiparametric efficiency bound under weaker projection-error conditions. Although an analogous property is not guaranteed for IPW Cox regression, Theorem~\ref{thm:mec_projection_gain} identifies a projection condition under which calibration can yield efficiency gains relative to the fixed-weight LW robust variance estimator. Second, the connection between MEC-Cox and doubly robust estimation \citep{bang2005doubly,luo2025doubly} remains an open question. For continuous-outcome ATE estimation, MEC can yield an IPW representation that is algebraically equivalent to a doubly robust, Neyman-orthogonal estimating equation \citep{chernozhukov2018double,mackey2018orthogonal,foster2023orthogonal}; see the Supplementary Material of \citet{lee2026mec}. For weighted Cox regression, however, such an equivalence is not trivial because of the analytical complexity of the weighted Cox score.

The latter two issues are more practical. Third, the current formulation assumes source-specific independent censoring. More realistic survival data may involve informative censoring that differs between the treated trial and external-control cohorts. Addressing such settings would require incorporating inverse probability of censoring weights \citep{cain2009inverse,robins2000correcting} into the estimating-equation weights. Finally, MEC-updated weights play a \emph{dual role}: they act as source-transport weights and as prognostic-score balancing weights, and a trade-off may arise between these two roles. Because this concept is relatively new in balancing weights for causal inference, practical implementation requires diagnostic and sensitivity-analysis tools tailored to this setting. Such tools should assess not only covariate balance, for example through SMD, but also weight stability, for example through the ESS, to support routine use.


\acks{J.K.K. was supported by the U.S. National Science Foundation under Grant No. 2242820.}


\appendix
\section{Proofs of the Main Theorems and Technical Details}
\label{sec:proof-main}
\subsection{Proof of Theorem \ref{thm:identification_weighted_cox}}\label{subsec:Proof of Theorem 1}
\paragraph{Notations.} Let \(p_1=\Pr(A_i=1)\), and write \(\overset{p}{\to}\) for convergence in probability. For \(a=0,1\), define the potential observed time, event indicator, counting process, and at-risk process by
\[
Y_i^a=\min(T_i^a,C_i), 
\qquad 
\delta_i^a=I(T_i^a\le C_i),
\]
\[
\mathcal{N}_i^a(t)=I(Y_i^a\le t,\delta_i^a=1),
\qquad
\mathcal{Y}_i^a(t)=I(Y_i^a\ge t).
\]
By consistency (A1),
\[
A_i=1 \implies \{\mathcal{N}_i(t),\mathcal{Y}_i(t)\}=\{\mathcal{N}_i^1(t),\mathcal{Y}_i^1(t)\},
\qquad
A_i=0 \implies \{\mathcal{N}_i(t),\mathcal{Y}_i(t)\}=\{\mathcal{N}_i^0(t),\mathcal{Y}_i^0(t)\}.
\]

For the trial target population \(A_i=1\), define the marginal at-risk function under potential treatment condition \(a\) by
\begin{align}
\label{eq:marginal_at_risk_function}
y_1^a(t)=\mathbb{E}\{\mathcal{Y}_i^a(t)\mid A_i=1\},
\qquad a=0,1.    
\end{align}

That is, \(y_1^a(t)\) is the probability that a randomly selected patient from the treated trial target population would remain at risk at time \(t\) under regime \(a\).

The proof below mainly uses the marginal mean behavior of the potential
counting processes in the trial target population. We begin by recalling the censoring assumptions that supplement the causal assumptions A1--A3 in the identification argument.


\paragraph{Censoring assumptions.}
Recall that, in the observed-data structure in the main paper, we assumed
factual source-specific random censoring:
\begin{align}
\nonumber
C_i \perp (T_i,X_i)\mid A_i.
\end{align}
For the identification argument below, we use the corresponding
potential-outcome version. Specifically, we assume
\begin{align}
\label{eq:censoring_assumption_identification_independence}
C_i \perp (T_i^0,T_i^1,X_i)\mid A_i .
\end{align}
This condition implies, in particular, that for each \(a=0,1\),
\[
C_i \perp T_i^a\mid A_i,
\qquad
C_i \perp X_i\mid A_i.
\]
Moreover, by the weak-union property of conditional independence,
\eqref{eq:censoring_assumption_identification_independence} also implies
\[
C_i \perp T_i^a \mid X_i,A_i,
\qquad a=0,1.
\]
Thus, within each source population, censoring is independent of the potential
event times, both marginally and conditionally on baseline covariates.

Because \(C_i\perp X_i\mid A_i\), the source-specific censoring survival
functions depend on the source indicator but not on baseline covariates. Define
\begin{align}
\nonumber
G_a(t)=\Pr(C_i\ge t\mid A_i=a)=\Pr(C_i\ge t\mid X_i, A_i=a),\qquad a=0,1,
\end{align}
and define the source-specific censoring ratio
\begin{align}
\nonumber
\gamma_C(t)=\frac{G_0(t)}{G_1(t)}.
\end{align}
Here, \(\cdot\perp\cdot\mid\cdot\) denotes conditional independence. We allow
\(G_0(t)\) and \(G_1(t)\) to differ; that is, the censoring distributions may
be source-specific.

\paragraph{Key identities.}
We now derive the key counting-process identities used in the proof. For
\(a=0,1\), the source-specific random censoring assumption gives
\begin{align}
\nonumber
\mathbb{E}\{d\mathcal{N}_i^a(t)\mid A_i=1\}
&=
\Pr\{\mathcal{N}_i^a(t+dt)-\mathcal{N}_i^a(t)=1\mid A_i=1\}  \\
\nonumber
&=
\Pr\{Y_i^a\in[t,t+dt),\delta_i^a=1\mid A_i=1\}  \\
\nonumber
&=
\Pr\{t\le T_i^a<t+dt,\ C_i\ge T_i^a\mid A_i=1\}  \\
\nonumber
&=
\Pr\{t\le T_i^a<t+dt,\ C_i\ge t\mid A_i=1\}
  +o(dt) \\
 \nonumber 
&=
\Pr(C_i\ge t\mid A_i=1)
  \Pr\{t\le T_i^a<t+dt\mid C_i\ge t, A_i=1\}
  +o(dt) \\
&=
\Pr(C_i\ge t\mid A_i=1)
  \Pr\{t\le T_i^a<t+dt\mid A_i=1\}
  +o(dt) \tag{\(\star_1\)} \\
&=
\Pr(C_i\ge t\mid A_i=1)
  \Pr(T_i^a\ge t\mid A_i=1)
  \lambda_1^a(t)\,dt
  +o(dt) \tag{\(\ast\)} \\
&=
\Pr(C_i\ge t\mid A_i=1)
  \Pr(T_i^a\ge t\mid C_i\ge t, A_i=1)
  \lambda_1^a(t)\,dt
  +o(dt) \tag{\(\star_2\)} \\
  \nonumber
&=
\Pr(C_i\ge t,\ T_i^a\ge t\mid A_i=1)
  \lambda_1^a(t)\,dt
  +o(dt) \\
  \label{eq:last_eq}
&=
\mathbb{E}\{\mathcal{Y}_i^a(t)\mid A_i=1\}\lambda_1^a(t)\,dt
  +o(dt).
\end{align}
Here, \((\star_1)\) and \((\star_2)\) use
\eqref{eq:censoring_assumption_identification_independence}. Both steps follow
because \eqref{eq:censoring_assumption_identification_independence} implies
\(C_i\perp T_i^a\mid A_i=1\). The step \((\ast)\) uses the definition of
the marginal hazard function in the trial target population:
\[
\Pr\{t\le T_i^a<t+dt\mid A_i=1\}
=
\Pr(T_i^a\ge t\mid A_i=1)\lambda_1^a(t)\,dt+o(dt).
\]
By the definition of \(y_1^a(t)\) in
\eqref{eq:marginal_at_risk_function}, \eqref{eq:last_eq} gives
\begin{align}
\label{eq:exp_d_N_i_a}
\mathbb{E}\{d\mathcal{N}_i^a(t)\mid A_i=1\}
=
y_1^a(t)\lambda_1^a(t)\,dt.    
\end{align}

By one-sided positivity (A3), \(1-\pi(X_i)>0\) on the support of the trial
target population, so the ATT odds weight $q(X_i)=\pi(X_i)/(1-\pi(X_i))$ is well defined. For any integrable function
\(f(X_i)\), we have
\begin{align}
\mathbb{E}\{A_i f(X_i)\}
&=
\mathbb{E}\!\left[
\mathbb{E}\{A_i f(X_i)\mid X_i\}
\right] \notag\\
&=
\mathbb{E}\{f(X_i)\Pr(A_i=1\mid X_i)\} \notag\\
&=
\mathbb{E}\{\pi(X_i)f(X_i)\} \notag\\
&=
\mathbb{E}\!\left[
\frac{\pi(X_i)}{1-\pi(X_i)}
\{1-\pi(X_i)\}f(X_i)
\right] \notag\\
&=
\mathbb{E}\!\left[
q(X_i)\Pr(A_i=0\mid X_i)f(X_i)
\right] \notag\\
&=
\mathbb{E}\!\left[
q(X_i)f(X_i)\mathbb{E}\{1-A_i\mid X_i\}
\right] \notag\\
&=
\mathbb{E}\!\left[
\mathbb{E}\{(1-A_i)q(X_i)f(X_i)\mid X_i\}
\right] \notag\\
&=
\mathbb{E}\{(1-A_i)q(X_i)f(X_i)\}.
\label{eq:core_equation}
\end{align}
Thus, \eqref{eq:core_equation} implies that the ATT odds weight \(q(X_i)\) reweights the external-control covariate
distribution to the treated trial target covariate distribution.

The same weighting argument extends to the control potential-outcome processes.
By consistency (A1), external-control patients satisfy
\[
(1-A_i)\mathcal{Y}_i(t)=(1-A_i)\mathcal{Y}_i^0(t),
\qquad
(1-A_i)d\mathcal{N}_i(t)=(1-A_i)d\mathcal{N}_i^0(t).
\]
Define the conditional control-risk and control-event means in the trial target
population by
\[
m_\mathcal{Y}^0(t,X_i)=\mathbb{E}\{\mathcal{Y}_i^0(t)\mid X_i,A_i=1\},
\qquad
m_\mathcal{N}^0(t,X_i)=\mathbb{E}\{d\mathcal{N}_i^0(t)\mid X_i,A_i=1\}.
\]

By A2 (survival transportability) and the source-specific random censoring
condition \eqref{eq:censoring_assumption_identification_independence}, we have
\begin{align}
\mathbb{E}\{\mathcal{Y}_i^0(t)\mid X_i,A_i=0\}
&=
\Pr(T_i^0\ge t,\ C_i\ge t\mid X_i,A_i=0) \notag\\
&=
\Pr(C_i\ge t\mid X_i,A_i=0)
\Pr(T_i^0\ge t\mid C_i\ge t, X_i,A_i=0) \notag\\
&=
\Pr(C_i\ge t\mid X_i,A_i=0)
\Pr(T_i^0\ge t\mid X_i,A_i=0) \tag{\(\star_1\)}\\
&=
G_0(t)\Pr(T_i^0\ge t\mid X_i,A_i=0) \notag\\
&=
G_0(t)\Pr(T_i^0\ge t\mid X_i,A_i=1) \tag{\(\star_2\)}\\
&=
\frac{G_0(t)}{G_1(t)}
G_1(t)\Pr(T_i^0\ge t\mid X_i,A_i=1) \notag\\
&=
\gamma_C(t)
\Pr(C_i\ge t,\ T_i^0\ge t\mid X_i,A_i=1) \tag{\(\star_3\)}\\
&=
\gamma_C(t)\mathbb{E}\{\mathcal{Y}_i^0(t)\mid X_i,A_i=1\} \notag\\
\nonumber
&=
\gamma_C(t)m_\mathcal{Y}^0(t,X_i).
\end{align}
Here, \((\star_1)\) and \((\star_3)\) use source-specific independent
censoring \eqref{eq:censoring_assumption_identification_independence}. The
step \((\star_2)\) uses A2 for the event-time part:
\[
\Pr(T_i^0\ge t\mid X_i,A_i=0)
=
\Pr(T_i^0\ge t\mid X_i,A_i=1).
\]
The factor \(\gamma_C(t)\) appears because the censoring distributions are
allowed to differ across the two sources.

Similarly,
\begin{align}
\mathbb{E}\{d\mathcal{N}_i^0(t)\mid X_i,A_i=0\}
&=
\Pr\{\mathcal{N}_i^0(t+dt)-\mathcal{N}_i^0(t)=1\mid X_i,A_i=0\} \notag\\
&=
\Pr\{\mathcal{Y}_i^0\in[t,t+dt),\delta_i^0=1\mid X_i,A_i=0\} \notag\\
&=
\Pr\{t\le T_i^0<t+dt,\ C_i\ge T_i^0\mid X_i,A_i=0\} \notag\\
&=
\Pr\{t\le T_i^0<t+dt,\ C_i\ge t\mid X_i,A_i=0\}
+o(dt) \notag\\
&=
\Pr(C_i\ge t\mid X_i,A_i=0)
\Pr\{t\le T_i^0<t+dt\mid X_i,A_i=0\}
+o(dt) \tag{\(\star_1\)}\\
&=
G_0(t)
\Pr\{t\le T_i^0<t+dt\mid X_i,A_i=1\}
+o(dt) \tag{\(\star_2\)}\\
&=
\frac{G_0(t)}{G_1(t)}
G_1(t)
\Pr\{t\le T_i^0<t+dt\mid X_i,A_i=1\}
+o(dt) \notag\\
&=
\gamma_C(t)
\Pr\{t\le T_i^0<t+dt,\ C_i\ge t\mid X_i,A_i=1\}
+o(dt) \tag{\(\star_3\)}\\
&=
\gamma_C(t)
\Pr\{t\le T_i^0<t+dt,\ C_i\ge T_i^0\mid X_i,A_i=1\}
+o(dt) \notag\\
&=
\gamma_C(t)\mathbb{E}\{d\mathcal{N}_i^0(t)\mid X_i,A_i=1\}
+o(dt) \notag\\
\nonumber
&=
\gamma_C(t)m_\mathcal{N}^0(t,X_i)+o(dt).
\end{align}
Here, \((\star_1)\) and \((\star_3)\) use source-specific independent
censoring \eqref{eq:censoring_assumption_identification_independence}. The
step \((\star_2)\) uses A2 for the latent control event-time increment:
\[
\Pr\{t\le T_i^0<t+dt\mid X_i,A_i=0\}
=
\Pr\{t\le T_i^0<t+dt\mid X_i,A_i=1\}.
\]
Therefore, in the usual counting-process differential notation, we write
\[
\mathbb{E}\{d\mathcal{N}_i^0(t)\mid X_i,A_i=0\}
=
\gamma_C(t)m_\mathcal{N}^0(t,X_i).
\]

Therefore,
\begin{align}
\nonumber
\mathbb{E}\{(1-A_i)q(X_i)\mathcal{Y}_i(t)\}
&=
\mathbb{E}\{(1-A_i)q(X_i)\mathcal{Y}_i^0(t)\} \\
\nonumber
&=
\mathbb{E}\{(1-A_i)q(X_i)\gamma_C(t)m_\mathcal{Y}^0(t,X_i)\} \\
\nonumber
&=
\gamma_C(t)\mathbb{E}\{(1-A_i)q(X_i)m_\mathcal{Y}^0(t,X_i)\} \\
\nonumber
&=
\gamma_C(t)\mathbb{E}\{A_i m_\mathcal{Y}^0(t,X_i)\} \qquad\qquad\qquad\qquad\qquad \text{by } \eqref{eq:core_equation} \\
\nonumber
&=
\gamma_C(t)\mathbb{E}\!\left[
A_i\mathbb{E}\{\mathcal{Y}_i^0(t)\mid X_i,A_i=1\}
\right] \\
\nonumber
&=
\gamma_C(t)\mathbb{E}\!\left[
\mathbb{E}\{A_i\mathcal{Y}_i^0(t)\mid X_i,A_i\}
\right] \\
\nonumber
&=
\gamma_C(t)\mathbb{E}\{A_i\mathcal{Y}_i^0(t)\} \\
\nonumber
&=
\gamma_C(t)\mathbb{E}\!\left[
\mathbb{E}\{A_i\mathcal{Y}_i^0(t)\mid A_i\}
\right] \\
\nonumber
&=
p_1\gamma_C(t)\mathbb{E}\{\mathcal{Y}_i^0(t)\mid A_i=1\} \\
\label{eq:weighted_control_risk_identity}
&=
p_1\gamma_C(t)y_1^0(t).
\end{align}

Similarly,
\begin{align}
\nonumber
\mathbb{E}\{(1-A_i)q(X_i)d\mathcal{N}_i(t)\}
&=
\mathbb{E}\{(1-A_i)q(X_i)d\mathcal{N}_i^0(t)\} \\
\nonumber
&=
\mathbb{E}\{(1-A_i)q(X_i)\gamma_C(t)m_\mathcal{N}^0(t,X_i)\} \\
\nonumber
&=
\gamma_C(t)\mathbb{E}\{(1-A_i)q(X_i)m_\mathcal{N}^0(t,X_i)\} \\
\nonumber
&=
\gamma_C(t)\mathbb{E}\{A_i m_\mathcal{N}^0(t,X_i)\} \qquad\qquad\qquad\qquad\qquad \text{by } \eqref{eq:core_equation} \\
\nonumber
&=
\gamma_C(t)\mathbb{E}\{A_id\mathcal{N}_i^0(t)\} \\
\nonumber
&=
p_1\gamma_C(t)\mathbb{E}\{d\mathcal{N}_i^0(t)\mid A_i=1\} \\
\label{eq:weighted_control_event_identity}
&=
p_1\gamma_C(t)y_1^0(t)\lambda_1^0(t)\,dt.
\end{align}

\paragraph{Limits of the weighted risk-set sums and the weighted Cox score.}
We now derive the probability limits of the weighted risk-set sums. Define the
weighted risk-set sums
\[
S_{\omega}^{(r)}(t;\theta)
=
\sum_{i=1}^n
\omega_i\mathcal{Y}_i(t)\exp(\theta A_i)A_i^r,
\qquad r=0,1.
\]
Then
\[
\overline A_{\omega}(t;\theta)
=
\frac{S_{\omega}^{(1)}(t;\theta)}{S_{\omega}^{(0)}(t;\theta)}.
\]

For \(r=0\), since \(\omega_i=A_i+(1-A_i)q(X_i)\) and \(A_i\in\{0,1\}\),
\[
\omega_i\exp(\theta A_i)
=
A_i\exp(\theta)+(1-A_i)q(X_i).
\]
Therefore,
\[
S_{\omega}^{(0)}(t;\theta)
=
\sum_{i=1}^n A_i\mathcal{Y}_i(t)\exp(\theta)
+
\sum_{i=1}^n(1-A_i)q(X_i)\mathcal{Y}_i(t).
\]

By the law of large numbers, together with
\eqref{eq:weighted_control_risk_identity} and the corresponding treated-cohort
identity
\[
\mathbb{E}\{A_i\mathcal{Y}_i(t)\}
=
\mathbb{E}\{A_i\mathcal{Y}_i^1(t)\}
=
\mathbb{E}\!\left[
\mathbb{E}\{A_i\mathcal{Y}_i^1(t)\mid A_i\}
\right]
=
p_1\mathbb{E}\{\mathcal{Y}_i^1(t)\mid A_i=1\}
=
p_1y_1^1(t),
\]
we have
\[
\begin{aligned}
n^{-1}S_{\omega}^{(0)}(t;\theta)
&\overset{p}{\to}
\mathbb{E}\{A_i\mathcal{Y}_i(t)\}\exp(\theta)
+
\mathbb{E}\{(1-A_i)q(X_i)\mathcal{Y}_i(t)\} \\
&=
p_1\exp(\theta)y_1^1(t)
+
p_1\gamma_C(t)y_1^0(t) \\
&=
p_1\{\exp(\theta)y_1^1(t)+\gamma_C(t)y_1^0(t)\}.
\end{aligned}
\]

Similarly,
\[
S_{\omega}^{(1)}(t;\theta)
=
\sum_{i=1}^n A_i\mathcal{Y}_i(t)\exp(\theta),
\]
and hence
\[
n^{-1}S_{\omega}^{(1)}(t;\theta)
\overset{p}{\to}
p_1\exp(\theta)y_1^1(t).
\]

Therefore, by the continuous mapping theorem,
\begin{align}
\label{eq:eq_A_to_a}
\overline A_{\omega}(t;\theta)
\overset{p}{\to}
\overline a(t;\theta),
\end{align}
where
\begin{align}
\label{eq:a_bar_limit}
\overline a(t;\theta)
=
\frac{\exp(\theta)y_1^1(t)}
     {\gamma_C(t)y_1^0(t)+\exp(\theta)y_1^1(t)}.
\end{align}

Next, decompose the weighted Cox score into treated and external-control event
contributions:
\[
\begin{aligned}
U_n^{\omega}(\theta)
&=
\sum_{i=1}^n
\int
A_i\{1-\overline A_{\omega}(t;\theta)\}\,d\mathcal{N}_i(t)  
-
\sum_{i=1}^n
\int
(1-A_i)q(X_i)\overline A_{\omega}(t;\theta)\,d\mathcal{N}_i(t).
\end{aligned}
\]

Using the convergence \eqref{eq:eq_A_to_a}, the treated event component
satisfies
\[
\begin{aligned}
n^{-1}
\sum_{i=1}^n
\int
A_i\{1-\overline A_{\omega}(t;\theta)\}\,d\mathcal{N}_i(t)
&=
\int
\{1-\overline A_{\omega}(t;\theta)\}
\left\{
n^{-1}\sum_{i=1}^n A_i\,d\mathcal{N}_i(t)
\right\}  \\
&\overset{p}{\to}
\int
\{1-\overline a(t;\theta)\}
\mathbb{E}\{A_i\,d\mathcal{N}_i(t)\}.
\end{aligned}
\]
By consistency (A1) and the marginal counting-process identity under treatment
condition \(a=1\), \eqref{eq:exp_d_N_i_a},
\[
\mathbb{E}\{A_i\,d\mathcal{N}_i(t)\}
=
\mathbb{E}\{A_i\,d\mathcal{N}_i^1(t)\}
=
p_1\mathbb{E}\{d\mathcal{N}_i^1(t)\mid A_i=1\}
=
p_1y_1^1(t)\lambda_1^1(t)\,dt.
\]
Therefore,
\[
n^{-1}
\sum_{i=1}^n
\int
A_i\{1-\overline A_{\omega}(t;\theta)\}\,d\mathcal{N}_i(t)
\overset{p}{\to}
p_1
\int
\{1-\overline a(t;\theta)\}
y_1^1(t)\lambda_1^1(t)\,dt.
\]

Similarly, the external-control event component satisfies
\[
\begin{aligned}
n^{-1}
\sum_{i=1}^n
\int
(1-A_i)q(X_i)\overline A_{\omega}(t;\theta)\,d\mathcal{N}_i(t)
&=
\int
\overline A_{\omega}(t;\theta)
\left\{
n^{-1}\sum_{i=1}^n (1-A_i)q(X_i)\,d\mathcal{N}_i(t)
\right\}  \\
&\overset{p}{\to}
\int
\overline a(t;\theta)
\mathbb{E}\{(1-A_i)q(X_i)\,d\mathcal{N}_i(t)\}.
\end{aligned}
\]
Using \eqref{eq:weighted_control_event_identity},
\[
\mathbb{E}\{(1-A_i)q(X_i)\,d\mathcal{N}_i(t)\}
=
p_1\gamma_C(t)y_1^0(t)\lambda_1^0(t)\,dt.
\]
Hence,
\[
n^{-1}
\sum_{i=1}^n
\int
(1-A_i)q(X_i)\overline A_{\omega}(t;\theta)\,d\mathcal{N}_i(t)
\overset{p}{\to}
p_1
\int
\overline a(t;\theta)
\gamma_C(t)y_1^0(t)\lambda_1^0(t)\,dt.
\]

Consequently, the population limit of the normalized weighted Cox score
\(n^{-1}U_n^{\omega}(\theta)\) is
\begin{align}
\label{eq:limit_of_normalized_weighted_score}
U(\theta)
=
p_1
\int
\Big[
\underbrace{\{1-\overline a(t;\theta)\}y_1^1(t)\lambda_1^1(t)}_{\text{treated event contribution}}
-
\underbrace{\overline a(t;\theta)\gamma_C(t)y_1^0(t)\lambda_1^0(t)}_{\text{control event contribution}}
\Big]dt.
\end{align}

\paragraph{Centering of the population weighted Cox score.}
It remains to show that \(U(\theta_{ATT})=0\), meaning that the true ATT
log-hazard ratio \(\theta_{ATT}\) is a root of the population weighted Cox
score equation \(U(\theta)=0\) in
\eqref{eq:limit_of_normalized_weighted_score}.

By the marginal proportional hazards model,
\begin{align}
\label{eq:marginal_hazards_model_ATT_HR}
\lambda_1^1(t)=\lambda_1^0(t)\exp(\theta_{ATT}).    
\end{align}
At \(\theta=\theta_{ATT}\), define
\[
D_{ATT}(t)=\gamma_C(t)y_1^0(t)+\exp(\theta_{ATT})y_1^1(t).
\]
Then, \(\overline a(t;\theta)\) in \eqref{eq:a_bar_limit} and
\(1-\overline a(t;\theta)\), evaluated at \(\theta_{ATT}\), are
\begin{align}
\label{eq:last_score_identity}
\overline a(t;\theta_{ATT})
=
\frac{\exp(\theta_{ATT})y_1^1(t)}{D_{ATT}(t)},
\qquad
1-\overline a(t;\theta_{ATT})
=
\frac{\gamma_C(t)y_1^0(t)}{D_{ATT}(t)}.
\end{align}
Therefore, the treated event contribution in the population score \(U(\theta)\)
in \eqref{eq:limit_of_normalized_weighted_score} is
{
\small
\begin{align*}
\{1-\overline a(t;\theta_{ATT})\}y_1^1(t)\lambda_1^1(t)
&=
\underbrace{\frac{\gamma_C(t)y_1^0(t)}{D_{ATT}(t)}}_{\because \,\,\eqref{eq:last_score_identity}}
\cdot
y_1^1(t)\cdot \underbrace{\lambda_1^0(t)\exp(\theta_{ATT})}_{\because \,\,\eqref{eq:marginal_hazards_model_ATT_HR}}   =
\frac{
\gamma_C(t)y_1^0(t)y_1^1(t)\lambda_1^0(t)\exp(\theta_{ATT})
}{
D_{ATT}(t)
}.
\end{align*}
}
The control event contribution is
{
\small
\begin{align*}
\overline a(t;\theta_{ATT})\gamma_C(t)y_1^0(t)\lambda_1^0(t)
&=
\underbrace{\frac{\exp(\theta_{ATT})y_1^1(t)}{D_{ATT}(t)}}_{\because \,\,\eqref{eq:last_score_identity}}
\cdot
\gamma_C(t)y_1^0(t)\cdot\lambda_1^0(t)  =
\frac{
\gamma_C(t)y_1^0(t)y_1^1(t)\lambda_1^0(t)\exp(\theta_{ATT})
}{
D_{ATT}(t)
}.
\end{align*}
}
The treated and control event contributions are identical for every \(t\).
Hence the integrand of \(U(\theta_{ATT})\) is zero pointwise, and therefore
\[
U(\theta_{ATT})=0.
\]

\paragraph{Conclusion.}
Finally, if \(U(\theta)\) in
\eqref{eq:limit_of_normalized_weighted_score} has a unique root and
\(U_n^{\omega}(\theta)\) converges uniformly to \(U(\theta)\), the standard
consistency theorem for \(Z\)-estimators \citep{van2000asymptotic} implies
that any solution \(\widehat\theta\) of \(U_n^\omega(\widehat\theta)=0\)
satisfies
\[
\widehat\theta\overset{p}{\to}\theta_{ATT}.
\]
This completes the proof.

\subsection{Proof of Theorem \ref{thm:Conditions_for_consistency_MEC_Cox}}\label{subsec:Proof of Theorem 2}
\paragraph{Basic setup and cross-fitting notation.} 
Recall that the pooled index set is partitioned into validation folds
\[
\{1,\ldots,n\}
=
\mathcal J^{(1)}\cup\cdots\cup\mathcal J^{(K)},
\]
and let
\[
\mathcal J^{(-k)}
=
\{1,\ldots,n\}\setminus \mathcal J^{(k)}
\]
denote the corresponding training set. For each \(i\in\mathcal J^{(k)}\), the
propensity-score estimator and the calibration-basis map are constructed using
the training set \(\mathcal J^{(-k)}\) and evaluated on the validation fold
\(\mathcal J^{(k)}\). Thus,
\[
\widehat\pi_i
=
\widehat\pi^{(-)}(X_i)
=
\widehat\pi^{(-k)}(X_i),
\qquad
\widehat h_i
=
\widehat h^{(-)}(X_i)
=
\widehat h^{(-k)}(X_i).
\]
The specific survival-probability basis used in MEC-Cox is one possible choice
of the cross-fitted calibration-basis map \(\widehat h^{(-)}(\cdot)\). Figure~\ref{fig:kfold_crossfitting} schematically illustrates this
cross-fitting construction.

\begin{figure}[h!]
\centering
\resizebox{0.9\textwidth}{!}{%
\begin{tikzpicture}[
    >=Latex,
    font=\small,
    fold/.style={draw, rectangle, minimum width=2.1cm, minimum height=0.8cm, align=center},
    val/.style={fold, fill=blue!15},
    train/.style={fold, fill=gray!15},
    proc/.style={draw, rounded corners, rectangle, minimum width=2.2cm, minimum height=0.95cm, align=center}
]

\node[blue!70!black, anchor=west] at (-1.8,1.5) {Validation folds};

\node at (4.6,1.7) {$\{1,\ldots,n\}=\mathcal J^{(1)}\cup\cdots\cup\mathcal J^{(K)}$};

\foreach \x/\lab in {0/{\(\mathcal J^{(1)}\)}, 2.3/{\(\mathcal J^{(2)}\)}, 4.6/{\(\mathcal J^{(3)}\)}, 6.9/{\(\cdots\)}, 9.2/{\(\mathcal J^{(K)}\)}}{
    \node[fold] at (\x,0.8) {\lab};
}

\node[anchor=east] at (-1.0,-0.8) {$k=1$};

\node[val]   (r11) at (0,-0.8)   {validation\\[-1mm] \(\mathcal J^{(1)}\)};
\node[train] (r12) at (2.3,-0.8) {training\\[-1mm] \(\mathcal J^{(-1)}\)};
\node[train] (r13) at (4.6,-0.8) {training\\[-1mm] \(\mathcal J^{(-1)}\)};
\node[train] (r14) at (6.9,-0.8) {\(\cdots\)};
\node[train] (r15) at (9.2,-0.8) {training\\[-1mm] \(\mathcal J^{(-1)}\)};

\draw[decorate,decoration={brace,mirror,amplitude=5pt}]
    ($(r12.south west)+(0,-0.05)$) -- ($(r15.south east)+(0,-0.05)$)
    node[midway, below=8pt] {training: \(\mathcal J^{(-1)}=\{1,\ldots,n\}\setminus\mathcal J^{(1)}\)};

\node[proc] (fit1) at (12.4,-0.8) {Fit\\ nuisance\\ learners};
\node[proc, draw=blue!70!black, text=blue!70!black] (eval1) at (15.5,-0.8)
{Evaluate\\ out-of-fold\\ predictions};

\draw[->] (10.35,-0.8) -- (fit1.west);
\draw[->] (fit1.east) -- (eval1.west);

\node[anchor=east] at (-1.0,-3.0) {$k=2$};

\node[train] (r21) at (0,-3.0)   {training\\[-1mm] \(\mathcal J^{(-2)}\)};
\node[val]   (r22) at (2.3,-3.0) {validation\\[-1mm] \(\mathcal J^{(2)}\)};
\node[train] (r23) at (4.6,-3.0) {training\\[-1mm] \(\mathcal J^{(-2)}\)};
\node[train] (r24) at (6.9,-3.0) {\(\cdots\)};
\node[train] (r25) at (9.2,-3.0) {training\\[-1mm] \(\mathcal J^{(-2)}\)};

\draw[decorate,decoration={brace,mirror,amplitude=5pt}]
    ($(r21.south west)+(0,-0.05)$) -- ($(r25.south east)+(0,-0.05)$)
    node[midway, below=8pt] {training: \(\mathcal J^{(-2)}=\{1,\ldots,n\}\setminus\mathcal J^{(2)}\)};

\node[proc] (fit2) at (12.4,-3.0) {Fit\\ nuisance\\ learners};
\node[proc, draw=blue!70!black, text=blue!70!black] (eval2) at (15.5,-3.0)
{Evaluate\\ out-of-fold\\ predictions};

\draw[->] (10.35,-3.0) -- (fit2.west);
\draw[->] (fit2.east) -- (eval2.west);

\node at (-1.0,-5.0) {\(\vdots\)};
\node at (4.6,-5.0) {\(\vdots\)};
\node at (12.4,-5.0) {\(\vdots\)};
\node at (15.5,-5.0) {\(\vdots\)};

\node[anchor=east] at (-1.0,-6.7) {$k=K$};

\node[train] (rK1) at (0,-6.7)   {training\\[-1mm] \(\mathcal J^{(-K)}\)};
\node[train] (rK2) at (2.3,-6.7) {training\\[-1mm] \(\mathcal J^{(-K)}\)};
\node[train] (rK3) at (4.6,-6.7) {training\\[-1mm] \(\mathcal J^{(-K)}\)};
\node[train] (rK4) at (6.9,-6.7) {\(\cdots\)};
\node[val]   (rK5) at (9.2,-6.7) {validation\\[-1mm] \(\mathcal J^{(K)}\)};

\draw[decorate,decoration={brace,mirror,amplitude=5pt}]
    ($(rK1.south west)+(0,-0.05)$) -- ($(rK4.south east)+(0,-0.05)$)
    node[midway, below=8pt] {training: \(\mathcal J^{(-K)}=\{1,\ldots,n\}\setminus\mathcal J^{(K)}\)};

\node[proc] (fitK) at (12.4,-6.7) {Fit\\ nuisance\\ learners};
\node[proc, draw=blue!70!black, text=blue!70!black] (evalK) at (15.5,-6.7)
{Evaluate\\ out-of-fold\\ predictions};

\draw[->] (10.35,-6.7) -- (fitK.west);
\draw[->] (fitK.east) -- (evalK.west);

\end{tikzpicture}%
}
\caption{Schematic illustration of \(K\)-fold cross-fitting.}
\label{fig:kfold_crossfitting}
\end{figure}
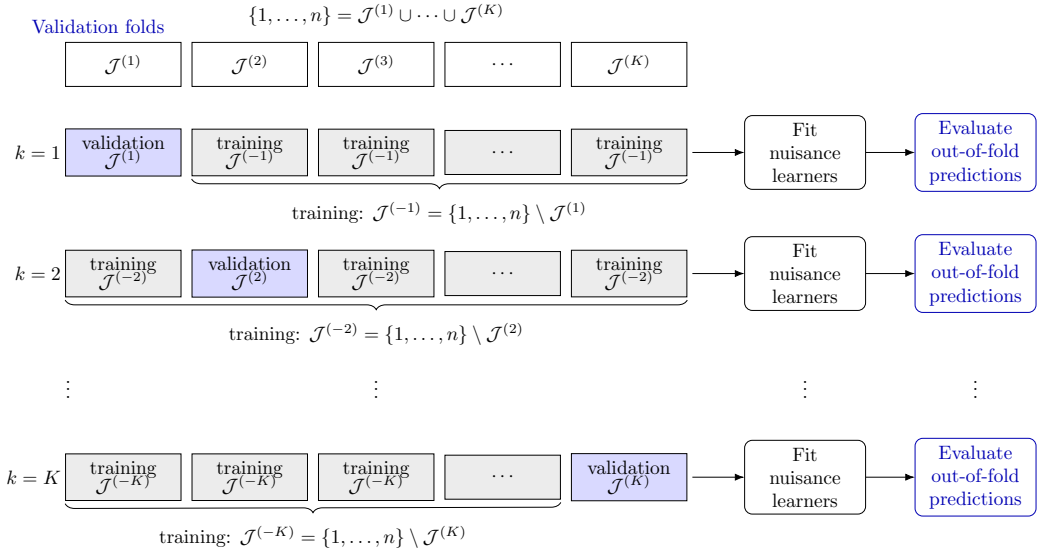

Let \(p_1=\Pr(A=1)\) and \(p_0=\Pr(A=0)\). Write the true and estimated
ATT odds weights as
\[
q(X)=\frac{\pi(X)}{1-\pi(X)},
\qquad
\widehat q_i=
\frac{\widehat\pi^{(-)}(X_i)}
{1-\widehat\pi^{(-)}(X_i)} .
\]
The normalized baseline external-control weight is defined by
\begin{equation}
\label{eq:normalized_weight_baseline}
\widehat d_i
=
\frac{n_1\widehat q_i}{\sum_{j\in\mathcal I_0}\widehat q_j}
=
\widehat c_n\widehat q_i,
\qquad
\widehat c_n
=
\frac{n_1}{\sum_{j\in\mathcal I_0}\widehat q_j},
\qquad i\in\mathcal I_0 .
\end{equation}

Throughout this proof, all \(o_p(1)\) and \(O_p(1)\) statements involving
vector-valued quantities are understood componentwise. Since the calibration
basis dimension \(p\) is fixed, componentwise convergence to zero is equivalent
to convergence to zero in any finite-dimensional norm.

\paragraph{Asymptotic feasibility of the baseline weights.} We next show that the normalized ATT-IPW weights
\(\{\widehat d_i:i\in\mathcal I_0\}\) are asymptotically feasible for the
MEC calibration constraint. By C1, namely the \(L_2\)-consistency of the
cross-fitted source propensity-score estimator under the external-control
covariate distribution, together with the stability of the odds map
\(\pi\mapsto \pi/(1-\pi)\) on the relevant support,
\[
\frac{1}{n}\sum_{j\in\mathcal I_0}\widehat q_j
=
\frac{1}{n}\sum_{j=1}^n(1-A_j)\widehat q_j
\overset{p}{\longrightarrow}
\mathbb E\{(1-A)q(X)\}.
\]
Moreover,
\[
\mathbb E\{(1-A)q(X)\}
=
\mathbb E\left[
\mathbb E\{(1-A)q(X)\mid X\}
\right]
=
\mathbb E\left[
\{1-\pi(X)\}\frac{\pi(X)}{1-\pi(X)}
\right]
=
\mathbb E\{\pi(X)\}
=
p_1 .
\]
Since \(n_1/n\overset{p}{\to}p_1\), it follows that
\[
\widehat c_n
=
\frac{n_1/n}{n^{-1}\sum_{j\in\mathcal I_0}\widehat q_j}
\overset{p}{\longrightarrow}
1.
\]
Therefore, the normalized baseline weights
\(\widehat d_i=\widehat c_n\widehat q_i\) in
\eqref{eq:normalized_weight_baseline} are asymptotically equivalent to the
usual estimated ATT odds weights \(\widehat q_i\).

The calibration constraint for MEC-Cox is
\begin{align}
\label{eq:calibration_constraint}
\sum_{i\in\mathcal I_0}w_i\widehat h_i
=
\sum_{i\in\mathcal I_1}\widehat h_i.    
\end{align}
Thus, it suffices to show that the baseline weights $\widehat d_i$ already satisfy this constraint \eqref{eq:calibration_constraint} asymptotically.

Define the raw cross-fitted ATT balance
\[
B_n
=
\frac{1}{n}\sum_{i=1}^n
\left\{
(1-A_i)\widehat q_i\widehat h_i
-
A_i\widehat h_i
\right\}.
\]
Decompose $B_n$ into two parts:
\[
B_n
=
\underbrace{
\frac{1}{n}\sum_{i=1}^n
(1-A_i)\{\widehat q_i-q(X_i)\}\widehat h_i
}_{B_{1n}}
+
\underbrace{
\frac{1}{n}\sum_{i=1}^n
\left\{
(1-A_i)q(X_i)\widehat h_i
-
A_i\widehat h_i
\right\}
}_{B_{2n}}.
\]
For $B_{1n}$, Cauchy--Schwarz inequality applied to the Euclidean norm gives
\[
\|B_{1n}\|
\le
\left[
\frac{1}{n}\sum_{i=1}^n
(1-A_i)\{\widehat q_i-q(X_i)\}^2
\right]^{1/2}
\left[
\frac{1}{n}\sum_{i=1}^n
(1-A_i)\|\widehat h_i\|^2
\right]^{1/2}.
\]
The first factor is $o_p(1)$ by C1 and stability of the odds map, and the second factor is $O_p(1)$ by C2. Hence
\begin{align}
\label{eq:B_1}
B_{1n}=o_p(1).    
\end{align}

For $B_{2n}$, we use the cross-fitting structure fold by fold. 
Define, for \(i\in\mathcal J^{(k)}\),
\[
\psi_i^{(k)}
=
\left\{(1-A_i)q(X_i)-A_i\right\}
\widehat h^{(-k)}(X_i).
\]
Then
\[
B_{2n}
=
\sum_{k=1}^K
\frac{1}{n}
\sum_{i\in\mathcal J^{(k)}}\psi_i^{(k)}.
\]

Let \(\mathcal T_{-k}\) denote the sigma-field generated by the training sample
\(\mathcal J^{(-k)}\) and the fitted nuisance learners used to construct
\(\widehat h^{(-k)}\). Conditional on \(\mathcal T_{-k}\), the map
\(\widehat h^{(-k)}(\cdot)\) is fixed, while the validation observations
\(\{O_i:i\in\mathcal J^{(k)}\}\) are independent of this fitted map. Hence,
applying the ATT density-ratio identity to the fixed vector-valued function
\(f(X)=\widehat h^{(-k)}(X)\) gives
\begin{align}
\label{eq:zero_condi}
\mathbb E\{\psi_i^{(k)}\mid \mathcal T_{-k}\}=0.    
\end{align}
Indeed, for any square-integrable vector-valued function $f(X)$,
\begin{equation}
\label{eq:proof_att_density_ratio_identity}
\mathbb E\{(1-A)q(X)f(X)\}
=
\mathbb E\{Af(X)\},
\end{equation}
because
\[
\mathbb E\{(1-A)q(X)f(X)\}
=
\mathbb E\left[
\{1-\pi(X)\}\frac{\pi(X)}{1-\pi(X)}f(X)
\right]
=
\mathbb E\{\pi(X)f(X)\}
=
\mathbb E\{Af(X)\}.
\]
Next, C2 and the finite weighted-moment regularity of the ATT odds weights imply
\begin{align*}
&\mathbb E\{\|\psi_i^{(k)}\|^2\mid \mathcal T_{-k}\}\le
2\mathbb E\left[
(1-A)q^2(X)\|\widehat h^{(-k)}(X)\|^2
+
A\|\widehat h^{(-k)}(X)\|^2
\,\middle|\,
\mathcal T_{-k}
\right]\\
&\quad=
2p_0\,
\underbrace{
\mathbb E\left[
q^2(X)\|\widehat h^{(-k)}(X)\|^2
\mid A=0,\mathcal T_{-k}
\right]
}_{\substack{O_p(1)\ \text{by C2 and finite}\\\text{weighted-moment regularity of }q}}
+
2p_1\,
\underbrace{
\mathbb E\left[
\|\widehat h^{(-k)}(X)\|^2
\mid A=1,\mathcal T_{-k}
\right]
}_{O_p(1)\ \text{by C2}}
\\
&\quad=
O_p(1).
\end{align*}

For $i\neq j$, conditional independence of the validation observations given
$\mathcal T_{-k}$ and the conditional mean-zero property in \eqref{eq:zero_condi} imply
\[
\begin{aligned}
\mathbb E\left[
\{\psi_i^{(k)}\}^{\top}\psi_j^{(k)}
\mid \mathcal T_{-k}
\right]
&=
\mathbb E\left[
\mathbb E\left\{
\{\psi_i^{(k)}\}^{\top}\psi_j^{(k)}
\mid \mathcal T_{-k},\psi_i^{(k)}
\right\}
\mid \mathcal T_{-k}
\right] \\
&=
\mathbb E\left[
\{\psi_i^{(k)}\}^{\top}
\mathbb E\{\psi_j^{(k)}\mid \mathcal T_{-k},\psi_i^{(k)}\}
\mid \mathcal T_{-k}
\right] \\
&=
\mathbb E\left[
\{\psi_i^{(k)}\}^{\top}
\mathbb E\{\psi_j^{(k)}\mid \mathcal T_{-k}\}
\mid \mathcal T_{-k}
\right] \\
&=
\mathbb E\left[
\{\psi_i^{(k)}\}^{\top}0
\mid \mathcal T_{-k}
\right] \\
&=0.
\end{aligned}
\]
Here, the third equality follows from the conditional independence of
\[
\psi_i^{(k)}
=
\psi^{(k)}\!\left(O_i;\widehat h^{(-k)}\right)
\qquad\text{and}\qquad
\psi_j^{(k)}
=
\psi^{(k)}\!\left(O_j;\widehat h^{(-k)}\right)
\]
given $\mathcal T_{-k}$. Indeed, conditional on $\mathcal T_{-k}$, the fitted map
$\widehat h^{(-k)}$ is fixed, and the validation observations $O_i$ and $O_j$ are independent for
$i\neq j$. Therefore, $\mathbb E\{\psi_j^{(k)}\mid \mathcal T_{-k},\psi_i^{(k)}\}
=
\mathbb E\{\psi_j^{(k)}\mid \mathcal T_{-k}\}.$

Hence,
\begin{align}
\mathbb E\left[
\left\|
\frac{1}{n}\sum_{i\in\mathcal J^{(k)}}\psi_i^{(k)}
\right\|^2
\,\middle|\,
\mathcal T_{-k}
\right] \nonumber
& =
\frac{1}{n^2}
\mathbb E\left[
\sum_{i\in\mathcal J^{(k)}}\|\psi_i^{(k)}\|^2
+
\sum_{\substack{i,j\in\mathcal J^{(k)}\\ i\neq j}}
\{\psi_i^{(k)}\}^{\top}\psi_j^{(k)}
\,\middle|\,
\mathcal T_{-k}
\right] \nonumber\\
&\quad =
\frac{1}{n^2}
\sum_{i\in\mathcal J^{(k)}}
\mathbb E\{\|\psi_i^{(k)}\|^2\mid \mathcal T_{-k}\} \nonumber\\
&\quad =
O_p\left(\frac{|\mathcal J^{(k)}|}{n^2}\right)
=
O_p(n^{-1}),
\label{eq:fold_average_second_moment}
\end{align}
where the second equality follows because the cross terms vanish for $i\neq j$, and the last equality uses that $K$ is fixed and $|\mathcal J^{(k)}|=O(n)$.

Let
\[
Z_{n,k}
=
\frac{1}{n}\sum_{i\in\mathcal J^{(k)}}\psi_i^{(k)} .
\]

Therefore, for any fixed $\varepsilon>0$, conditional Markov's inequality gives
\[
\Pr\left(
\|Z_{n,k}\|>\varepsilon
\,\middle|\,
\mathcal T_{-k}
\right)
\le
\frac{1}{\varepsilon^2}
\mathbb E\left[
\|Z_{n,k}\|^2
\,\middle|\,
\mathcal T_{-k}
\right]
=
O_p(n^{-1})
=
o_p(1),
\]
where the second last equality is from $\mathbb E\left[
\|Z_{n,k}\|^2
\,\middle|\,
\mathcal T_{-k}
\right]
=
O_p(n^{-1})$ \eqref{eq:fold_average_second_moment}.

Taking expectations over $\mathcal T_{-k}$ yields
\[
\Pr\left(
\|Z_{n,k}\|>\varepsilon
\right)
\to 0.
\]
Hence,
\[
\frac{1}{n}\sum_{i\in\mathcal J^{(k)}}\psi_i^{(k)}
=
o_p(1).
\]

Since $K$ is fixed, summing over folds gives
\begin{align}
    \label{eq:B_2}
    B_{2n}
=
\sum_{k=1}^K
\frac{1}{n}\sum_{i\in\mathcal J^{(k)}}\psi_i^{(k)}
=
o_p(1).
\end{align}

Combining the bounds for $B_{1n}$ \eqref{eq:B_1} and $B_{2n}$ \eqref{eq:B_2} gives the raw cross-fitted ATT balance
\begin{equation}
\label{eq:proof_crossfitted_balance_raw}
B_n
=
\frac{1}{n}\sum_{i=1}^n
\left\{
(1-A_i)\widehat q_i\widehat h_i
-
A_i\widehat h_i
\right\}
=
o_p(1).
\end{equation}
This raw balance \eqref{eq:proof_crossfitted_balance_raw} is not yet the MEC calibration feasibility statement, because the baseline weights
$\widehat d_i$ are normalized versions of $\widehat q_i$. We therefore translate
\eqref{eq:proof_crossfitted_balance_raw} into a statement for the normalized ATT-IPW weights.

Define
\[
Q_n=\frac{1}{n}\sum_{i=1}^n(1-A_i)\widehat q_i,
\,\,
P_n=\frac{n_1}{n},
\,\,
H_{0n}
=
\frac{1}{n}\sum_{i=1}^n(1-A_i)\widehat q_i\widehat h_i,
\,\,
H_{1n}
=
\frac{1}{n}\sum_{i=1}^n A_i\widehat h_i.
\]

Then \eqref{eq:proof_crossfitted_balance_raw} implies
\[
H_{0n}-H_{1n}=o_p(1).
\]
Also, by C1 and the ATT density-ratio identity, $Q_n\overset{p}{\to}p_1$, while
$P_n=n_1/n\overset{p}{\to}p_1$, thus $P_n-Q_n=o_p(1)$. By C2, $H_{1n}=O_p(1)$.

Using the normalization $\widehat d_i=n_1\widehat q_i/\sum_{j\in\mathcal I_0}\widehat q_j$, we have
\[
\frac{1}{n_1}
\sum_{i\in\mathcal I_0}
\widehat d_i\widehat h_i
=
\frac{1}{n_1}
\sum_{i\in\mathcal I_0}
\frac{n_1\widehat q_i}{\sum_{j\in\mathcal I_0}\widehat q_j}
\widehat h_i
=
\frac{
\sum_{i\in\mathcal I_0}\widehat q_i\widehat h_i
}{
\sum_{j\in\mathcal I_0}\widehat q_j
}
=
\frac{
n^{-1}\sum_{i=1}^n(1-A_i)\widehat q_i\widehat h_i
}{
n^{-1}\sum_{i=1}^n(1-A_i)\widehat q_i
}
=
\frac{H_{0n}}{Q_n}.
\]
and
\[
\frac{1}{n_1}
\sum_{i\in\mathcal I_1}
\widehat h_i
=
\frac{
n^{-1}\sum_{i=1}^n A_i\widehat h_i
}{
n_1/n
}
=
\frac{H_{1n}}{P_n}.
\]
Therefore,
\[
\begin{aligned}
\frac{1}{n_1}
\left\{
\sum_{i\in\mathcal I_0}\widehat d_i\widehat h_i
-
\sum_{i\in\mathcal I_1}\widehat h_i
\right\}
&=
\frac{H_{0n}}{Q_n}
-
\frac{H_{1n}}{P_n}  =
\frac{H_{0n}-H_{1n}}{Q_n}
+
H_{1n}
\left(
\frac{1}{Q_n}-\frac{1}{P_n}
\right) \\
&=
\frac{o_p(1)}{Q_n}
+
O_p(1)\frac{P_n-Q_n}{P_nQ_n} \\
&=
o_p(1)
+
O_p(1)o_p(1) \\
&=
o_p(1),
\end{aligned}
\]
where the fourth equality follows because $Q_n\overset{p}{\to}p_1$, $P_n\overset{p}{\to}p_1$, $p_1>0$, and hence $P_nQ_n$ is bounded away from zero with probability tending to one, while $P_n-Q_n=o_p(1)$.

Thus, the normalized ATT-IPW weights are asymptotically feasible for the MEC calibration constraint:
\begin{equation}
\label{eq:proof_baseline_att_feasible}
\frac{1}{n_1}
\left\{
\sum_{i\in\mathcal I_0}\widehat d_i\widehat h_i
-
\sum_{i\in\mathcal I_1}\widehat h_i
\right\}
=
o_p(1).
\end{equation}
This implies that the MEC calibration constraint is asymptotically compatible with the baseline
ATT weighting scheme. In particular, the exact finite-sample balance imposed by MEC does not
define a new population target; rather, it enforces a cross-fitted calibration-basis balance condition
that the normalized ATT-IPW weights already satisfy asymptotically. Therefore, the calibration
step can be viewed as a finite-sample refinement of the ATT weighting scheme.

\paragraph{Stability of the Bregman calibration perturbation.}
We now show that the MEC calibration perturbation is asymptotically negligible.
Define the normalized dual calibration map
\[
F_n(\lambda)
=
\frac{1}{n_1}
\left\{
\sum_{i\in I_0} w_i(\lambda)\widehat h_i
-
\sum_{i\in I_1} \widehat h_i
\right\},
\qquad
w_i(\lambda)
=
g^{-1}\{g(\widehat d_i)+\lambda^\top \widehat h_i\}.
\]
Since \(w_i(0)=\widehat d_i\), equation
\eqref{eq:proof_baseline_att_feasible} implies
\[
F_n(0)=o_p(1).
\]
By C3-(i), the dual calibration equation \(F(\lambda)=0\) admits a local
solution \(\widehat\lambda\). Hence this MEC dual solution satisfies
\(F_n(\widehat\lambda)=0\). A Taylor expansion around \(\lambda=0\) gives
\[
0
=
F_n(\widehat\lambda)
=
F_n(0)+J_n(\widetilde\lambda)\widehat\lambda,
\]
where \(\widetilde\lambda\) lies between \(0\) and \(\widehat\lambda\), and
\[
J_n(\lambda)
=
\frac{\partial F_n(\lambda)}{\partial\lambda^\top}.
\]
By C3-(ii), \(J_n(\widetilde\lambda)\) is nonsingular with probability tending
to one and \(J_n(\widetilde\lambda)^{-1}=O_p(1)\). Therefore,
\begin{equation}
\label{eq:lambda_negligible}
\widehat\lambda
=
-
J_n(\widetilde\lambda)^{-1}F_n(0)
=
o_p(1).
\end{equation}
Next, by C3-(iii), the inverse gradient map \(g^{-1}\) is locally smooth on the
relevant neighborhood. Let
\[
a_i=g(\widehat d_i),
\qquad
b_i=\widehat\lambda^\top\widehat h_i.
\]
Then
\[
w_i(\widehat\lambda)
=
g^{-1}(a_i+b_i),
\qquad
\widehat d_i
=
g^{-1}(a_i).
\]
By the mean-value theorem,
\[
w_i(\widehat\lambda)-\widehat d_i
=
g^{-1}(a_i+b_i)-g^{-1}(a_i)
=
\{g^{-1}\}'(\xi_i)b_i
\]
for some \(\xi_i\) between \(a_i\) and \(a_i+b_i\). Since
\(\{g^{-1}\}'(\cdot)\) is locally bounded on the relevant neighborhood, there
exists a finite constant \(C<\infty\), independent of \(i\) and \(n\), such that
\[
\left|w_i(\widehat\lambda)-\widehat d_i\right|
\le
C\left|\widehat\lambda^\top\widehat h_i\right|.
\]
Equivalently,
\[
\left|w_i(\widehat\lambda)-\widehat d_i\right|
\lesssim
\left|\widehat\lambda^\top\widehat h_i\right|,
\]
where \(\lesssim\) denotes an inequality up to a finite constant independent of
\(i\) and \(n\). Therefore,
\begin{equation}
\label{eq:core_equation_2}
\frac{1}{n}\sum_{i\in\mathcal I_0}
\{w_i(\widehat\lambda)-\widehat d_i\}^2
\lesssim
\|\widehat\lambda\|^2
\frac{1}{n}\sum_{i\in\mathcal I_0}\|\widehat h_i\|^2
=
o_p(1),
\end{equation}
where the last equality follows from
\(\widehat\lambda=o_p(1)\) in \eqref{eq:lambda_negligible} and C2. Thus, the
MEC calibration perturbation is asymptotically negligible in empirical
\(L_2\) norm.

Consequently, the calibrated weights
\(\{w_i(\widehat\lambda):i\in\mathcal I_0\}\) are asymptotically close to the
normalized ATT-IPW weights \(\{\widehat d_i:i\in\mathcal I_0\}\) in empirical
\(L_2\) norm. Their difference is the Bregman calibration perturbation, which
vanishes asymptotically because \(\widehat\lambda=o_p(1)\) and the cross-fitted
calibration basis is \(L_2\)-bounded.

\paragraph{Connection to the ATT-weighted Cox score.}
It remains to connect the MEC-weighted Cox score to the ATT-weighted Cox
population score. Define the MEC estimating-equation weight and the normalized ATT-IPW estimating-equation weight by
\[
\widetilde\omega_i(\widehat\lambda)
=
A_i+(1-A_i)w_i(\widehat\lambda),
\qquad
\widehat\omega_i^{d}
=
A_i+(1-A_i)\widehat d_i .
\]
Let
\[
U_n^{\mathrm{MEC}}(\theta)
=
\sum_{i=1}^n
\int
\widetilde\omega_i(\widehat\lambda)
\left\{
A_i-\bar A_n^{\mathrm{MEC}}(t;\theta)
\right\}
\,d\mathcal{N}_i(t),
\]
where
\[
\bar A_n^{\mathrm{MEC}}(t;\theta)
=
\frac{
\sum_{j=1}^n
\widetilde\omega_j(\widehat\lambda)\mathcal{Y}_j(t)\exp(\theta A_j)A_j
}{
\sum_{j=1}^n
\widetilde\omega_j(\widehat\lambda)\mathcal{Y}_j(t)\exp(\theta A_j)
}.
\]
Similarly, define the normalized ATT-IPW Cox score
\[
U_n^{\mathrm{IPW}}(\theta)
=
\sum_{i=1}^n
\int
\widehat\omega_i^{d}
\left\{
A_i-\bar A_n^{d}(t;\theta)
\right\}
\,d\mathcal{N}_i(t),
\]
where
\[
\bar A_n^{d}(t;\theta)
=
\frac{
\sum_{j=1}^n
\widehat\omega_j^{d}\mathcal{Y}_j(t)\exp(\theta A_j)A_j
}{
\sum_{j=1}^n
\widehat\omega_j^{d}\mathcal{Y}_j(t)\exp(\theta A_j)
}.
\]
The two scores differ only through the external-control weights:
\(U_n^{\mathrm{IPW}}(\theta)\) uses the normalized ATT-IPW weights
\(\widehat d_i\), whereas \(U_n^{\mathrm{MEC}}(\theta)\) uses the calibrated
weights \(w_i(\widehat\lambda)\).

By Cauchy--Schwarz, we have $(1/n)\sum_{i\in\mathcal I_0} a_i
=
1/n\sum_{i\in\mathcal I_0}1\cdot a_i
\le $ $(n_0/n)^{1/2} $ $ 
((1/n) $ $ \sum_{i\in\mathcal I_0}a_i^2
)^{1/2}.$ Hence, by \eqref{eq:core_equation_2}, with
\(a_i=|w_i(\widehat\lambda)-\widehat d_i|\), 
\[
\frac{1}{n}\sum_{i\in\mathcal I_0}
\left|w_i(\widehat\lambda)-\widehat d_i\right|
\le
\left(\frac{n_0}{n}\right)^{1/2}
\left[
\frac{1}{n}\sum_{i\in\mathcal I_0}
\{w_i(\widehat\lambda)-\widehat d_i\}^2
\right]^{1/2}
=
o_p(1),
\]
because \(n_0/n=O_p(1)\).

Therefore, under the usual boundedness and nondegeneracy conditions for the Cox
risk-set denominators, replacing \(\widehat d_i\) by
\(w_i(\widehat\lambda)\) changes the normalized weighted Cox score by only an
asymptotically negligible amount. Hence,
\[
\sup_{\theta\in\Theta_0}
\left|
n^{-1}U_n^{\mathrm{MEC}}(\theta)
-
n^{-1}U_n^{\mathrm{IPW}}(\theta)
\right|
=
o_p(1),
\]
where \(\Theta_0\) is a compact neighborhood of \(\theta_{ATT}\).

Next, define the ATT-weighted Cox population score by
\[
U(\theta)
=
\mathbb E
\left[
\int
\omega
\left\{
A-\bar A(t;\theta)
\right\}
\,d\mathcal{N}(t)
\right],
\qquad
\omega=A+(1-A)q(X),
\]
where
\[
\bar A(t;\theta)
=
\frac{
\mathbb E\{\omega \mathcal{Y}(t)\exp(\theta A)A\}
}{
\mathbb E\{\omega \mathcal{Y}(t)\exp(\theta A)\}
}.
\]
By C1, the asymptotic normalization of the baseline weights, and the ATT
density-ratio identity in \eqref{eq:proof_att_density_ratio_identity}, the
normalized ATT-IPW Cox score converges uniformly to the ATT-weighted Cox
population score:
\[
\sup_{\theta\in\Theta_0}
\left|
n^{-1}U_n^{\mathrm{IPW}}(\theta)-U(\theta)
\right|
=
o_p(1).
\]
Combining the two preceding displays by the triangle inequality gives
\begin{equation}
\label{eq:proof_mec_score_limit}
\sup_{\theta\in\Theta_0}
\left|
n^{-1}U_n^{\mathrm{MEC}}(\theta)-U(\theta)
\right|
=
o_p(1).
\end{equation}

\paragraph{Conclusion.}
By Theorem \ref{thm:identification_weighted_cox}, under causal assumptions A1--A3,
independent censoring, and the marginal proportional hazards model, the
population score satisfies \(U(\theta_{ATT})=0\). Under the assumed Cox
regularity conditions, \(U(\theta)\) has a unique root at \(\theta_{ATT}\) in
\(\Theta_0\). Since \(\widehat\theta_{\mathrm{MEC}}\) solves
\(U_n^{\mathrm{MEC}}(\theta)=0\), the standard consistency theorem for
\(Z\)-estimators \citep{van2000asymptotic} applied to \eqref{eq:proof_mec_score_limit} yields
\[
\widehat\theta_{\mathrm{MEC}}
\overset{p}{\longrightarrow}
\theta_{ATT}.
\]
This completes the proof.

\subsection{Proof of Theorem \ref{thm:mec_projection_gain}}
\label{subsec:proof_mec_projection_gain}

Recall that the stacked estimating system is based on the subject-level contribution
\[
\widehat\Psi_i
=
\begin{pmatrix}
\widehat\eta_i\\
\widehat\rho_i
\end{pmatrix},
\]
where \(\widehat\eta_i\) is the Lin--Wei/Binder empirical Cox contribution evaluated at
\((\widehat\theta_{\rm MEC},\widehat\lambda)\), and \(\widehat\rho_i\) is the subject-level contribution
to the dual calibration equation. The empirical derivative matrix has the block upper-triangular
form
\[
\widehat D
=
\begin{pmatrix}
\widehat D_{\theta\theta} & \widehat D_{\theta\lambda}\\
0 & \widehat D_{\lambda\lambda}
\end{pmatrix}.
\]
Therefore,
\[
\widehat D^{-1}
=
\begin{pmatrix}
\widehat D_{\theta\theta}^{-1}
&
-\widehat D_{\theta\theta}^{-1}
\widehat D_{\theta\lambda}
\widehat D_{\lambda\lambda}^{-1}\\
0
&
\widehat D_{\lambda\lambda}^{-1}
\end{pmatrix}.
\]

The estimated subject-level linearized contribution to
\(\widehat\xi-\xi_0\), where
\(\widehat\xi=(\widehat\theta_{\rm MEC},\widehat\lambda)\), is
\[
\widehat\phi_i
=
-\widehat D^{-1}\widehat\Psi_i.
\]
Taking the first component gives
\[
\widehat\phi_i^{\rm MEC}
=
[-\widehat D^{-1}\widehat\Psi_i]_1
=
-\widehat D_{\theta\theta}^{-1}
\left\{
\widehat\eta_i
-
\widehat C_{\rm lin}\widehat\rho_i
\right\},
\qquad
\widehat C_{\rm lin}
=
\widehat D_{\theta\lambda}
\widehat D_{\lambda\lambda}^{-1}.
\]
Hence the MEC-Cox sandwich variance for \(\widehat\theta_{\rm MEC}\) is
\[
\widehat{\operatorname{Var}}_{\rm MEC}
(\widehat\theta_{\rm MEC})
=
\sum_{i=1}^n
\left(
\widehat\phi_i^{\rm MEC}
\right)^2
=
\widehat D_{\theta\theta}^{-1}
\left\{
\sum_{i=1}^n
\left(
\widehat\eta_i
-
\widehat C_{\rm lin}\widehat\rho_i
\right)^2
\right\}
(\widehat D_{\theta\theta}^{-1})^\top.
\]

If the calibration equation is ignored, the calibrated weights are treated as fixed. In that case,
the Lin--Wei/Binder variance uses the full Cox contribution \(\widehat\eta_i\), giving
\[
\widehat{\operatorname{Var}}_{\rm LW}
(\widehat\theta_{\rm MEC})
=
\widehat D_{\theta\theta}^{-1}
\left\{
\sum_{i=1}^n
\widehat\eta_i^2
\right\}
(\widehat D_{\theta\theta}^{-1})^\top.
\]

Now define the empirical projection coefficient
\[
\widehat C_{\rm proj}
=
\arg\min_{C\in\mathbb R^{1\times p}}
\sum_{i=1}^n
\left(
\widehat\eta_i
-
C\widehat\rho_i
\right)^2 .
\]
The normal equations for this least-squares problem are
\[
\sum_{i=1}^n
\left(
\widehat\eta_i
-
\widehat C_{\rm proj}\widehat\rho_i
\right)
\widehat\rho_i^\top
=
0.
\]
Equivalently,
\[
\sum_{i=1}^n
\widehat\eta_i\widehat\rho_i^\top
-
\widehat C_{\rm proj}
\sum_{i=1}^n
\widehat\rho_i\widehat\rho_i^\top
=
0.
\]
Since \(\widehat B_{\lambda\lambda}
=
\sum_{i=1}^n
\widehat\rho_i\widehat\rho_i^\top\) is nonsingular, this yields
\[
\widehat C_{\rm proj}
=
\left[
\sum_{i=1}^n
\widehat\eta_i\widehat\rho_i^\top
\right]
\left[
\sum_{i=1}^n
\widehat\rho_i\widehat\rho_i^\top
\right]^{-1}
=
\widehat B_{\theta\lambda}
\widehat B_{\lambda\lambda}^{-1}
\in \mathbb R^{1\times p}.
\]

Suppose now that the linearization correction coincides with the projection correction:
\[
\widehat C_{\rm lin}
=
\widehat C_{\rm proj}.
\]
Then
\[
\widehat{\operatorname{Var}}_{\rm MEC}
(\widehat\theta_{\rm MEC})
=
\widehat D_{\theta\theta}^{-1}
\left\{
\sum_{i=1}^n
\left(
\widehat\eta_i
-
\widehat C_{\rm proj}\widehat\rho_i
\right)^2
\right\}
(\widehat D_{\theta\theta}^{-1})^\top.
\]

By the least-squares projection identity,
\[
\begin{aligned}
\sum_{i=1}^n
\left(
\widehat\eta_i
-
\widehat C_{\rm proj}\widehat\rho_i
\right)^2
&=
\sum_{i=1}^n
\widehat\eta_i^2
-
2\widehat C_{\rm proj}
\sum_{i=1}^n
\widehat\rho_i\widehat\eta_i
+
\widehat C_{\rm proj}
\left(
\sum_{i=1}^n
\widehat\rho_i\widehat\rho_i^\top
\right)
\widehat C_{\rm proj}^\top
\\
&=
\widehat B_{\theta\theta}
-
2\widehat C_{\rm proj}\widehat B_{\lambda\theta}
+
\widehat C_{\rm proj}
\widehat B_{\lambda\lambda}
\widehat C_{\rm proj}^\top
\\
&=
\widehat B_{\theta\theta}
-
2\widehat B_{\theta\lambda}
\widehat B_{\lambda\lambda}^{-1}
\widehat B_{\lambda\theta}
+
\widehat B_{\theta\lambda}
\widehat B_{\lambda\lambda}^{-1}
\widehat B_{\lambda\lambda}
\widehat B_{\lambda\lambda}^{-1}
\widehat B_{\lambda\theta}
\\
&=
\sum_{i=1}^n
\widehat\eta_i^2
-
\widehat B_{\theta\lambda}
\widehat B_{\lambda\lambda}^{-1}
\widehat B_{\lambda\theta}.
\end{aligned}
\]

Therefore,
\[
\widehat{\operatorname{Var}}_{\rm MEC}
(\widehat\theta_{\rm MEC})
=
\widehat D_{\theta\theta}^{-1}
\left\{
\widehat B_{\theta\theta}
-
\widehat B_{\theta\lambda}
\widehat B_{\lambda\lambda}^{-1}
\widehat B_{\lambda\theta}
\right\}
(\widehat D_{\theta\theta}^{-1})^\top,
\]
where
\[
\widehat B_{\theta\theta}
=
\sum_{i=1}^n
\widehat\eta_i^2.
\]
Combining this expression with the fixed-weight Lin--Wei/Binder variance gives
\begin{align*}
\widehat{\operatorname{Var}}_{\rm LW}
(\widehat\theta_{\rm MEC})
-
\widehat{\operatorname{Var}}_{\rm MEC}
(\widehat\theta_{\rm MEC})
&=
\widehat D_{\theta\theta}^{-1}
\left[
\widehat B_{\theta\theta}
-
\left\{
\widehat B_{\theta\theta}
-
\widehat B_{\theta\lambda}
\widehat B_{\lambda\lambda}^{-1}
\widehat B_{\lambda\theta}
\right\}
\right]
(\widehat D_{\theta\theta}^{-1})^\top\\
&=
\widehat D_{\theta\theta}^{-1}
\widehat B_{\theta\lambda}
\widehat B_{\lambda\lambda}^{-1}
\widehat B_{\lambda\theta}
(\widehat D_{\theta\theta}^{-1})^\top.
\end{align*}
Because \(\widehat B_{\lambda\lambda}\) is positive definite under the stated nonsingularity condition,
\[
\widehat B_{\theta\lambda}
\widehat B_{\lambda\lambda}^{-1}
\widehat B_{\lambda\theta}
\ge 0.
\]
Hence,
\[
\widehat{\operatorname{Var}}_{\rm MEC}
(\widehat\theta_{\rm MEC})
\le
\widehat{\operatorname{Var}}_{\rm LW}
(\widehat\theta_{\rm MEC}).
\]
The inequality is strict whenever
\[
\widehat B_{\theta\lambda}
\widehat B_{\lambda\lambda}^{-1}
\widehat B_{\lambda\theta}
>0,
\]
that is, whenever the calibration contribution explains a nonzero component of the Cox contribution.
This completes the proof.

\subsection{Computation of the MEC-Cox Estimator via the Dual Newton Solver}
\label{subsec:dual_newton_solver}
This section describes the dual Newton solver used to compute the MEC-Cox calibrated weights, which are then used to construct the MEC-Cox estimator.

\paragraph{Step 1. Construction of the cross-fitted predictor basis.}
Let \(\kappa(i)\in\{1,\ldots,K\}\) denote the fold index such that \(i\in\mathcal J^{(\kappa(i))}\). For each subject, define the cross-fitted calibration-feature vector
\begin{align}
\label{eq:predictor_basis}
\widehat h_i
=
\left(
1,
\widehat S_0^{(-\kappa(i))}(t_1\mid X_i),
\ldots,
\widehat S_0^{(-\kappa(i))}(t_{p-1}\mid X_i)
\right)
\in\mathbb R^{p}.
\end{align}

The non-intercept components of the predictor basis in \eqref{eq:predictor_basis} are obtained from a cross-fitted control-survival learner. Specifically, for each fold \(k\), the learner is trained using only the external-control subjects in the training set, \(\mathcal I_0\cap\mathcal J^{(-k)}\), and is then evaluated for all validation subjects \(i\in\mathcal J^{(k)}\). Thus, \(\widehat S_0^{(-\kappa(i))}(t_\ell\mid X_i)\) is evaluated out of sample for subject \(i\).

We consider two implementations of the control-survival learner:
\begin{enumerate}[leftmargin=1.5em, label={}, itemsep=0.5em]
    \item \textbf{\textit{Cox proportional hazards regression \citep{cox1972regression,cox1975partial}.}}
    The first implementation fits a Cox proportional hazards working model among the external-control subjects in the training fold,
    \[
    \lambda_0^{(-k)}(t\mid X)
    =
    \lambda_{00}^{(-k)}(t)\exp\{X^\top\beta_0^{(-k)}\},
    \qquad
    i\in \mathcal I_0\cap\mathcal J^{(-k)}.
    \]
    Let \(\widehat\beta_0^{(-k)}\) and \(\widehat\Lambda_{00}^{(-k)}(t)\) denote the fitted regression coefficient and baseline cumulative hazard estimator, respectively. The predicted control-survival probability for a validation subject \(i\in\mathcal J^{(k)}\) is then
    \[
    \widehat S_{0,\mathrm{Cox}}^{(-k)}(t_\ell\mid X_i)
    =
    \exp\left[
    -\widehat\Lambda_{00}^{(-k)}(t_\ell)
    \exp\{X_i^\top\widehat\beta_0^{(-k)}\}
    \right],
    \qquad
    \ell=1,\ldots,L.
    \]

    \item \textbf{\textit{Random survival forest \citep{ishwaran2008random}.}} The second implementation fits a random survival forest among the external-control subjects in the training fold. For each tree \(b=1,\ldots,B\), let \(\widehat S_{0b}^{(-k)}(t\mid X_i)\) denote the terminal-node survival estimate for a validation subject with covariates \(X_i\). The random-survival-forest prediction is
    \[
    \widehat S_{0,\mathrm{RSF}}^{(-k)}(t_\ell\mid X_i)
    =
    \frac{1}{B}
    \sum_{b=1}^{B}
    \widehat S_{0b}^{(-k)}(t_\ell\mid X_i),
    \qquad
    \ell=1,\ldots,L.
    \]
    This provides a flexible, nonparametric estimate of the conditional control survival function \(S_0(t\mid X_i)\), evaluated at the landmark times \(t_1,\ldots,t_{p-1}\).
\end{enumerate}
In both cases, the resulting predictor basis summarizes the estimated counterfactual control-survival profile for each subject and is used only as a calibration feature.

\paragraph{Step 2. Baseline ATT odds weights.}
We next construct the baseline weights around which the MEC calibration is performed. For each fold \(k=1,\ldots,K\), estimate the source propensity score
\[
\pi(X_i)=\Pr(A_i=1\mid X_i)
\]
using the pooled training sample \(\mathcal J^{(-k)}\). Denote the resulting fold-specific estimator by \(\widehat\pi^{(-k)}(X)\). For each validation subject \(i\in\mathcal J^{(k)}\), define the cross-fitted propensity-score estimate
\begin{align}
\label{eq:pi_i_cross_fitted}
\widehat\pi_i
=
\widehat\pi^{(-k)}(X_i),    
\end{align}
possibly after truncation to avoid extreme values. For external-control subjects \(i\in\mathcal I_0\), define the corresponding estimated ATT odds weight by
\[
\widehat q_i
=
\frac{\widehat\pi_i}{1-\widehat\pi_i}.
\]
We then normalize these external-control weights to have the same total mass as the treated trial cohort. Specifically, with \(n_1=|\mathcal I_1|\), define
\[
\widehat d_i
=
\frac{n_1\widehat q_i}
{\sum_{j\in\mathcal I_0}\widehat q_j},
\qquad i\in\mathcal I_0.
\]
Then
\[
\sum_{i\in\mathcal I_0}\widehat d_i=n_1,
\]
so the baseline external-control weights are normalized to the size of the treated trial cohort. These normalized ATT odds weights serve as the baseline weights in the subsequent Bregman calibration step.

We consider two broad implementations of the source propensity-score learner \eqref{eq:pi_i_cross_fitted}:
\begin{enumerate}[leftmargin=1.5em, label={}, itemsep=0.5em]
    \item \textbf{\textit{Logistic regression.}} The first implementation uses a logistic regression model fitted on the pooled training sample, $\pi_\gamma(X_i)
    =
    \Pr(A_i=1\mid X_i;\gamma)
    =\exp(\gamma^\top \widetilde X_i)/\{1+\exp(\gamma^\top \widetilde X_i)\},$ 
    where \(\widetilde X_i\) includes an intercept and selected baseline covariates. Let \(\widehat\gamma^{(-k)}\) denote the estimator obtained using subjects in \(\mathcal J^{(-k)}\). Then, for \(i\in\mathcal J^{(k)}\),
    \[
    \widehat\pi^{(-k)}_{\mathrm{glm}}(X_i)
    =
    \frac{\exp\{(\widehat\gamma^{(-k)})^\top \widetilde X_i\}}
    {1+\exp\{(\widehat\gamma^{(-k)})^\top \widetilde X_i\}}.
    \]

    \item \textbf{\textit{Flexible machine-learning classifier.}} The second implementation estimates \(\pi(X_i)\) using a flexible binary classifier trained on the pooled training sample \(\mathcal J^{(-k)}\), such as DL \citep{lecun2015deep}, random forest \citep{breiman2001random}, or BART \citep{chipman2010bart}. The cross-fitted estimate for \(i\in\mathcal J^{(k)}\) is
    \[
    \widehat\pi^{(-k)}_{\mathrm{ML}}(X_i)
    =
    \widehat \Pr(A_i=1\mid X_i).
    \]
\end{enumerate}
In both implementations, the propensity-score learner is trained on the pooled training sample and evaluated on the held-out validation fold, so that \(\widehat\pi_i\) is computed out of sample for each subject.

\paragraph{Step 3. Formulation of the primal problem.}
Write the index sets as
\[
\mathcal I_0=\{i_1^0,\ldots,i_{n_0}^0\},
\qquad
\mathcal I_1=\{i_1^1,\ldots,i_{n_1}^1\},
\]
where \(i_m^0\) denotes the index of the \(m\)th external-control subject and \(i_m^1\) denotes the index of the \(m\)th treated trial subject.

Stacking the vectors \(\widehat h_i\) \eqref{eq:predictor_basis} obtained in Step 1 within each cohort in a row, define the calibration-feature matrices
\[
H_0
=
\begin{pmatrix}
\widehat h_{i_1^0}\\
\vdots\\
\widehat h_{i_{n_0}^0}
\end{pmatrix}
=
\begin{pmatrix}
1 & \widehat S_0^{(-\kappa(i_1^0))}(t_1\mid X_{i_1^0}) & \cdots & \widehat S_0^{(-\kappa(i_1^0))}(t_{p-1}\mid X_{i_1^0})\\
\vdots & \vdots & \ddots & \vdots\\
1 & \widehat S_0^{(-\kappa(i_{n_0}^0))}(t_1\mid X_{i_{n_0}^0}) & \cdots & \widehat S_0^{(-\kappa(i_{n_0}^0))}(t_{p-1}\mid X_{i_{n_0}^0})
\end{pmatrix}
\in\mathbb R^{n_0\times p},
\]
and
\[
H_1
=
\begin{pmatrix}
\widehat h_{i_1^1}\\
\vdots\\
\widehat h_{i_{n_1}^1}
\end{pmatrix}
=
\begin{pmatrix}
1 & \widehat S_0^{(-\kappa(i_1^1))}(t_1\mid X_{i_1^1}) & \cdots & \widehat S_0^{(-\kappa(i_1^1))}(t_{p-1}\mid X_{i_1^1})\\
\vdots & \vdots & \ddots & \vdots\\
1 & \widehat S_0^{(-\kappa(i_{n_1}^1))}(t_1\mid X_{i_{n_1}^1}) & \cdots & \widehat S_0^{(-\kappa(i_{n_1}^1))}(t_{p-1}\mid X_{i_{n_1}^1})
\end{pmatrix}
\in\mathbb R^{n_1\times p}.
\]
Here \(H_0\) contains the cross-fitted control-survival features for external-control subjects, whereas \(H_1\) contains the same features evaluated at the covariate values of treated trial subjects.

The target calibration total is then expressed as the \(p\)-dimensional vector
\[
T_H
=
H_1^\top \mathbf 1_{n_1}
=
\sum_{i\in\mathcal I_1}\widehat h_i
\in \mathbb{R}^{p}.
\]

For \(i\in\mathcal I_0\), let \(\widehat d_i\) denote the normalized baseline ATT odds weight,
\[
\widehat d_i
=
\frac{n_1\widehat q_i}
{\sum_{j\in\mathcal I_0}\widehat q_j},
\qquad
\widehat q_i
=
\frac{\widehat\pi_i}{1-\widehat\pi_i}
=
\frac{\widehat\pi^{(-k)}(X_i)}{1-\widehat\pi^{(-k)}(X_i)},
\]
where $\widehat\pi_i$ is obtained in Step 2.

The MEC-Cox calibrated weights are obtained by the Bregman projection
\[
\widehat w
=
\arg\min_{\{w_i>0:i\in\mathcal I_0\}}
\sum_{i\in\mathcal I_0}
D_G(w_i\|\widehat d_i)
\]
subject to
\[
H_0^\top w
=
H_1^\top\mathbf 1_{n_1}
=
T_H\in \mathbb{R}^{p},
\]
where
\[
D_G(w_i\|\widehat d_i)
=
G(w_i)-G(\widehat d_i)-g(\widehat d_i)(w_i-\widehat d_i),
\qquad
g=G',
\]
and \(G\) is a strictly convex generator. 

Directly solving this primal problem would require optimizing over \(n_0\) external-control weights. Instead, MEC-Cox solves the equivalent low-dimensional dual problem, whose dimension is only \(p\), the number of calibration features.

\paragraph{Step 4. Conversion to the dual problem.}
We next convert the constrained primal problem into a low-dimensional dual problem. Recall that the MEC-Cox calibration problem is
\[
\widehat w
=
\arg\min_{\{w_i>0:i\in\mathcal I_0\}}
\sum_{i\in\mathcal I_0}
D_G(w_i\|\widehat d_i)
\quad
\text{subject to}
\quad
H_0^\top w=T_H \in \mathbb{R}^{p}.
\]

Let \(\lambda\in\mathbb R^p\) denote the Lagrange multiplier associated with the calibration constraint. Define the Lagrangian
\[
\mathcal L(w,\lambda)
=
\sum_{i\in\mathcal I_0}
D_G(w_i\|\widehat d_i)
-
\lambda^\top
\left(
H_0^\top w
-
T_H
\right).
\]
Because
\[
D_G(w_i\|\widehat d_i)
=
G(w_i)-G(\widehat d_i)-g(\widehat d_i)(w_i-\widehat d_i),
\qquad g=G',
\]
we have
\[
\frac{\partial}{\partial w_i}
D_G(w_i\|\widehat d_i)
=
g(w_i)-g(\widehat d_i).
\]
Therefore, the stationarity condition of the Karush--Kuhn--Tucker equations is
\[
\frac{\partial \mathcal L(w,\lambda)}{\partial w_i}
=
g(w_i)-g(\widehat d_i)-\lambda^\top\widehat h_i
=
0,
\qquad i\in\mathcal I_0.
\]
Equivalently,
\begin{align}
\label{eq:g_eq_1}
g(w_i)
=
g(\widehat d_i)+\lambda^\top\widehat h_i.    
\end{align}
Since \(G\) is strictly convex, \(g=G'\) is strictly increasing and hence invertible on its domain. Thus, for any fixed \(\lambda\), the primal weight satisfying the stationarity condition is
\[
w_i(\lambda)
=
g^{-1}\{g(\widehat d_i)+\lambda^\top\widehat h_i\},
\qquad i\in\mathcal I_0.
\]

The remaining KKT condition is primal feasibility 
\[
H_0^\top w(\lambda)=T_H.
\]
Substituting the dual representation \(w_i(\lambda)\) into the calibration constraint yields the dual estimating equation
\begin{align}
\label{eq:dual estimating equation}
F(\lambda)
=
\sum_{i\in\mathcal I_0}
w_i(\lambda)\widehat h_i
-
T_H
=
0.    
\end{align}
Equivalently,
\[
F(\lambda)
=
\sum_{i\in\mathcal I_0}
\widehat h_i
g^{-1}\{g(\widehat d_i)+\lambda^\top\widehat h_i\}
-
\sum_{i\in\mathcal I_1}\widehat h_i
=
0,
\]
or, in matrix form,
\[
F(\lambda)
=
H_0^\top w(\lambda)-T_H.
\]
Thus, instead of optimizing over the \(n_0\)-dimensional vector \(w\), MEC-Cox solves the \(p\)-dimensional nonlinear equation \(F(\lambda)=0\). 

It is important to note that this dual estimating equation \eqref{eq:dual estimating equation} will later be used as the calibration component of the stacked estimating-equation system for sandwich variance estimation for the MEC-Cox estimator.

\paragraph{Step 5. Dual Newton solver.}
The Newton update is based on the Jacobian of \(F(\lambda)\) \eqref{eq:dual estimating equation}. We view the dual calibration map as
\[
F:\mathcal D_{\lambda}\subseteq \mathbb R^{p}\longrightarrow \mathbb R^{p},
\]
where
\[
\mathcal D_{\lambda}
=
\left\{
\lambda\in\mathbb R^{p}:
g(\widehat d_i)+\lambda^\top \widehat h_i
\in \operatorname{dom}(g^{-1})
\ \text{for all } i\in\mathcal I_0
\right\}.
\]
That is, the domain of \(F\) consists of all dual parameters \(\lambda\) for which the implied weights $w_i(\lambda)
=
g^{-1}\{g(\widehat d_i)+\lambda^\top \widehat h_i\},$ $(i\in\mathcal I_0)$ are well-defined. 

By \eqref{eq:g_eq_1}, let
\[
\nu_i(\lambda)
=
g(\widehat d_i)+\lambda^\top\widehat h_i,
\qquad
w_i(\lambda)=g^{-1}\{\nu_i(\lambda)\}.
\]
Since
\[
\frac{\partial w_i(\lambda)}{\partial \nu_i}
=
\frac{1}{g'\{w_i(\lambda)\}},
\]
the Jacobian of \(F(\lambda)\) is
\[
J(\lambda)
=
\frac{\partial F(\lambda)}{\partial \lambda^\top}
=
\sum_{i\in\mathcal I_0}
\frac{1}{g'\{w_i(\lambda)\}}
\widehat h_i\widehat h_i^\top \in \mathbb{R}^{p \times p}.
\]
In matrix form,
\[
J(\lambda)
=
H_0^\top
\operatorname{diag}
\left[
\frac{1}{g'\{w_i(\lambda)\}}:i\in\mathcal I_0
\right]
H_0 \in \mathbb{R}^{p \times p}.
\]
At iteration \(m\), the Newton direction \(\Delta_m\) solves
\[
\{J(\lambda_m)+\rho I_{p}\}\Delta_m
=
F(\lambda_m),
\]
where \(\rho>0\) is a small ridge constant used for numerical stability. The dual parameter is updated by
\[
\lambda_{m+1}
=
\lambda_m-\alpha_m\Delta_m=
\lambda_m-\alpha_m \{J(\lambda_m)+\rho I_{p}\}^{-1}F(\lambda_m),
\]
where \(\alpha_m\in(0,1]\) is chosen by step-halving if necessary to keep all dual  arguments $\nu_i(\lambda_{m+1})
=
g(\widehat d_i)+\lambda_{m+1}^\top\widehat h_i$ inside the domain of \(g^{-1}\). The iteration stops when the calibration residual is sufficiently small, for example when
\[
\|F(\lambda_m)\|_2
=
\left\|
H_0^\top w(\lambda_m)-T_H
\right\|_2
\le \varepsilon
\]
for a prespecified tolerance \(\varepsilon\).

The ridge constant \(\rho\), the step-halving rule, and the stopping tolerance \(\varepsilon\) are user-specified numerical choices. As a default, we use \(\rho=10^{-8}\), initialize \(\alpha_m=1\), and apply step-halving only when necessary to keep the dual argument within the domain of \(g^{-1}\). The stopping tolerance is \(\varepsilon=10^{-10}\), with at most 100 Newton iterations.

\paragraph{Step 6. Final MEC-Cox estimator.}
After convergence of the dual Newton solver, let $\widehat\lambda$ denote the resulting dual
solution and define the calibrated external-control weights by
\[
\widehat w_i
=
g^{-1}\{g(\widehat d_i)+\widehat\lambda^\top \widehat h_i\},
\qquad i\in \mathcal I_0 .
\]
The final MEC-Cox estimating-equation weight is
\[
\widetilde\omega_i
=
A_i + (1-A_i)\widehat w_i,
\qquad i=1,\ldots,n.
\]
Thus, treated trial patients retain weight one, whereas external-control patients receive the MEC-updated ATT weights.

Using these final weights, define
\[
S_{\rm MEC, \omega}^{(r)}(t;\theta)
=
\sum_{i=1}^n
\widetilde\omega_i
\mathcal{Y}_i(t)\exp(\theta A_i)A_i^r,
\qquad r=0,1,
\]
and
\[
\overline A_{\rm MEC, \omega}(t;\theta)
=
\frac{S_{\rm MEC, \omega}^{(1)}(t;\theta)}
     {S_{\rm MEC, \omega}^{(0)}(t;\theta)} .
\]
The MEC-Cox estimator is the solution to the weighted Cox estimating equation
\[
U_n^{\rm MEC}(\theta)
=
\sum_{i=1}^n
\int
\widetilde\omega_i
\{A_i-\overline A_{\rm MEC, \omega}(t;\theta)\}
\,d\mathcal{N}_i(t)
=
0.
\]
Equivalently, $\widehat\theta_{\rm MEC}$ maximizes the weighted Cox partial log-likelihood
\[
\ell_{\rm MEC}(\theta)
=
\sum_{i=1}^n
\widetilde\omega_i
\int
\left[
\theta A_i
-
\log\{S_{\rm MEC, \omega}^{(0)}(t;\theta)\}
\right]
\,d\mathcal{N}_i(t).
\]
The final MEC-Cox hazard-ratio estimator is therefore
\[
\widehat{\rm HR}_{\rm MEC}
=
\exp(\widehat\theta_{\rm MEC}).
\]

\paragraph{Computational algorithm.} The computational procedure for the MEC-Cox estimator is summarized in Algorithm~\ref{alg:mec_cox}. 

\begin{algorithm}[H]
\caption{Computational algorithm for the MEC-Cox estimator}
\label{alg:mec_cox}
\begin{algorithmic}[1]
\State Partition the pooled sample into \(K\) folds, stratified by the source indicator \(A_i\).
\State For each validation subject \(i\in\mathcal J^{(k)}\), compute cross-fitted nuisance estimates
\[
\widehat\pi_i=\widehat\pi^{(-k)}(X_i),
\qquad
\widehat h_i
=
\left(1,\widehat S_0^{(-k)}(t_1\mid X_i),\ldots,
\widehat S_0^{(-k)}(t_{p-1}\mid X_i)\right),
\]
where \(\widehat\pi^{(-k)}\) is trained on \(\mathcal J^{(-k)}\), and \(\widehat S_0^{(-k)}\) is trained using external controls in \(\mathcal I_0\cap\mathcal J^{(-k)}\).
\State For \(i\in\mathcal I_0\), define the normalized baseline ATT odds weight
\[
\widehat d_i
=
\frac{n_1\widehat q_i}{\sum_{j\in\mathcal I_0}\widehat q_j},
\qquad
\widehat q_i=\frac{\widehat\pi_i}{1-\widehat\pi_i}.
\]
\State Solve the dual calibration equation
\begin{align}
\nonumber
F(\lambda)
=
\sum_{i\in\mathcal I_0}
\widehat h_i
g^{-1}\{g(\widehat d_i)+\lambda^\top\widehat h_i\}
-
\sum_{i\in\mathcal I_1}\widehat h_i
=
0    
\end{align}
by a damped Newton solver, and denote the solution by \(\widehat\lambda\).
\State Recover the MEC-updated external-control weights and the final
estimating-equation weights:
\[
\widehat w_i
=
g^{-1}\{g(\widehat d_i)+\widehat\lambda^\top\widehat h_i\},
\qquad
\widetilde\omega_i
=
\omega_i(\widehat\lambda)=
A_i+(1-A_i)\widehat w_i.
\]
\State Report \(\widehat{HR}_{\mathrm{MEC}}=\exp(\widehat\theta_{\mathrm{MEC}})\)
by solving
\[
U_n^{\mathrm{MEC}}(\theta)
=
\sum_{i=1}^n
\int
\widetilde\omega_i
\{A_i-\overline A_{\rm MEC,\omega}(t;\theta)\}
\,d\mathcal{N}_i(t)
=
0.
\]
\end{algorithmic}
\end{algorithm}

The immediate computational benefit of MEC-Cox is that the calibration step is carried out in the low-dimensional dual space rather than over the \(n_0\)-dimensional primal weight vector. Specifically, the dual Newton solver operates on the \(p\)-dimensional parameter \(\lambda\), where \(p\) is the
dimension of the calibration basis, including the intercept. Thus, it does not depend on the original covariate dimension or the external-control sample size. At each iteration, the solver forms the \(p\times p\) Jacobian
\[
J(\lambda)
=
\frac{\partial F(\lambda)}{\partial \lambda^\top}
=
\sum_{i\in\mathcal I_0}
\frac{1}{g'\{w_i(\lambda)\}}
\widehat h_i\widehat h_i^\top \in \mathbb{R}^{p \times p}
\]
which requires \(O(n_0p^2)\) operations to assemble and \(O(p^3)\) operations to solve for the Newton step. Since \(p\) is typically small, this cost is negligible even for large external-control cohorts with \(n_0>1000\). Consequently, MEC-Cox is computationally fast and numerically stable, while still allowing flexible machine-learning methods to be used in constructing the survival-based calibration features.

\subsection{Derivation of the Lin--Wei/Binder Empirical Cox Contribution in MEC-Cox Variance Estimation}
\label{subsec:lin_wei_binder_empirical_contribution}
\paragraph{Overview.}
We derive the empirical Cox contribution used in the MEC-Cox stacked sandwich
variance estimator. Conditional on the calibration parameter \(\lambda\), the
calibrated weights \(\widetilde\omega_i(\lambda)\) are treated as fixed. Under
this fixed-weight view, the relevant Cox contribution is the Lin--Wei/Binder
subject-level robust contribution, adapted to the MEC-Cox weighted risk sets.
Our derivation is included to make this contribution explicit in the present
notation: we perturb the contribution of one subject at a time in the weighted
Cox score, including both that subject's event contribution and its effect on the
weighted risk-set averages. This yields the finite-sample subject-level quantity
\(\eta_i(\theta,\lambda)\) used in the stacked sandwich variance estimator. The resulting contribution is therefore the Lin--Wei/Binder robust Cox
contribution \citep{lin1989robust,binder1992fitting} computed conditionally on the fixed dual parameter \(\lambda\).

\paragraph{Basic setup and notations.}
For a fixed \(\lambda\), define the calibrated external-control weight
\[
w_i(\lambda)
=
g^{-1}\{g(\widehat d_i)+\lambda^\top\widehat h_i\},
\qquad i\in\mathcal I_0,
\]
and the corresponding full MEC-Cox estimating-equation weight
\[
\widetilde\omega_i(\lambda)
=
\begin{cases}
1, & i\in\mathcal I_1,\\
w_i(\lambda), & i\in\mathcal I_0.
\end{cases}
\]

For this fixed \(\lambda\), define the weighted risk-set sums
\begin{align}
\label{eq:weighted_risk_set_sums}
S_\lambda^{(r)}(t;\theta)
=
\sum_{i=1}^n
\widetilde\omega_i(\lambda)
\mathcal{Y}_i(t)\exp(\theta A_i)A_i^r,
\qquad r=0,1,    
\end{align}
and the weighted risk-set treatment average
\[
\overline A_{\rm MEC,\omega}(t;\theta,\lambda)
=
\frac{
S_\lambda^{(1)}(t;\theta)
}{
S_\lambda^{(0)}(t;\theta)
}.
\]
The MEC-weighted Cox score is
\begin{align}
\label{eq:supp_mec_cox_score_fixed_lambda}
U_{\rm MEC, \theta}(\theta,\lambda)
&=
\sum_{i=1}^n
\int
\widetilde\omega_i(\lambda)
\{A_i-\overline A_{\rm MEC,\omega}(t;\theta,\lambda)\}
\,d\mathcal{N}_i(t).
\end{align}

Note that \(U_{\rm MEC,\theta}(\theta,\lambda)\) \eqref{eq:supp_mec_cox_score_fixed_lambda} is not an additive
estimating equation in the usual sense. Although it is written as a sum over
subjects, the risk-set average
\(\overline A_{\rm MEC,\omega}(t;\theta,\lambda)\) depends on the full
weighted risk set at time \(t\). Hence, the raw summand
\[
\int
\widetilde\omega_i(\lambda)
\{A_i-\overline A_{\rm MEC,\omega}(t;\theta,\lambda)\}
\,d\mathcal N_i(t)
\]
is not, by itself, the appropriate subject-level contribution for sandwich
variance estimation. Instead, we need the subject-level contribution,
which accounts for both the direct event contribution of subject \(i\) and the
indirect effect of subject \(i\) on the weighted risk-set averages. To derive
this contribution, we use an infinitesimal perturbation that upweights subject
\(i\) in the empirical score while keeping the fixed estimating-equation weights
\(\widetilde\omega_j(\lambda)\), \(j=1,\ldots,n\), unchanged.

For notational simplicity in the derivation below, suppress the dependence on \((\theta,\lambda)\) and write
\begin{align}
\label{eq:notation_1}
\omega_i=\widetilde\omega_i(\lambda),
\qquad
\overline A(t)=\overline A_{\rm MEC,\omega}(t;\theta,\lambda),
\qquad
S^{(r)}(t)=S_\lambda^{(r)}(t;\theta),\quad r=0,1.    
\end{align}

Define the weighted treatment-event and weighted event increments by
\begin{align}
\label{eq:notation_2}
d\mathcal{N}_{\omega,A}(t)
=
\sum_{i=1}^n \omega_i A_i\,d\mathcal{N}_i(t),
\qquad
d\mathcal{N}_{\omega}(t)
=
\sum_{i=1}^n \omega_i\,d\mathcal{N}_i(t).    
\end{align}
Then the MEC-weighted Cox score in \eqref{eq:supp_mec_cox_score_fixed_lambda} can be rewritten as
\begin{align}
U_{\rm MEC, \theta}(\theta,\lambda)
&=
\sum_{i=1}^n
\int
\omega_i
\{A_i-\overline A(t)\}
\,d\mathcal{N}_i(t)
=
\int
\sum_{i=1}^n
\omega_i A_i\,d\mathcal{N}_i(t)
-
\int
\overline A(t)
\sum_{i=1}^n
\omega_i\,d\mathcal{N}_i(t)
\nonumber\\
\label{eq:MEC_cox_score_simple}
&=
\int d\mathcal{N}_{\omega,A}(t)
-
\int \overline A(t)\,d\mathcal{N}_{\omega}(t).
\end{align}

\paragraph{Basic algebraic idea (general).}
We first describe the one-subject perturbation device used to extract a
finite-sample subject-level linearized contribution. Consider a simple additive
quantity
\begin{align}
\label{eq:empirical_sum}
L=\sum_{j=1}^n z_j,
\end{align}
where \(z_j\) denotes the contribution of subject \(j\). Suppose we are
interested in the contribution of a particular subject
\(i^\star\in\{1,\ldots,n\}\). Perturbing subject \(i^\star\) by an infinitesimal
amount \(\varepsilon\) gives
\[
L^{(\varepsilon;i^\star)}
=
L+\varepsilon z_{i^\star},
\]
and hence
\[
\left.
\frac{d}{d\varepsilon}
L^{(\varepsilon;i^\star)}
\right|_{\varepsilon=0}
=
z_{i^\star}.
\]
Thus, for a simple additive empirical quantity such as
\eqref{eq:empirical_sum}, the subject-level contribution is obtained directly
by differentiating the perturbed quantity. Note that, \eqref{eq:empirical_sum} can be expressed as
\[
L
=
\sum_{j=1}^n
\left.
\frac{d}{d\varepsilon}
L^{(\varepsilon;j)}
\right|_{\varepsilon=0}.
\]

More generally, suppose an estimating equation can be written as a smooth
function of several empirical quantities,
\[
T
=
H(L_1,\ldots,L_m),
\qquad
L_\ell=\sum_{j=1}^n z_{\ell j},
\quad \ell=1,\ldots,m.
\]
Assume that \(H\) is homogeneous of degree one in its arguments; that is,
\begin{align}
\label{eq:homogeneous_of_degree_one}
H(cL_1,\ldots,cL_m)
=
cH(L_1,\ldots,L_m),
\qquad c>0.
\end{align}
Perturbing the contribution of subject \(i^\star\) to each empirical quantity
\(L_\ell\) gives
\[
L_\ell^{(\varepsilon;i^\star)}
=
L_\ell+\varepsilon z_{\ell i^\star},
\qquad \ell=1,\ldots,m.
\]
Therefore, the perturbed version of \(T\) is
\[
T^{(\varepsilon;i^\star)}
=
H\{L_1^{(\varepsilon;i^\star)},\ldots,
L_m^{(\varepsilon;i^\star)}\}.
\]
The subject-level linearized contribution of subject \(i^\star\) to \(T\) is
\[
\eta_{i^\star}
:=
\left.
\frac{d}{d\varepsilon}
T^{(\varepsilon;i^\star)}
\right|_{\varepsilon=0}
=
\sum_{\ell=1}^m
\frac{\partial H}{\partial L_\ell}
(L_1,\ldots,L_m)
z_{\ell i^\star},
\]
where the last equality follows from the chain rule. Summing over subjects gives
\begin{align*}
\sum_{i=1}^n \eta_i
&=
\sum_{i=1}^n
\sum_{\ell=1}^m
\frac{\partial H}{\partial L_\ell}
(L_1,\ldots,L_m)
z_{\ell i} =
\sum_{\ell=1}^m
\frac{\partial H}{\partial L_\ell}
(L_1,\ldots,L_m)
\sum_{i=1}^n z_{\ell i} \\
&=
\sum_{\ell=1}^m
\frac{\partial H}{\partial L_\ell}
(L_1,\ldots,L_m)
L_\ell =
H(L_1,\ldots,L_m)
=
T.
\end{align*}
Here, the first equality substitutes the chain-rule expression for \(\eta_i\). The
second equality rearranges the finite sums. The third equality uses
\(L_\ell=\sum_{i=1}^n z_{\ell i}\). The fourth equality follows from the
degree-one homogeneity property in \eqref{eq:homogeneous_of_degree_one}. More
specifically, differentiating both sides of
\eqref{eq:homogeneous_of_degree_one} with respect to \(c\) gives
\[
\sum_{\ell=1}^m
\frac{\partial H}{\partial L_\ell}
(cL_1,\ldots,cL_m)
L_\ell
=
H(L_1,\ldots,L_m).
\]
Evaluating this identity at \(c=1\) yields
\[
\sum_{\ell=1}^m
\frac{\partial H}{\partial L_\ell}
(L_1,\ldots,L_m)
L_\ell
=
H(L_1,\ldots,L_m).
\]
Thus, when the empirical functional is homogeneous of degree one in the
empirical quantities being perturbed, the subject-level linearized
contributions provide an exact finite-sample decomposition:
\[
T
=
\sum_{i=1}^n \eta_i
=
\sum_{i=1}^n
\left.
\frac{d}{d\varepsilon}
T^{(\varepsilon;i)}
\right|_{\varepsilon=0}.
\]

The MEC-weighted Cox score in \eqref{eq:MEC_cox_score_simple} has the
degree-one homogeneity property. Seeing the analytical form of the score function,
\begin{align*}
U_{\rm MEC,\theta}(\theta,\lambda)
&=
H\left(d\mathcal N_{\omega,A}(t), 
d\mathcal N_{\omega}(t),
S^{(1)}(t),
S^{(0)}(t)
\right)=
\int d\mathcal N_{\omega,A}(t)
-
\int
\frac{S^{(1)}(t)}{S^{(0)}(t)}
\,d\mathcal N_{\omega}(t),
\end{align*}
it is notable that the score is a function of the empirical quantities
\(d\mathcal N_{\omega,A}(t)\), \(d\mathcal N_{\omega}(t)\),
\(S^{(1)}(t)\), and \(S^{(0)}(t)\). If all these empirical quantities are
multiplied by the same constant \(c>0\), then
\[
c\,d\mathcal N_{\omega,A}(t)
-
\frac{cS^{(1)}(t)}{cS^{(0)}(t)}
\,c\,d\mathcal N_{\omega}(t)
=
c\left\{
d\mathcal N_{\omega,A}(t)
-
\frac{S^{(1)}(t)}{S^{(0)}(t)}
\,d\mathcal N_{\omega}(t)
\right\}.
\]
Therefore, the full MEC-weighted Cox score is homogeneous of degree one in the
empirical quantities being perturbed. The nonlinearity enters only through the
risk-set average, which is a ratio of two weighted risk-set sums and hence is
homogeneous of degree zero.

Consequently, using the one-subject perturbation technique described above, we
define the subject-level linearized contribution in the direction of subject
\(i^\star\) as
\[
\eta_{i^\star}(\theta,\lambda)
=
\left.
\frac{d}{d\varepsilon}
U_{{\rm MEC},\theta}^{(\varepsilon;i^\star)}(\theta,\lambda)
\right|_{\varepsilon=0}.
\]
Because \(i^\star\) is arbitrary, this construction yields one contribution
\(\eta_i(\theta,\lambda)\) for each subject \(i=1,\ldots,n\). By the
degree-one homogeneity property of the MEC-weighted Cox score, these
contributions satisfy the finite-sample decomposition
\[
U_{\rm MEC,\theta}(\theta,\lambda)
=
\sum_{i=1}^n \eta_i(\theta,\lambda)
=
\sum_{i=1}^n
\left.
\frac{d}{d\varepsilon}
U_{{\rm MEC},\theta}^{(\varepsilon;i)}(\theta,\lambda)
\right|_{\varepsilon=0}.
\]
This identity shows that \(\eta_i(\theta,\lambda)\) provides a valid
subject-level decomposition of the full MEC-weighted Cox score. Unlike the raw
event summand in \eqref{eq:supp_mec_cox_score_fixed_lambda}, \(\eta_i(\theta,\lambda)\) accounts for both the direct event
contribution of subject \(i\) and the indirect effect of subject \(i\) on the
weighted risk-set averages.
\paragraph{Applying the basic idea to the MEC-weighted Cox score functional.}

Apply this perturbation first to the weighted treatment-event increment
\(d\mathcal{N}_{\omega,A}(t)\) in
\eqref{eq:MEC_cox_score_simple}. Under the perturbation in the direction of
subject \(i^\star\),
\begin{align}
\label{eq:dN_omega,A}
d\mathcal{N}_{\omega,A}^{(\varepsilon;i^\star)}(t)
&=
\sum_{j=1}^n \omega_j A_j\,d\mathcal{N}_j(t)
+
\varepsilon \omega_{i^\star} A_{i^\star}\,
d\mathcal{N}_{i^\star}(t) =
d\mathcal{N}_{\omega,A}(t)
+
\varepsilon \omega_{i^\star} A_{i^\star}\,
d\mathcal{N}_{i^\star}(t).
\end{align}
Similarly, for the weighted event increment \(d\mathcal{N}_{\omega}(t)\) in
\eqref{eq:MEC_cox_score_simple},
\begin{align}
\label{eq:dN_omega}
d\mathcal{N}_{\omega}^{(\varepsilon;i^\star)}(t)
&=
\sum_{j=1}^n \omega_j\,d\mathcal{N}_j(t)
+
\varepsilon \omega_{i^\star}\,d\mathcal{N}_{i^\star}(t) =
d\mathcal{N}_{\omega}(t)
+
\varepsilon \omega_{i^\star}\,d\mathcal{N}_{i^\star}(t).
\end{align}

The same perturbation applies to the weighted risk-set sums
\eqref{eq:weighted_risk_set_sums}. Increasing only the empirical contribution
of subject \(i^\star\) gives
\begin{align*}
S^{(r,\varepsilon;i^\star)}(t)
&=
\sum_{j=1}^n
\omega_j\mathcal{Y}_j(t)\exp(\theta A_j)A_j^r
+
\varepsilon
\omega_{i^\star}\mathcal{Y}_{i^\star}(t)
\exp(\theta A_{i^\star})A_{i^\star}^r\\
&=
S^{(r)}(t)
+
\varepsilon
\omega_{i^\star}\mathcal{Y}_{i^\star}(t)
\exp(\theta A_{i^\star})A_{i^\star}^r,
\qquad r=0,1.
\end{align*}

Therefore, the perturbed risk-set average is
\[
\overline A^{(\varepsilon;i^\star)}(t)
=
\frac{S^{(1,\varepsilon;i^\star)}(t)}
{S^{(0,\varepsilon;i^\star)}(t)}.
\]

Differentiating this ratio at \(\varepsilon=0\) gives
\begin{align}
\dot{\overline A}_{i^\star}(t)
&=
\left.
\frac{\partial}{\partial\varepsilon}
\overline A^{(\varepsilon;i^\star)}(t)
\right|_{\varepsilon=0}
=
\left.
\frac{\partial}{\partial\varepsilon}
\frac{S^{(1,\varepsilon;i^\star)}(t)}
{S^{(0,\varepsilon;i^\star)}(t)}
\right|_{\varepsilon=0}
\nonumber\\
&=
\frac{
\dot S_{i^\star}^{(1)}(t)S^{(0)}(t)
-
S^{(1)}(t)\dot S_{i^\star}^{(0)}(t)
}{
\{S^{(0)}(t)\}^2
}
\nonumber\\
&=
\frac{
\omega_{i^\star}\mathcal{Y}_{i^\star}(t)
\exp(\theta A_{i^\star})A_{i^\star} S^{(0)}(t)
-
S^{(1)}(t)
\omega_{i^\star}\mathcal{Y}_{i^\star}(t)
\exp(\theta A_{i^\star})
}{
\{S^{(0)}(t)\}^2
}
\nonumber\\
&=
\frac{
\omega_{i^\star}\mathcal{Y}_{i^\star}(t)
\exp(\theta A_{i^\star})
\{A_{i^\star}-\overline A(t)\}
}{
S^{(0)}(t)
}.
\label{eq:A_dot}
\end{align}
This derivative measures how subject \(i^\star\)'s risk-set contribution changes
the weighted risk-set treatment average.

Now differentiate the perturbed score
\[
U_{\rm MEC,\theta}^{(\varepsilon;i^\star)}(\theta,\lambda)
=
\int d\mathcal{N}_{\omega,A}^{(\varepsilon;i^\star)}(t)
-
\int
\overline A^{(\varepsilon;i^\star)}(t)
\,d\mathcal{N}_{\omega}^{(\varepsilon;i^\star)}(t)
\]
with respect to \(\varepsilon\) at zero. Using
\eqref{eq:dN_omega,A}, \eqref{eq:dN_omega}, and the first-order expansion
\[
\overline A^{(\varepsilon;i^\star)}(t)
=
\overline A(t)
+
\varepsilon \dot{\overline A}_{i^\star}(t)
+
o(\varepsilon),
\]
we obtain
\begin{align*}
U_{\rm MEC,\theta}^{(\varepsilon;i^\star)}(\theta,\lambda)
&=
\int
\underbrace{\left\{d\mathcal{N}_{\omega,A}(t)
+
\varepsilon \omega_{i^\star} A_{i^\star}
\,d\mathcal{N}_{i^\star}(t)
\right\}}_{d\mathcal{N}_{\omega,A}^{(\varepsilon;i^\star)}(t)}
\\
&\quad -
\int
\underbrace{\left\{
\overline A(t)
+
\varepsilon \dot{\overline A}_{i^\star}(t)
+
o(\varepsilon)
\right\}}_{\overline A^{(\varepsilon;i^\star)}(t)}\
\,\,
\underbrace{\left\{
d\mathcal{N}_{\omega}(t)
+
\varepsilon \omega_{i^\star}
\,d\mathcal{N}_{i^\star}(t)
\right\}}_{d\mathcal{N}_{\omega}^{(\varepsilon;i^\star)}(t)}
.
\end{align*}

Keeping only the first-order terms in \(\varepsilon\), this becomes
\begin{align*}
U_{\rm MEC,\theta}^{(\varepsilon;i^\star)}(\theta,\lambda)
&=
\int d\mathcal{N}_{\omega,A}(t)
-
\int \overline A(t)\,d\mathcal{N}_{\omega}(t)
\\
&\quad
+
\varepsilon
\left[
\int \omega_{i^\star} A_{i^\star}\,
d\mathcal{N}_{i^\star}(t)
-
\int \overline A(t)\omega_{i^\star}\,
d\mathcal{N}_{i^\star}(t)
-
\int \dot{\overline A}_{i^\star}(t)\,
d\mathcal{N}_{\omega}(t)
\right]
+
o(\varepsilon).
\end{align*}
Here, second-order terms in \(\varepsilon\) are absorbed into
\(o(\varepsilon)\). In particular, the product
\(\varepsilon \dot{\overline A}_{i^\star}(t)
\cdot\varepsilon \omega_{i^\star}d\mathcal{N}_{i^\star}(t)\)
is of order \(O(\varepsilon^2)\) and is therefore omitted in the first-order
derivative calculation.

Therefore,
\begin{align}
\left.
\frac{\partial}{\partial\varepsilon}
U_{\rm MEC,\theta}^{(\varepsilon;i^\star)}(\theta,\lambda)
\right|_{\varepsilon=0}
&=
\int \omega_{i^\star} A_{i^\star}\,
d\mathcal{N}_{i^\star}(t)
-
\int \overline A(t)\omega_{i^\star}\,
d\mathcal{N}_{i^\star}(t)
-
\int \dot{\overline A}_{i^\star}(t)\,
d\mathcal{N}_{\omega}(t)
\nonumber\\
&=
\int
\omega_{i^\star}\{A_{i^\star}-\overline A(t)\}
\,d\mathcal{N}_{i^\star}(t)
-
\int
\dot{\overline A}_{i^\star}(t)
\,d\mathcal{N}_{\omega}(t).
\label{eq:supp_score_perturb_derivative}
\end{align}
The first term in \eqref{eq:supp_score_perturb_derivative} is the direct event
contribution of subject \(i^\star\):
\[
\int
\omega_{i^\star}\{A_{i^\star}-\overline A(t)\}
\,d\mathcal{N}_{i^\star}(t).
\]
The second term in \eqref{eq:supp_score_perturb_derivative} is the risk-set
compensation term. It appears because subject \(i^\star\) also changes the
weighted risk-set average \(\overline A(t)\) at every event time for which the
subject is still at risk. Substituting \eqref{eq:A_dot} into
\eqref{eq:supp_score_perturb_derivative} gives
\[
-
\int
\dot{\overline A}_{i^\star}(t)
\,d\mathcal{N}_{\omega}(t)
=
-
\int
\frac{
\omega_{i^\star}\mathcal{Y}_{i^\star}(t)
\exp(\theta A_{i^\star})
\{A_{i^\star}-\overline A(t)\}
}{
S^{(0)}(t)
}
\,d\mathcal{N}_{\omega}(t).
\]
Thus, the subject-level Lin--Wei/Binder empirical Cox contribution in the
direction of subject \(i^\star\) is
\[
\eta_{i^\star}(\theta,\lambda)
=
\int
\omega_{i^\star}\{A_{i^\star}-\overline A(t)\}
\,d\mathcal{N}_{i^\star}(t)
-
\int
\frac{
\omega_{i^\star}\mathcal{Y}_{i^\star}(t)\exp(\theta A_{i^\star})
\{A_{i^\star}-\overline A(t)\}
}{
S^{(0)}(t)
}
\,d\mathcal{N}_{\omega}(t).
\]

\paragraph{Conclusion.}
Since \(i^\star\) was arbitrary, we relabel \(i^\star\) as \(i\) in the final
expression. Restoring the dependence on \(\lambda\) from the simplified
notations \eqref{eq:notation_1}--\eqref{eq:notation_2}, with
\[
d\mathcal N_{\omega,\lambda}(t)
=
\sum_{j=1}^n
\widetilde\omega_j(\lambda)\,d\mathcal N_j(t)
\]
denoting the \(\lambda\)-dependent version of
\(d\mathcal N_{\omega}(t)\), we obtain
\begin{align}
\label{eq:supp_eta_mec}
\eta_i(\theta,\lambda)
&=
\int
\widetilde\omega_i(\lambda)
\{A_i-\overline A_{\rm MEC,\omega}(t;\theta,\lambda)\}
\,d\mathcal N_i(t)
\nonumber\\
&\quad -
\int
\frac{
\widetilde\omega_i(\lambda)\mathcal Y_i(t)\exp(\theta A_i)
\{A_i-\overline A_{\rm MEC,\omega}(t;\theta,\lambda)\}
}{
S_\lambda^{(0)}(t;\theta)
}
\,d\mathcal N_{\omega,\lambda}(t).
\end{align}

This is the fixed-weight Lin--Wei/Binder empirical Cox contribution evaluated
conditionally on the calibration parameter \(\lambda\). By the degree-one
homogeneity property of the score in \eqref{eq:MEC_cox_score_simple}, this
contribution satisfies
\[
\sum_{i=1}^n \eta_i(\theta,\lambda)
=
U_{{\rm MEC},\theta}(\theta,\lambda).
\]
Indeed, summing the first term in \eqref{eq:supp_eta_mec} over \(i\) gives
\(U_{{\rm MEC},\theta}(\theta,\lambda)\), while the second term sums to zero
because
\[
\sum_{i=1}^n
\frac{
\widetilde\omega_i(\lambda)\mathcal Y_i(t)\exp(\theta A_i)
\{A_i-\overline A_{\rm MEC,\omega}(t;\theta,\lambda)\}
}{
S_\lambda^{(0)}(t;\theta)
}
=
\frac{
S_\lambda^{(1)}(t;\theta)
-
\overline A_{\rm MEC,\omega}(t;\theta,\lambda)
S_\lambda^{(0)}(t;\theta)
}{
S_\lambda^{(0)}(t;\theta)
}
=
0.
\]
Thus, at the MEC-Cox solution and optimized dual parameter,
\[
\sum_{i=1}^n
\eta_i(\widehat\theta_{\rm MEC},\widehat\lambda)
=
U_{{\rm MEC},\theta}(\widehat\theta_{\rm MEC},\widehat\lambda)
=
0.
\]
Finally, in the MEC-Cox variance estimator, the empirical Cox contribution is
evaluated at the MEC-Cox solution:
\begin{align}
\widehat\eta_i
&=
\eta_i(\widehat\theta_{\rm MEC},\widehat\lambda)
\nonumber\\
&=
\int
\widetilde\omega_i(\widehat\lambda)
\{A_i-\overline A_{\rm MEC,\omega}(t;\widehat\theta_{\rm MEC},\widehat\lambda)\}
\,d\mathcal N_i(t)
\nonumber\\
\label{eq:MEC_Cox_LinWei_Binder_contribution_optimized}
&\quad -
\int
\frac{
\widetilde\omega_i(\widehat\lambda)\mathcal Y_i(t)
\exp(\widehat\theta_{\rm MEC} A_i)
\{A_i-\overline A_{\rm MEC,\omega}(t;\widehat\theta_{\rm MEC},\widehat\lambda)\}
}{
S_{\widehat\lambda}^{(0)}(t;\widehat\theta_{\rm MEC})
}
\,d\mathcal N_{\omega,\widehat\lambda}(t).
\end{align}
This represents the fixed-weight Lin--Wei/Binder empirical Cox contribution
evaluated at the MEC-Cox estimator \(\widehat\theta_{\rm MEC}\) and the estimated
calibration parameter \(\widehat\lambda\), and is the Cox block used in the
empirical meat of the stacked sandwich variance estimator.


\subsection{Oracle Prognostic Score Basis for MEC-Cox}
\label{subsec:Oracle prognostic score basis for MEC-Cox}

This section develops a theoretically ideal calibration basis for MEC-Cox. Let \(T^0\) denote the counterfactual event time under control. Following the
prognostic-score idea of \citet{hansen2008prognostic}, suppose that there exists
a possibly low-dimensional function \(\Psi(X)\) such that
\begin{align}
\label{eq:control_prognostic_score}
T^0 \perp X \mid \Psi(X).
\end{align}
Condition~\eqref{eq:control_prognostic_score} means that \(\Psi(X)\) is a
sufficient summary of the baseline covariates for the counterfactual control
event-time distribution. This observation suggests that the ideal MEC-Cox calibration basis should be
generated by the control-prognostic score itself. Specifically, define the
oracle prognostic score basis as
\[
h_{\rm oracle}(X)
=
(1,\Psi(X)).
\]

This basis contains the outcome-relevant information in \(X\) for the
counterfactual control survival outcome, in the sense that no component of \(X\)
outside \(\Psi(X)\) provides additional information about the distribution of
\(T^0\).



The following proposition makes precise the sense in which the control-prognostic
score provides an oracle calibration basis for MEC-Cox.

\begin{proposition}
\label{prop:ideal_prognostic_basis}
Fix the baseline external-control weights
\(\widehat d=\{\widehat d_i:i\in\mathcal I_0\}\), the Bregman generator \(G\),
the source and target index sets \((\mathcal I_0,\mathcal I_1)\), and the
weighted Cox estimating equation \(U_n^{\omega}(\theta)=0\). Let \(\mathcal H\)
be a class of admissible covariate-based calibration bases. For each
\(h\in\mathcal H\), define
\begin{align}
\label{eq:oracle_basis_calibration_problem}
\widehat w(h)
=
\arg\min_{\{w_i>0:i\in\mathcal I_0\}}
\sum_{i\in\mathcal I_0}D_G(w_i\Vert \widehat d_i)
\quad
\text{subject to}
\quad
\sum_{i\in\mathcal I_0}w_i h(X_i)
=
\sum_{i\in\mathcal I_1}h(X_i).
\end{align}
Set \(\widetilde\omega_i(h)=A_i+(1-A_i)\widehat w_i(h)\), and let
\(\widehat\theta_{\rm MEC}(h)\) denote a solution to
\(U_n^{\widetilde\omega(h)}(\theta)=0\). Define the class of MEC-Cox estimators
generated by \(\mathcal H\) as
\begin{align}
\label{eq:class_mec_cox_estimators}
\mathcal C_{\rm MEC}(G,\widehat d,\mathcal H)
:=
\bigl\{\widehat\theta_{\rm MEC}(h):\;
h\in\mathcal H,\ \widehat w(h)\text{ solves }
\eqref{eq:oracle_basis_calibration_problem}, \,\,
U_n^{\widetilde\omega(h)}\{\widehat\theta_{\rm MEC}(h)\}=0
\bigr\}.
\end{align}
Let \(\eta\) denote a generic limiting external-control version of the
Lin--Wei/Binder subject-level Cox contribution in
\eqref{eq:MEC_Cox_LinWei_Binder_contribution_optimized}, evaluated at the
relevant limiting target values, and define the weight-normalized contribution
\[
\eta^0=\frac{\eta}{q(X)}.
\]
For \(h\in\mathcal H\), define the external-control, weight-normalized residual
variation
\begin{align}
\label{eq:res_var}
\mathcal V_{\rm res}^{0}(h)
=
\operatorname{Var}\left[
\eta^0
-
\mathbb E\{\eta^0\mid h(X),A=0\}
\,\middle|\, A=0
\right].
\end{align}

\noindent
Assume the following additional conditions:
\begin{itemize}[leftmargin=1.5em, itemsep=0.35em]

    \item[] \textbf{\textit{P1. Regularity for the weight-normalized Cox contribution.}} 
    The assumptions required for consistency of the MEC-Cox estimator hold
    (i.e., the conditions of Theorem~\ref{thm:Conditions_for_consistency_MEC_Cox} hold). 

    \item[] \textbf{\textit{P2. Existence of an oracle prognostic score basis.}} 
    There exists a prognostic score \(\Psi(X)\) such that
    \[
    T^0\perp X\mid \Psi(X).
    \]
    That is, conditional on \(\Psi(X)\), the full baseline covariate vector
    \(X\) contains no additional information about the counterfactual control
    event time \(T^0\). Define
    \[
    h_{\rm oracle}(X)=(1,\Psi(X)) \in\mathcal H.
    \]    
\end{itemize}

\noindent
Then, for every
\(\widehat\theta_{\rm MEC}(h)\in
\mathcal C_{\rm MEC}(G,\widehat d,\mathcal H)\), we have
\[
\mathcal V_{\rm res}^{0}(h)
\ge
\mathcal V_{\rm res}^{0}(h_{\rm oracle}),
\]
where $\mathcal V_{\rm res}^{0}(h_{\rm oracle})
=
\operatorname{Var}\left[
\eta^0
-
\mathbb E\{\eta^0\mid h_{\rm oracle}(X),A=0\}
\,\middle|\, A=0
\right].$ Consequently, \(h_{\rm oracle}(X)=(1,\Psi(X))\), or any basis
generating the same information, is oracle-optimal within
$\mathcal C_{\rm MEC}$ $ (G,\widehat d,\mathcal H)$ in the sense of minimizing
the external-control, weight-normalized residual variation
\eqref{eq:res_var}.
\end{proposition}

\begin{proof}
\noindent\textbf{\textit{Recall the empirical subject-level Lin--Wei/Binder Cox contribution.}} 
Fix an arbitrary estimator
\(\widehat\theta_{\rm MEC}(h)\in
\mathcal C_{\rm MEC}(G,\widehat d,\mathcal H)\); equivalently, fix an
arbitrary admissible calibration basis \(h\in\mathcal H\). Recall that the
empirical subject-level Lin--Wei/Binder Cox contribution evaluated at the MEC-Cox solution \eqref{eq:MEC_Cox_LinWei_Binder_contribution_optimized} is
\begin{align}
\nonumber
\widehat\eta_i
&=
\eta_i(\widehat\theta_{\rm MEC},\widehat\lambda)
=
\int
\widetilde\omega_i(\widehat\lambda)
\{A_i-\overline A_{\rm MEC,\omega}(t;\widehat\theta_{\rm MEC},\widehat\lambda)\}
\,d\mathcal N_i(t)
\nonumber\\
\label{eq:Cox_contribution}
&\quad -
\int
\frac{
\widetilde\omega_i(\widehat\lambda)\mathcal Y_i(t)
\exp(\widehat\theta_{\rm MEC} A_i)
\{A_i-\overline A_{\rm MEC,\omega}(t;\widehat\theta_{\rm MEC},\widehat\lambda)\}
}{
S_{\widehat\lambda}^{(0)}(t;\widehat\theta_{\rm MEC})
}
\,d\mathcal N_{\omega,\widehat\lambda}(t).
\end{align}
The contribution \(\widehat\eta_i\) in \eqref{eq:Cox_contribution} is defined
for subjects from both cohorts: treated trial subjects with \(A_i=1\) and
external-control subjects with \(A_i=0\). In the oracle argument below, we focus
on its external-control component, obtained by setting \(A_i=0\).

\paragraph{Limiting external-control Cox contribution \(\eta\) and its weight-normalized version \(\eta^0\).}
For an external-control subject, \(A_i=0\), so
\(\exp(\widehat\theta_{\rm MEC}A_i)=1\),
\(\mathcal N_i(t)=\mathcal N_i^0(t)\), and
\(\mathcal Y_i(t)=\mathcal Y_i^0(t)\) by consistency. Hence
\begin{align}
\widehat\eta_i
&=
-\int
\widetilde\omega_i(\widehat\lambda)
\overline A_{\rm MEC,\omega}(t;\widehat\theta_{\rm MEC},\widehat\lambda)
\,d\mathcal N_i^0(t)
\nonumber+
\int
\frac{
\widetilde\omega_i(\widehat\lambda)\mathcal Y_i^0(t)
\overline A_{\rm MEC,\omega}(t;\widehat\theta_{\rm MEC},\widehat\lambda)
}{
S_{\widehat\lambda}^{(0)}(t;\widehat\theta_{\rm MEC})
}
\,d\mathcal N_{\omega,\widehat\lambda}(t).
\label{eq:Cox_contribution_external_control}
\end{align}
At the population target values, \(\widehat\theta_{\rm MEC}\to\theta_{ATT}=:\theta_0\)
and \(\widehat\lambda\to0=:\lambda_0\). Moreover, the calibrated
external-control weight converges to the normalized ATT odds weight, that is,
up to a deterministic normalizing constant \(c_q\), $\widetilde\omega_i(\widehat\lambda)
\approx
c_q q(X_i)$ with $q(X_i)=\pi(X_i)/(1-\pi(X_i)).$ Thus, if \(\eta\) denotes the corresponding limiting external-control Cox
contribution, then
\[
\eta
=
c_q q(X)
\left[
-\int \overline A_0(t)\,d\mathcal N^0(t)
+
\int
\frac{
\mathcal Y^0(t)\overline A_0(t)
}{
s_0^{(0)}(t;\theta_0)
}
\,d\mathcal N_{\omega,0}(t)
\right],
\]
where \(\overline A_0(t)\), \(s_0^{(0)}(t;\theta_0)\), and
\(d\mathcal N_{\omega,0}(t)\) denote the corresponding population limits of
$\overline A_{\rm MEC,\omega}$ $(t;\widehat\theta_{\rm MEC},$ $ \widehat\lambda)$,
\(S_{\widehat\lambda}^{(0)}(t;\widehat\theta_{\rm MEC})\), and
\(d\mathcal N_{\omega,\widehat\lambda}(t)\), respectively.
Consequently, after factoring out the ATT odds weight,
\[
\eta^0
:=
\frac{\eta}{q(X)}
=
c_q
\left[
-\int \overline A_0(t)\,d\mathcal N^0(t)
+
\int
\frac{
\mathcal Y^0(t)\overline A_0(t)
}{
s_0^{(0)}(t;\theta_0)
}
\,d\mathcal N_{\omega,0}(t)
\right].
\]
Therefore, \(\eta^0\) is a linear functional of the counterfactual control
counting and at-risk processes, with coefficient functions determined by
population Cox risk-set quantities rather than by the individual covariate
vector \(X\). Here, when writing \(\eta^0=\eta/q(X)\), we implicitly restrict attention to the region where \(q(X)>0\), equivalently \(\pi(X)>0\), so that \(1/q(X)\) is
well defined. This normalization is used only to factor out the ATT odds
weight from the external-control Cox contribution.

\paragraph{Global prognostic sufficiency and one-sided transportability imply
external-control prognostic sufficiency.}
We first derive the external-control version of the prognostic sufficiency
condition. Let \(S=\Psi(X)\) for notational simplicity. By P2,
\[
T^0 \perp X \mid S.
\]
In addition, P1 includes the causal assumptions of Theorem~\ref{thm:identification_weighted_cox}, and hence the
one-sided survival transportability condition
\begin{align}
\label{eq:one_sided_trans_assumption}
T^0 \perp A \mid X .
\end{align}
We show that these two conditions imply $T^0 \perp X \mid S,A=0,$ or equivalently,
\begin{align}
\label{eq:external_control_prognostic_sufficiency}
T^0 \perp X \mid \Psi(X),A=0.    
\end{align}

Let \(B\) be an arbitrary Borel set in the support of \(T^0\). Fix \(x\) and
write \(s=\Psi(x)\). By one-sided survival transportability \eqref{eq:one_sided_trans_assumption},
\begin{align}
\label{eq:eq_1_1}
P(T^0\in B\mid X=x,A=0)
=
P(T^0\in B\mid X=x,A=1).
\end{align}
Thus, the following equality holds
\[
\begin{aligned}
P(T^0\in B\mid X=x)
&=
\sum_{a=0}^1
P(T^0\in B\mid X=x,A=a)P(A=a\mid X=x)  \\
&=
P(T^0\in B\mid X=x,A=0)
\sum_{a=0}^1 P(A=a\mid X=x) \\
&=
P(T^0\in B\mid X=x,A=0),
\end{aligned}
\]
where the second equality uses \eqref{eq:eq_1_1}. Hence,
\begin{align}
\label{eq:eq_1__2}
P(T^0\in B\mid X=x,A=0)
=
P(T^0\in B\mid X=x).
\end{align}

By the global prognostic-score condition \(T^0\perp X\mid S\), and because
\(S=\Psi(X)\), for any \(x\) satisfying \(s=\Psi(x)\), we have
\[
P(T^0\in B\mid X=x,S=s)
=
P(T^0\in B\mid S=s).
\]
Moreover, since \(S\) is a deterministic function of \(X\), conditioning on
\(X=x\) already determines \(S=\Psi(x)=s\). Hence,
\[
P(T^0\in B\mid X=x,S=s)
=
P(T^0\in B\mid X=x).
\]
Therefore,
\begin{align}
\label{eq:eq_1__3}
P(T^0\in B\mid X=x)
=
P(T^0\in B\mid S=s).
\end{align}

From \eqref{eq:eq_1__2} and \eqref{eq:eq_1__3}, we have 
\begin{align}
    \label{eq:eq_1__4}
    P(T^0\in B\mid X=x,A=0)
=
P(T^0\in B\mid S=s).
\end{align}

It remains to relate the right-hand side of \eqref{eq:eq_1__4} to the
conditional law given \((S,A=0)\). Averaging over the conditional distribution
of \(X\) given \(S=s\) and \(A=0\), we obtain
\[
\begin{aligned}
P(T^0\in B\mid S=s,A=0)
&=
\int
P(T^0\in B\mid X=u,S=s,A=0)
\,dP(u\mid S=s,A=0) \\
&=
\int
P(T^0\in B\mid X=u,A=0)
\,dP(u\mid S=s,A=0) \\
&=
\int
P(T^0\in B\mid S=s)
\,dP(u\mid S=s,A=0) \\
&=
P(T^0\in B\mid S=s)
\,
\int dP(u\mid S=s,A=0) \\
&=
P(T^0\in B\mid S=s).
\end{aligned}
\]
The second equality holds because \(S=\Psi(X)\) is a deterministic function of \(X\), so
conditioning on \(X=u\) already determines \(S=\Psi(u)\); moreover, the
conditional distribution \(dP(u\mid S=s,A=0)\) is supported on values of \(u\)
satisfying \(\Psi(u)=s\). The third equality uses \eqref{eq:eq_1__4}.

Combining the preceding identities gives
\[
P(T^0\in B\mid X=x,A=0)
=
P(T^0\in B\mid S=s,A=0).
\]
Since \(B\) was arbitrary, this proves $T^0 \perp X \mid S,A=0,$ or equivalently, $T^0 \perp X \mid \Psi(X),A=0$ \eqref{eq:external_control_prognostic_sufficiency}.

\paragraph{Core identity implied by P1--P2.}
We first show that, under P1--P2,
\begin{align}
\label{eq:weight_normalized_sufficiency}
\mathbb E(\eta^0\mid X,A=0)
=
\mathbb E\{\eta^0\mid h_{\rm oracle}(X),A=0\}.
\end{align}
From the preceding representation of \(\eta^0\), we have
\begin{align}
\mathbb E(\eta^0\mid X,A=0)
&=
c_q
\left[
-\int \overline A_0(t)\,
\mathbb E\{d\mathcal N^0(t)\mid X,A=0\}
\right.
\nonumber\\
&\qquad\qquad\left.
+
\int
\frac{\overline A_0(t)}
{s_0^{(0)}(t;\theta_0)}
\mathbb E\{\mathcal Y^0(t)\mid X,A=0\}
\,d\mathcal N_{\omega,0}(t)
\right].
\label{eq:cond_mean_eta0_X}
\end{align}
Here, \(\overline A_0(t)\), \(s_0^{(0)}(t;\theta_0)\), and
\(d\mathcal N_{\omega,0}(t)\) are population-level quantities and therefore do
not introduce additional individual-level dependence on \(X\).

By source-specific independent censoring in P1, we have
\[
C\perp (T^0,X)\mid A=0.
\]
Let
\[
G_0(t)=P(C\ge t\mid X, A=0)=P(C\ge t\mid A=0).
\]
Then the conditional mean of the counterfactual control at-risk process satisfies
\begin{align}
\mathbb E\{\mathcal Y^0(t)\mid X,A=0\}
&=
P(T^0\ge t,\ C\ge t\mid X,A=0)
\nonumber\\
&=
P(C\ge t\mid A=0)P(T^0\ge t\mid X,A=0)
\nonumber\\
&=
G_0(t)P(T^0\ge t\mid X,A=0).
\label{eq:cond_mean_Y0}
\end{align}
Similarly, suppressing left-limit notation for simplicity, the conditional mean
of the counterfactual control event increment satisfies
\begin{align}
\mathbb E\{d\mathcal N^0(t)\mid X,A=0\}
&=
P(T^0\in dt,\ C\ge t\mid X,A=0)
\nonumber\\
&=
P(C\ge t\mid A=0)P(T^0\in dt\mid X,A=0)
\nonumber\\
&=
G_0(t)P(T^0\in dt\mid X,A=0).
\label{eq:cond_mean_dN0}
\end{align}

By the external-control prognostic sufficiency condition just established \eqref{eq:external_control_prognostic_sufficiency},
\[
T^0\perp X\mid \Psi(X),A=0.
\]
Therefore,
\[
P(T^0\ge t\mid X,A=0)
=
P\{T^0\ge t\mid \Psi(X),A=0\},
\]
and
\[
P(T^0\in dt\mid X,A=0)
=
P\{T^0\in dt\mid \Psi(X),A=0\}.
\]
Because \(h_{\rm oracle}(X)=(1,\Psi(X))\) generates the same
information as \(\Psi(X)\), equations
\eqref{eq:cond_mean_Y0}--\eqref{eq:cond_mean_dN0} imply
\[
\mathbb E\{\mathcal Y^0(t)\mid X,A=0\}
=
\mathbb E\{\mathcal Y^0(t)\mid h_{\rm oracle}(X),A=0\},
\]
and
\[
\mathbb E\{d\mathcal N^0(t)\mid X,A=0\}
=
\mathbb E\{d\mathcal N^0(t)\mid h_{\rm oracle}(X),A=0\}.
\]
Substituting these two identities into
\eqref{eq:cond_mean_eta0_X} gives
\[
\mathbb E(\eta^0\mid X,A=0)
=
\mathbb E\{\eta^0\mid h_{\rm oracle}(X),A=0\},
\]
which proves \eqref{eq:weight_normalized_sufficiency}.

\paragraph{Variance decomposition for the projection argument.}
Now, since \(h(X)\) is a measurable function of \(X\), the tower property gives
\[
\mathbb E\{\eta^0\mid h(X),A=0\}
=
\mathbb E\left[
\mathbb E(\eta^0\mid X,A=0)
\mid h(X),A=0
\right].
\]
Using \eqref{eq:weight_normalized_sufficiency},
\[
\mathbb E\{\eta^0\mid h(X),A=0\}
=
\mathbb E\left[
\mathbb E\{\eta^0\mid h_{\rm oracle}(X),A=0\}
\mid h(X),A=0
\right].
\]
Let
\[
M_{\rm oracle}
=
\mathbb E\{\eta^0\mid h_{\rm oracle}(X),A=0\}.
\]
Then
\[
\mathbb E\{\eta^0\mid h(X),A=0\}
=
\mathbb E(M_{\rm oracle}\mid h(X),A=0).
\]
By the law of total variance conditional on \(A=0\),
\begin{align*}
&\operatorname{Var}(M_{\rm oracle}\mid A=0)\\
&\qquad =
\operatorname{Var}\{\mathbb E(M_{\rm oracle}\mid h(X),A=0)\mid A=0\}
+
\mathbb E\{\operatorname{Var}(M_{\rm oracle}\mid h(X),A=0)\mid A=0\}.
\end{align*}
Since the second term on the right-hand side is nonnegative,
\[
\operatorname{Var}\{\mathbb E(M_{\rm oracle}\mid h(X),A=0)\mid A=0\}
\le
\operatorname{Var}(M_{\rm oracle}\mid A=0).
\]
Recovering $M_{\rm oracle}$ gives
\begin{align}
\label{eq:ineq_1}
\operatorname{Var}\left[
\mathbb E\{\eta^0\mid h(X),A=0\}
\mid A=0
\right]
\le
\operatorname{Var}\left[
\mathbb E\{\eta^0\mid h_{\rm oracle}(X),A=0\}
\mid A=0
\right].
\end{align}

On the other hand, for any calibration basis \(h\), write
\[
\eta^0
=
\mathbb E\{\eta^0\mid h(X),A=0\}
+
\left[
\eta^0-\mathbb E\{\eta^0\mid h(X),A=0\}
\right].
\]
Taking the conditional variance given \(A=0\), we obtain
\begin{align*}
\operatorname{Var}(\eta^0\mid A=0)
&=
\operatorname{Var}\left[
\mathbb E\{\eta^0\mid h(X),A=0\}
\,\middle|\, A=0
\right] \\
&\quad+
\operatorname{Var}\left[
\eta^0-\mathbb E\{\eta^0\mid h(X),A=0\}
\,\middle|\, A=0
\right] \\
&\quad+
2\,\operatorname{Cov}\left(
\mathbb E\{\eta^0\mid h(X),A=0\},
\eta^0-\mathbb E\{\eta^0\mid h(X),A=0\}
\,\middle|\, A=0
\right).
\end{align*}
The covariance term is zero. Indeed,
\begin{align*}
&\operatorname{Cov}\left(
\mathbb E\{\eta^0\mid h(X),A=0\},
\eta^0-\mathbb E\{\eta^0\mid h(X),A=0\}
\,\middle|\, A=0
\right) \\
&\quad=
\mathbb E\left[
\mathbb E\{\eta^0\mid h(X),A=0\}
\left\{
\eta^0-\mathbb E\{\eta^0\mid h(X),A=0\}
\right\}
\,\middle|\, A=0
\right] \\
&\quad=
\mathbb E\left[
\mathbb E\left[
\mathbb E\{\eta^0\mid h(X),A=0\}
\left\{
\eta^0-\mathbb E\{\eta^0\mid h(X),A=0\}
\right\}
\,\middle|\, h(X),A=0
\right]
\,\middle|\, A=0
\right] \\
&\quad=
\mathbb E\left[
\mathbb E\{\eta^0\mid h(X),A=0\}
\,
\mathbb E\left\{
\eta^0-\mathbb E\{\eta^0\mid h(X),A=0\}
\mid h(X),A=0
\right\}
\,\middle|\, A=0
\right] \\
&\quad=
\mathbb E\left[
\mathbb E\{\eta^0\mid h(X),A=0\}
\left\{
\mathbb E(\eta^0\mid h(X),A=0)
-
\mathbb E(\eta^0\mid h(X),A=0)
\right\}
\,\middle|\, A=0
\right] \\
&\quad=0.
\end{align*}
Therefore,
\[
\operatorname{Var}(\eta^0\mid A=0)
=
\operatorname{Var}\left[
\mathbb E\{\eta^0\mid h(X),A=0\}
\mid A=0
\right]
+
\operatorname{Var}\left[
\eta^0-\mathbb E\{\eta^0\mid h(X),A=0\}
\mid A=0
\right].
\]

\paragraph{Conclusion.} By the definition of \(\mathcal V_{\rm res}^{0}(h)\), this becomes
\begin{align}
\label{eq:variance_decomposition}
\operatorname{Var}(\eta^0\mid A=0)
=
\operatorname{Var}\left[
\mathbb E\{\eta^0\mid h(X),A=0\}
\mid A=0
\right]
+
\mathcal V_{\rm res}^{0}(h).
\end{align}
Because \eqref{eq:variance_decomposition} holds for any $h \in \mathcal{H}$, we have
\[
\operatorname{Var}(\eta^0\mid A=0)
=
\operatorname{Var}\left[
\mathbb E\{\eta^0\mid h_{\rm oracle}(X),A=0\}
\mid A=0
\right]
+
\mathcal V_{\rm res}^{0}(h_{\rm oracle}).
\]
Combining these two decompositions with \eqref{eq:ineq_1} yields
\[
\mathcal V_{\rm res}^{0}(h)
\ge
\mathcal V_{\rm res}^{0}(h_{\rm oracle}).
\]
Because \(h\in\mathcal H\) was arbitrary, the inequality holds for every
admissible calibration basis \(h\), and therefore for every MEC-Cox estimator
generated by such a basis.
\end{proof}

Proposition~\ref{prop:ideal_prognostic_basis} implies that, within the
MEC-Cox class \(\mathcal C_{\rm MEC}(G,\widehat d,\mathcal H)\) defined in
\eqref{eq:class_mec_cox_estimators}, the oracle control-prognostic score basis
\(h_{\rm oracle}(X)=(1,\Psi(X))\) minimizes the residual variation
of the weight-normalized external-control Cox contribution. Intuitively, the
ATT odds weight \(q(X)\) first transports the external-control cohort to the
treated trial target population in baseline covariate distribution. After this
transport component is factored out, the remaining variation in the
external-control Cox contribution is driven by the counterfactual control
survival process. Under the prognostic-score condition
\(T^0\perp X\mid \Psi(X)\), the oracle basis captures the largest possible
baseline-covariate-explained component of this remaining variation. Thus,
Proposition~\ref{prop:ideal_prognostic_basis} provides an oracle interpretation
of the MEC-Cox calibration basis: among admissible bases, the control-prognostic
score basis removes the maximal outcome-relevant component of the
weight-normalized external-control Cox contribution.

Although \citet{hansen2008prognostic} introduced the prognostic score primarily
as a conditioning variable for matching, subclassification, or adjustment, our
use of the prognostic score is different. In MEC-Cox, the control-prognostic
score is used as a calibration basis for the external-control weights. It is
therefore not used to define strata or to enter the Cox model directly; instead,
it guides the calibration step so that the weighted external-control cohort is
balanced with the treated trial cohort in counterfactual control prognosis.
Proposition~\ref{prop:ideal_prognostic_basis} formalizes this intuition by
showing that, within the MEC-Cox class, the oracle control-prognostic score
basis minimizes the residual variation of the weight-normalized external-control
Cox contribution. This suggests that prognostic-score calibration may improve
efficiency by removing the largest possible baseline-covariate-explained
component of the external-control Cox contribution.



\section{Simulation Setup}\label{sec:simulation_setup}
\paragraph{Objective.}
We describe the simulation setup used in the main paper and in the additional
simulation experiments reported in Section \ref{sec:Additional simulation experiments} in the Appendix. The objective of
the simulation studies is to evaluate the finite-sample performance of the
proposed MEC-Cox estimator for estimating the ATT marginal log-hazard ratio in
hypothetical externally controlled single-arm trials. We compare MEC-Cox with
existing ATT-IPW Cox estimators, including the naive model-based variance
estimator, the Lin--Wei/Binder robust sandwich variance estimator
\citep{lin1989robust,binder1992fitting}, the Shu corrected sandwich variance
estimator \citep{shu2021variance}, and a fixed-weight Lin--Wei/Binder estimator
that uses the same propensity-score model as MEC-Cox.

The simulation design mimics a setting in which treated trial patients are
compared with an external-control cohort after transport weighting. We evaluate
three performance measures based on \(R=1000\) Monte Carlo replications:
empirical coverage of the nominal \(95\%\) confidence interval, Monte Carlo
bias, and root mean squared error (RMSE). Bias and RMSE are evaluated on the
log-hazard-ratio scale.

\subsection{Simulation procedure}\label{subsec:Simulation procedure}
The simulation procedure, including data generation, model fitting, and
performance-metric reporting, is as follows.

\begin{enumerate}[leftmargin=1.5em, label={}, itemsep=0.5em]

\item[] \textbf{Step 1. Specify the simulation scenario.}
For each scenario, we fix the treated sample size \(n_1\), the
external-control sample size \(n_0\), the covariate dimension \(M\), the
source-selection model, and the outcome model.

The simulation framework accommodates both linear and nonlinear
source-selection and outcome models. We let \(\kappa_\pi\ge 0\) denote the
degree of nonlinearity in the source propensity-score model and
\(\kappa_m\ge 0\) denote the degree of nonlinearity in the outcome model. The
linear setting corresponds to \(\kappa_\pi=\kappa_m=0\), whereas positive
values of \(\kappa_\pi\) or \(\kappa_m\) generate increasingly nonlinear
data-generating mechanisms.

\item[] \textbf{Step 2. Generate baseline covariates and source membership.}
Let
\[
X_i=(X_{i1},\ldots,X_{iM}) \in \mathbb{R}^M
\]
denote the \(M\)-dimensional baseline covariate vector. Candidate covariate
vectors are generated independently from a multivariate standard normal
distribution. Source membership is then generated according to the source
propensity score
\[
\pi(X_i)=\Pr(A_i=1\mid X_i).
\]
The general source-selection model is
\begin{align}
\label{eq:true_log_ps_model}
\operatorname{logit}\{\pi(X_i)\}
=
-0.2+\ell_\pi(X_i)+\kappa_\pi r_\pi(X_i),    
\end{align}
where the linear component is
\[
\ell_\pi(X_i)
=
0.75X_{i1}
+0.75X_{i2}
+0.65X_{i3}
+0.65X_{i4}
+0.55X_{i5},
\]
and the nonlinear component is
\[
\begin{aligned}
r_\pi(X_i)
={}&
0.70\sin(1.25X_{i1})
+0.45(X_{i2}^2-1)
-0.55\{I(X_{i3}>0)-0.5\}  \\
&\quad
+0.35X_{i4}X_{i5}
+0.25\{\cos(X_{i1}+X_{i2})-\exp(-1)\}.
\end{aligned}
\]
Thus, the linear source-selection model is obtained by setting
\(\kappa_\pi=0\), whereas positive values of \(\kappa_\pi\) define nonlinear settings. The true source
propensity scores are truncated to lie in \([0.02,0.98]\) to avoid extreme
assignment probabilities.

The desired finite-sample ratio \(n_1:n_0\) is imposed by the sampling design,
rather than by changing the intercept of the source propensity-score model.
Specifically, we keep the intercept in the source-selection model fixed, as in
\eqref{eq:true_log_ps_model}, repeatedly generate candidate subjects from the
super-population, assign source membership using \(\pi(X_i)\), and retain
subjects until exactly \(n_1\) treated trial patients and \(n_0\)
external-control patients are obtained. This procedure samples covariates from
the induced conditional distributions \(X_i\mid A_i=1\) and \(X_i\mid A_i=0\),
while fixing the realized cohort sizes in each Monte Carlo replicate.

\item[] \textbf{Step 3. Generate event times from a Weibull proportional hazards model.}
For \(a=0,1\), where \(a=0\) denotes control and \(a=1\) denotes treatment,
event times are generated from the conditional hazard model
\begin{align}
\label{eq:weibull_ph_model}
\lambda^a(t\mid X_i)
=
\eta\lambda_0 t^{\eta-1}
\exp\{m_0(X_i)+a\beta\},
\end{align}
where \(\lambda_0\) is the Weibull baseline scale parameter and \(\eta\) is the
Weibull shape parameter. Across all simulation scenarios, we set
\(\lambda_0=0.00008\) and \(\eta=2\).

The proportional hazards assumption holds because treatment enters the
log-hazard additively through the time-invariant term \(a\beta\). Therefore,
for any fixed covariate value \(X_i\), the conditional hazard ratio comparing
treatment with control is
\[
\frac{\lambda^1(t\mid X_i)}{\lambda^0(t\mid X_i)}
=
\frac{
\eta\lambda_0 t^{\eta-1}\exp\{m_0(X_i)+\beta\}
}{
\eta\lambda_0 t^{\eta-1}\exp\{m_0(X_i)\}
}
=
\exp(\beta),
\]
which does not depend on time \(t\). Thus, \(\beta\) represents the conditional
log-hazard ratio in the data-generating model.

The log-hazard prognostic function is
\begin{align}
\label{eq:log_hazard_prognostic_function}
m_0(X_i)=\ell_m(X_i)+\kappa_m r_m(X_i),    
\end{align}
where the linear component is
\[
\ell_m(X_i)=X_i^\top b,
\]
with
\begin{align*}
b
&=
\bigl(
\log(1.75),\log(1.75),
\log(1.60),\log(1.60),
\log(1.50), \\
&\quad
\log(1.25),\log(1.25),\log(1.25),\log(1.25),\log(1.25),
0,\ldots,0
\bigr)^\top\in\mathbb R^M,
\end{align*}
and the nonlinear component is
\[
\begin{aligned}
r_m(X_i)
={}&
0.45\sin(X_{i2})
+0.35(X_{i3}^2-1)
+0.30\{I(X_{i4}>0)-0.5\}  \\
&\quad
+0.25X_{i1}X_{i5}
+0.20\{\cos(X_{i2}+X_{i5})-\exp(-1)\}.
\end{aligned}
\]
The linear proportional hazards outcome model is obtained by setting
\(\kappa_m=0\), whereas positive values of \(\kappa_m\) introduce nonlinear prognostic effects.

Under this construction, \(X_{i1},\ldots,X_{i5}\) enter both the
source-selection and outcome models, \(X_{i6},\ldots,X_{i10}\) enter only the
outcome model, and \(X_{i11},\ldots,X_{iM}\), when present, are noise variables.

The conditional log-hazard ratio is set to
\begin{align}
\label{eq:conditional_HR}
\beta=\log(0.70).
\end{align}
Consequently, under the Weibull proportional hazards model
\eqref{eq:weibull_ph_model}, potential event times are generated by inverse
transformation as
\[
T_i^a
=
\left[
\frac{-\log U_i^a}
{\lambda_0\exp\{m_0(X_i)+a\beta\}}
\right]^{1/\eta},
\qquad a=0,1,
\]
where \(U_i^a\sim\operatorname{Uniform}(0,1)\). External-control patients are
observed under the control event time \(T_i^0\), whereas treated trial patients
are observed under the treated event time \(T_i^1\).

\item[] \textbf{Step 4. Generate censoring and observed survival data.}
Independent censoring times are generated as
\[
C_i\sim\operatorname{Exp}(0.0008),
\]
where \(\operatorname{Exp}(\lambda_C)\) denotes the exponential distribution with
rate parameter \(\lambda_C\), so that \(\mathbb E(C_i)=1/\lambda_C=1250\). Because \(C_i\) is
generated independently of the potential event times and baseline covariates
within each source group, this censoring mechanism is consistent with the
source-specific independent censoring condition
\[
C_i \perp (T_i^0,T_i^1,X_i)\mid A_i .
\]

The observed time and event indicator are
\[
Y_i=\min(T_i,C_i),
\qquad
\delta_i=I(T_i\le C_i).
\]
The observed data for subject \(i\) are therefore
\[
O_i=(X_i,A_i,Y_i,\delta_i).
\]

\item[] \textbf{Step 5. Compute the scenario-specific true marginal ATT log-hazard ratio.}
The target parameter for simulation evaluation is the marginal ATT
log-hazard ratio \(\theta_{ATT}\), not the conditional log-hazard
ratio \(\beta\) in \eqref{eq:conditional_HR}. This distinction is important
because the conditional proportional hazards model specifies a
covariate-specific hazard ratio, whereas the ATT-weighted Cox estimator targets
a marginal Cox projection in the treated trial population. Following the
strategy of \citet{austin2016variance}, who computed the true marginal hazard
ratio numerically using a very large simulated population, we compute
\(\theta_{ATT}\) before running the finite-sample Monte Carlo
experiments.

For each data-generating scenario, we generate a large super-population from
the same data-generating mechanism. In implementation, we use
\(n_1^{\mathrm{super}}=30000\) and \(n_0^{\mathrm{super}}=60000\). Let \(n^{\mathrm{super}}=n_1^{\mathrm{super}}+n_0^{\mathrm{super}}\). These
super-population sizes are used only to approximate the population Cox
projection target; they are not intended to determine the finite-sample
simulation ratio \(n_1:n_0\). The finite-sample ratio controls the amount of
external-control information available in each Monte Carlo replicate, whereas
the marginal ATT target is determined by the treated trial target population
and the transported external-control outcome distribution.

For each subject in the super-population, we evaluate the true source
propensity score \(\pi(X_i)\) and form the true ATT odds weight
\[
q(X_i)=\frac{\pi(X_i)}{1-\pi(X_i)}.
\]
The external-control weights are normalized to have total mass equal to the
number of treated trial patients:
\[
d_i^{\mathrm{true}}
=
\frac{n_1^{\mathrm{super}}q(X_i)}
{\sum_{j\in\mathcal I_0^{\mathrm{super}}}q(X_j)},
\qquad i\in\mathcal I_0^{\mathrm{super}}.
\]
The corresponding true ATT-weighted Cox estimating-equation weight is
\[
\omega_i^{\mathrm{true}}
=
A_i+(1-A_i)d_i^{\mathrm{true}}.
\]
We then define \(\theta_{ATT}\) as the solution to the weighted Cox
estimating equation in this super-population:
\[
\sum_{i=1}^{n^{\mathrm{super}}}
\int
\omega_i^{\mathrm{true}}
\{A_i-\bar A_{\omega}^{\mathrm{true}}(t;\theta)\}
\,d\mathcal N_i(t)
=
0,
\]
where
\[
\bar A_{\omega}^{\mathrm{true}}(t;\theta)
=
\frac{
\sum_{i=1}^{n^{\mathrm{super}}}
\omega_i^{\mathrm{true}}\mathcal Y_i(t)\exp(\theta A_i)A_i
}{
\sum_{i=1}^{n^{\mathrm{super}}}
\omega_i^{\mathrm{true}}\mathcal Y_i(t)\exp(\theta A_i)
}.
\]
Equivalently, \(\theta_{ATT}\) is obtained as the coefficient of
\(A_i\) from a weighted Cox regression in the super-population using the true
ATT weights. Bias and RMSE are computed relative to this scenario-specific
value of \(\theta_{ATT}\), and empirical coverage is evaluated for
confidence intervals on the log-hazard-ratio scale.

\item[] \textbf{Step 6. Fit the competing weighted Cox estimators.}
For each Monte Carlo dataset, we fit the comparison methods described in
Subsection~\ref{subsec:Comparison_methods}. Each method estimates the marginal
ATT log-hazard ratio by fitting a weighted Cox regression with the
treatment/source indicator \(A_i\) as the only regression covariate.

\item[] \textbf{Step 7. Summarize Monte Carlo performance.}
For each simulation scenario and each estimator, we repeat the Monte Carlo
procedure \(R=1000\) times. Let \(\widehat\theta^{(r)}\) and
\(\widehat{\operatorname{se}}^{(r)}\) denote the point estimate and estimated
standard error obtained in replication \(r\), respectively. A Wald-type
confidence interval is constructed on the log-hazard-ratio scale as
\[
\left[
\widehat\theta^{(r)}
-
1.96\,\widehat{\operatorname{se}}^{(r)},
\quad
\widehat\theta^{(r)}
+
1.96\,\widehat{\operatorname{se}}^{(r)}
\right].
\]

We summarize performance using Monte Carlo bias, RMSE, and empirical coverage.
Bias is computed as
\[
\operatorname{Bias}
=
\frac{1}{R}
\sum_{r=1}^R
\{\widehat\theta^{(r)}-\theta_{ATT}\}.
\]
RMSE is computed as
\[
\operatorname{RMSE}
=
\left[
\frac{1}{R}
\sum_{r=1}^R
\{\widehat\theta^{(r)}-\theta_{ATT}\}^2
\right]^{1/2}.
\]
Empirical coverage is computed as
\[
\operatorname{Coverage}
=
\frac{1}{R}
\sum_{r=1}^R
I\left[
\theta_{ATT}
\in
\left\{
\widehat\theta^{(r)}
\pm
1.96\,\widehat{\operatorname{se}}^{(r)}
\right\}
\right].
\]
Desirable methods should achieve empirical coverage close to the nominal
\(95\%\) level while maintaining small bias and RMSE. A smaller RMSE alone is
not sufficient if the corresponding confidence intervals fail to achieve
adequate coverage.
\end{enumerate}

\subsection{Comparison methods}\label{subsec:Comparison_methods}

For each Monte Carlo dataset, we compare MEC-Cox with four ATT-weighted Cox
comparators. These methods differ in how the ATT weights are constructed and in
whether the variance estimator accounts for weight estimation and calibration.
The first three comparators use baseline ATT weights constructed from a
logistic-regression source propensity-score model. The fourth comparator uses
the same propensity-score learner as MEC-Cox but does not apply MEC
calibration.

\begin{enumerate}[leftmargin=1.5em, label={}, itemsep=0.5em]
    \item[] \textbf{1. Naive:} ATT-IPW Cox with naive model-based variance.
This method fits an ATT-weighted Cox regression using baseline ATT weights
constructed from a logistic-regression source propensity-score model. The
external-control ATT odds weights are normalized to have total mass equal to
the treated trial sample size. The variance is estimated using the usual
partial-likelihood model-based variance estimator for the weighted Cox model,
treating the estimated weights as fixed. This is analogous to the naive
model-based variance estimator considered by \citet{austin2016variance} for
IPTW Cox regression. Because this estimator ignores both the uncertainty in
estimating the propensity-score weights and the dependence induced by weighting,
it is generally biased \citep{shu2021variance}.

    \item[] \textbf{2. Robust sandwich:} ATT-IPW Cox with Lin--Wei/Binder
    robust variance. This method uses the same logistic-regression ATT-IPW Cox point estimator as
in the previous method, but estimates the variance using the Lin--Wei/Binder
robust sandwich variance estimator
\citep{lin1989robust,binder1992fitting}. This variance estimator accounts for
the weighted Cox estimating-equation structure induced by the IPW weights, but
it treats the estimated propensity-score weights as fixed and therefore does
not incorporate the additional structure from estimating the weights
\citep{shu2021variance}.

A limitation of this approach is that the resulting variance estimator can be
conservative. In particular, \citet{shu2021variance} showed that the standard
robust sandwich variance estimator tends to overestimate the variance in IPW
Cox models, leading to confidence intervals with coverage above the nominal
level and hence less efficient inference.

    \item[] \textbf{3. Corrected sandwich:} ATT-IPW Cox with Shu corrected
    sandwich variance. This method also uses ATT weights constructed from a logistic-regression
source propensity-score model, but estimates the variance using the corrected
sandwich estimator of \citet{shu2021variance}. This estimator stacks the
weighted Cox estimating equation with the propensity-score score equation,
thereby accounting for the additional uncertainty from estimating the
propensity-score weights. 

A limitation of this approach is that it is tied to a
correctly specified parametric propensity-score model; if the logistic
propensity-score model is misspecified, the ATT weights may be biased and the
corresponding variance correction may no longer provide reliable inference.
Moreover, this correction does not directly accommodate flexible
machine-learning propensity-score estimators without additional linearization or
resampling arguments.


    \item[] \textbf{4. MEC-Cox:} MEC-Cox with the proposed stacked sandwich
    variance. This method uses MEC-calibrated external-control weights and estimates the
variance using the proposed stacked sandwich variance estimator. The stacked
system combines the weighted Cox estimating equation with the MEC calibration
equation, thereby accounting for the effect of the calibration step on the Cox
estimator. The propensity-score and calibration-basis learners are constructed
by cross-fitting, as described in
Subsection~\ref{subsec:nuisance_parameter_estimation_basis_construction}.
\end{enumerate}

All weighted Cox regressions use the treatment/source indicator \(A_i\) as the
only regression covariate. Therefore, the fitted coefficient estimates the
marginal ATT log-hazard-ratio target.

\subsection{Nuisance estimation, calibration-basis construction, and Bregman-generator choice for MEC-Cox}
\label{subsec:nuisance_parameter_estimation_basis_construction}
 We now describe the nuisance-parameter estimation procedures used to construct
the baseline ATT transport weights and the MEC-Cox calibration basis, as well as
the choice of Bregman generator.

\subsubsection{Source propensity-score modeling} The propensity-score learner depends on the simulation setting. Following the
standard propensity-score framework of \citet{rosenbaum1983central}, we use
logistic-regression-based estimators in linear settings and flexible learners
in nonlinear settings:
\begin{enumerate}[leftmargin=1.5em, label={}, itemsep=0.5em]
    \item[] \textbf{1. Logistic-regression propensity-score learner.}
    In the low-to-moderate-dimensional linear setting, the source propensity
    score is estimated by logistic regression using the baseline covariates.
    In \texttt{R}, this is implemented using
    \texttt{glm(..., family = binomial())}.

    \item[] \textbf{2. Lasso-logistic propensity-score learner.}
    In the high-dimensional sparse linear setting, the source propensity score
    is estimated by lasso logistic regression with cross-validated tuning
    \citep{tibshirani1996regression,friedman2010regularization}. In \texttt{R}, this is implemented using
    \texttt{glmnet::cv.glmnet} \texttt{(..., family = "binomial", alpha = 1)}.

    \item[] \textbf{3. Flexible machine-learning propensity-score learner.}
    In nonlinear settings, the source propensity score is estimated using
    flexible machine-learning methods, such as BART \citep{chipman2010bart}, DL \citep{lecun2015deep}, and \(k\)-nearest neighbors (KNN) \citep{cover1967nearest}. 
\end{enumerate}
For all ATT-weighted Cox estimators in
Subsection~\ref{subsec:Comparison_methods}, estimated source propensity scores
are truncated to lie in \([0.01,0.99]\) before constructing the ATT odds
weights. For MEC-Cox, the source propensity-score fitted values are obtained
out of fold using the same \(K=5\) cross-fitting partition used to construct the
calibration basis. (In the simulations in the main paper, we used \(K=10\)-fold
cross-fitting for MEC-Cox.) For the other comparison methods, cross-fitting is not used.

\subsubsection{Prognostic basis construction} 
We consider two main constructions of the MEC-Cox calibration basis:
\begin{enumerate}[leftmargin=1.5em, label={}, itemsep=0.5em]
    \item[] \textbf{1. Landmark-survival basis.}
For each validation fold \(\mathcal J^{(k)}\), a survival learner for the
external-control outcome distribution is trained using only external-control
subjects in \(\mathcal I_0\cap\mathcal J^{(-k)}\) and then evaluated on all
subjects in \(\mathcal J^{(k)}\). With five landmark times
\(t_1,\ldots,t_5\), the cross-fitted calibration basis is
\[
\widehat h_i
=
\left(
1,
\widehat S_0^{(-k)}(t_1\mid X_i),
\ldots,
\widehat S_0^{(-k)}(t_5\mid X_i)
\right),
\qquad i\in\mathcal J^{(k)}.
\]
The landmark times are chosen as empirical quantiles of the observed event
times among external-control subjects. Specifically, we use equally spaced
quantile levels between \(0.10\) and \(0.90\):
\[
(\tau_1,\ldots,\tau_5)
=
(0.10,0.30,0.50,0.70,0.90),
\qquad
t_\ell
=
\widehat Q_{0,\delta=1}(\tau_\ell),
\quad \ell=1,\ldots,5,
\]
where \(\widehat Q_{0,\delta=1}(\tau)\) denotes the empirical
\(\tau\)-quantile of \(Y_i\) among external-control subjects with
\(\delta_i=1\). Thus, the landmark basis captures predicted control-survival
probabilities across early, middle, and late event-time regions.

We use two survival learners to construct the landmark-survival basis:
\begin{itemize}[leftmargin=1.5em, itemsep=0.5em]
    \item \textbf{Cox survival learner.}
In linear settings, control-survival probabilities are estimated using a
Cox proportional hazards model \citep{cox1972regression}. In \texttt{R}, this
is implemented using \texttt{survival::coxph()}, together with
\texttt{survival::basehaz()} and \texttt{predict(..., type = "lp")}. The
estimated baseline cumulative hazard and the predicted linear predictor are
then combined to obtain predicted survival probabilities at the landmark times.

    \item \textbf{Random survival forest learner.}
In nonlinear settings, control-survival probabilities are estimated using
random survival forests \citep{ishwaran2008random}. In \texttt{R}, this is
implemented using \texttt{ranger::ranger()} \citep{wright2017ranger} with
\texttt{Surv(time, delta) \textasciitilde{} .} and
\texttt{splitrule = "logrank"}. Predicted survival probabilities at the
landmark times are extracted from the survival curves returned by
\texttt{predict(...)}.
\end{itemize}

\item[] \textbf{2. Cox-linear-predictor basis.}
For each validation fold \(\mathcal J^{(k)}\), an external-control Cox model is
trained using only external-control subjects in
\(\mathcal I_0\cap\mathcal J^{(-k)}\). The fitted model is then evaluated on
the validation fold to obtain an out-of-fold control-prognostic linear
predictor. The resulting calibration basis is
\[
\widehat h_i
=
\bigl(1,\widehat\zeta^{(-k)}(X_i)\bigr),
\qquad i\in\mathcal J^{(k)},
\]
where \(\widehat\zeta^{(-k)}(X_i)\) denotes the out-of-fold Cox linear
predictor.

This basis is particularly useful when the baseline covariate vector
$X_i=(X_{i1},\ldots,$ $X_{iM})$ $\in\mathbb R^M$ is high-dimensional and the
prognostic effect is expected to be sparse. In such settings, directly
calibrating on all components of \(X_i\) may be unstable or infeasible because
the number of calibration constraints can be large relative to the sample size.
The Cox-linear-predictor basis instead compresses the potentially
high-dimensional covariate information into a low-dimensional
control-prognostic summary. Therefore, it provides a dimension-reduced
calibration basis that targets outcome-relevant imbalance while keeping the
number of calibration constraints small.

\begin{itemize}[leftmargin=1.5em, itemsep=0.5em]

    \item \textbf{Lasso-Cox linear-predictor learner.}
    In high-dimensional sparse settings, \(\widehat\zeta^{(-k)}(X_i)\) is
    estimated using lasso-Cox regression with cross-validated tuning
    \citep{tibshirani1997lasso,friedman2010regularization}. In \texttt{R},
    this is implemented using
    \texttt{glmnet::cv.glmnet(..., family = "cox", alpha = 1)}. The penalty
    parameter is selected using the one-standard-error rule.
\end{itemize}

\end{enumerate}

For penalized and machine-learning nuisance learners, hyperparameters are
selected using lightweight tuning procedures to keep the Monte Carlo study
computationally feasible. For lasso logistic regression and lasso-Cox
regression, the penalty parameter is selected by cross-validation using the
one-standard-error rule. For flexible propensity-score learners, such as BART,
hyperparameters are selected using a small stratified tuning subset and
validation binary log-loss. For random survival forests, tuning is performed
using a small external-control tuning subset; candidate values of the number
of trees, the number of candidate variables considered at each split, and the
minimum node size are compared using validation concordance. Standard logistic
regression and Cox proportional hazards models are not tuned. Additional implementation details are provided in the accompanying
\texttt{R} simulation code.

\subsubsection{Generator choice}\label{subsubsec:Generator choice}
Given the baseline weights \(\widehat d_i\) and the cross-fitted calibration
basis \(\widehat h_i\), MEC-Cox updates the external-control weights by solving
the Bregman calibration problem described in the main text. We evaluate four Bregman generators: Kullback--Leibler (KL), Hellinger,
empirical likelihood (EL), and R\'enyi with \(\alpha=1/2\). MEC-KL is treated as the primary implementation,
whereas the remaining generators are used for generator-sensitivity analyses.


\subsection{Overview of simulation scenarios and nuisance settings}
\label{subsec:simulation_scenario_summary}

Table~\ref{tab:simulation_scenario_summary} summarizes the simulation scenarios considered in the main paper and Appendix, together with the corresponding nuisance-estimation settings for MEC-Cox.

\begin{table}[h!]
\centering
\scriptsize
\setlength{\tabcolsep}{3pt}
\renewcommand{\arraystretch}{1.08}
\caption{Summary of simulation scenarios and nuisance-estimation settings.}
\label{tab:simulation_scenario_summary}
\begin{tabularx}{\textwidth}{
@{}
>{\RaggedRight\arraybackslash}p{0.07\textwidth}
>{\RaggedRight\arraybackslash}p{0.19\textwidth}
>{\RaggedRight\arraybackslash}p{0.27\textwidth}
>{\RaggedRight\arraybackslash}p{0.21\textwidth}
>{\RaggedRight\arraybackslash}X
@{}}
\toprule
Study & Scenario & Design & Source PS of MEC-Cox & OR / calibration basis of MEC-Cox \\
\midrule
Main
&
Scenario~1: linear source-selection and outcome models
&
\(\kappa_\pi=\kappa_m=0\), \(M=50\);
\(n_1:n_0\in\{1:2,1:3,1:4\}\)
&
Cross-fitted logistic regression (\(K=10\))
&
Cox-based landmark survival basis (\(K=10\)) with KL generator
\\

Main
&
Scenario~2: increasing nonlinearity
&
\((\kappa_\pi,\kappa_m)\in\{(0,0),(1,2),(2,5)\}\);
\(M=10\), \(n_1:n_0=1:4\)
&
Cross-fitted BART (\(K=10\))
&
Cox- or RSF-based landmark survival basis (\(K=10\)) with KL generator
\\

Appx.
&
Sparse-linear high-dimensional study
&
\(\kappa_\pi=\kappa_m=0\);
\(n_1:n_0=1:4\);
\(M\in\{50,200\}\)
&
Cross-fitted lasso logistic regression (\(K=5\))
&
Lasso-Cox linear-predictor basis (\(K=5\)) with KL generator
\\

Appx.
&
Bregman-generator sensitivity
&
Scenario~1 with \(n_1:n_0\in\{1:2,1:4\}\)
&
Cross-fitted logistic regression (\(K=5\))
&
Cox-based landmark survival basis (\(K=5\)); generators compared are KL, EL, Hellinger, and R\'enyi
\\

Appx.
&
Censoring-rate sensitivity
&
Scenario~1 with \(n_1:n_0\in\{1:2,1:4\}\) and heavier censoring;
\(C_i\sim \mathrm{Exp}(0.0016)\)
&
Cross-fitted logistic regression (\(K=5\))
&
Cox-based landmark survival basis (\(K=5\)) with KL generator
\\

Appx.
&
Source PS learner sensitivity
&
\(\kappa_\pi\in\{0.5,1.0,1.5,2.0\}\);
\(\kappa_m\in\{0,1\}\);
\(M=50\), \(n_1:n_0=1:4\)
&
Cross-fitted logistic regression, BART, DL, and KNN (\(K=5\))
&
Cox-based landmark survival basis (\(K=5\)) with KL generator
\\
\bottomrule
\end{tabularx}

\vspace{1mm}
\parbox{\textwidth}{\footnotesize
\emph{Note.} PS denotes source propensity score. OR denotes the outcome-regression
or prognostic learner used to construct the MEC-Cox calibration basis. Main-paper
MEC-Cox simulations use \(K=10\)-fold cross-fitting, whereas supplemental MEC-Cox
simulations use \(K=5\)-fold cross-fitting. For the standard ATT-IPW Cox estimators, namely the naive model-based variance
estimator, the Lin--Wei/Binder robust sandwich estimator, and the Shu corrected sandwich estimator, the source PS model is fitted by logistic regression without
cross-fitting.
}
\end{table}
Although not reported in the paper, we also explored ATT-IPW Cox estimators that
use ML-based source propensity-score estimation together with the
Lin--Wei/Binder robust sandwich variance estimator. However, the performance of
these estimators was highly unstable, often producing extreme coverage behavior,
and therefore we do not report these results in this paper.

\section{Additional simulation studies}\label{sec:Additional simulation experiments}
\subsection{Sparse-linear high-dimensional simulation study}
\label{subsec:simulation_sparse_increasing_dimension}
In this additional simulation study, we examine the performance of MEC-Cox in
sparse linear settings with increasing covariate dimension. Both the source
propensity-score model and the outcome model are sparse linear models,
corresponding to \(\kappa_\pi=\kappa_m=0\). We fix the
treated-to-external-control sample-size ratio at \(n_1:n_0=1:4\) and vary the
covariate dimension over \(M\in\{50,200\}\).

For MEC-Cox, the baseline ATT transport weights are obtained from cross-fitted
lasso logistic regression \citep{tibshirani1996regression}. The calibration
basis is constructed as the two-dimensional cross-fitted lasso-Cox
linear-predictor basis
\begin{align}
\label{eq:lasso-Cox_linear_predictor_basis}
\widehat h_i
=
(1,\widehat\zeta^{(-k)}(X_i)),
\quad
\widehat\zeta^{(-k)}(X_i)=X_i^\top \widehat b^{(-k)}_{\lambda},
\quad
\widehat b^{(-k)}_{\lambda}
=
\arg\min_b
\{-\ell^{(-k)}_{\mathrm{Cox}}(b)+\lambda\|b\|_1\}.
\end{align}
Here, \(\widehat\zeta^{(-k)}(X_i)\) denotes the out-of-fold lasso-Cox linear
predictor trained on the external-control observations in the training fold,
\(\lambda\) is selected by cross-validation, and
\(\ell^{(-k)}_{\mathrm{Cox}}(b)\) denotes the training-fold Cox partial
log-likelihood based on external-control observations
\citep{tibshirani1997lasso}. This construction compresses the high-dimensional
covariate vector into a low-dimensional control-prognostic summary, thereby
keeping the number of calibration constraints small while targeting
outcome-relevant imbalance.

\begin{figure}[h!]
    \centering
    \includegraphics[width=0.9\linewidth]{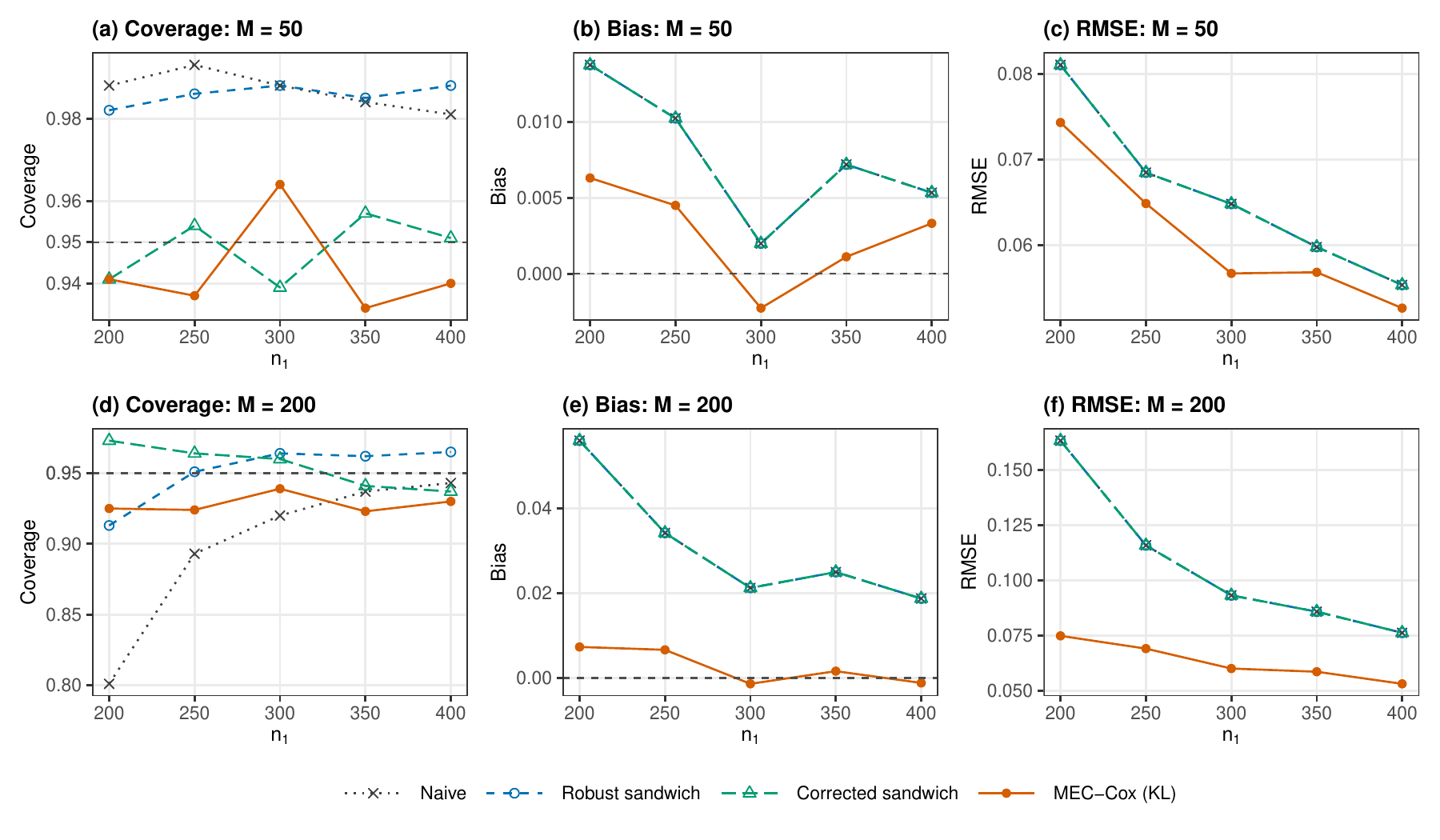}
    \caption{Simulation results with \(\kappa_\pi=\kappa_m=0\) and \(n_1:n_0=1:4\). Panels (a)--(c) and (d)--(f) correspond to \(M=50\) and \(M=200\), respectively. MEC-Cox uses lasso logistic regression for source propensity-score estimation and the Cox linear predictor calibration basis in \eqref{eq:lasso-Cox_linear_predictor_basis}, constructed from a lasso-penalized Cox model trained on the external-control data.}
    \label{fig:supp_sparse_linear_highdim}
\end{figure}

Figure~\ref{fig:supp_sparse_linear_highdim} summarizes the results. Panels
(a)--(c) correspond to the setting with \(M=50\), whereas panels (d)--(f)
correspond to the setting with \(M=200\). Across both covariate dimensions,
MEC-Cox achieves smaller bias and RMSE than the standard
ATT-IPW Cox estimators. The gain is especially pronounced when \(M=200\), where
nuisance estimation is more challenging and regularization is required. The Robust sandwich estimator tends to be
conservative, whereas the Corrected sandwich estimator generally gives coverage
closer to the nominal 95\% level. MEC-Cox maintains reasonably stable coverage
while improving point-estimation accuracy, suggesting that the cross-fitted
lasso-Cox prognostic basis provides useful outcome-relevant information for calibration as the covariate dimension increases.

\subsection{Bregman-generator sensitivity analysis for MEC-Cox}
\label{subsec:sensitivity_generator_choice}

Figure~\ref{fig:sensitivity_analysis_for_generator_choice} presents a
sensitivity analysis for the choice of Bregman generator in MEC-Cox under
Scenario~1 of the main paper, considering
\(n_1:n_0\in\{1:2,1:4\}\). In the main paper, we use the KL generator as the
canonical generator for the ATT marginal hazard-ratio setting. This choice is
natural because the baseline external-control weights are normalized ATT odds
weights, and the KL generator updates these weights through a log-linear
fluctuation around the initial ATT transport weights. Thus, the KL update
preserves the original source-to-target transport structure while applying a
prognostic-balance correction through the MEC calibration constraint.

\begin{figure}[h!]
    \centering
    \includegraphics[width=0.9\linewidth]{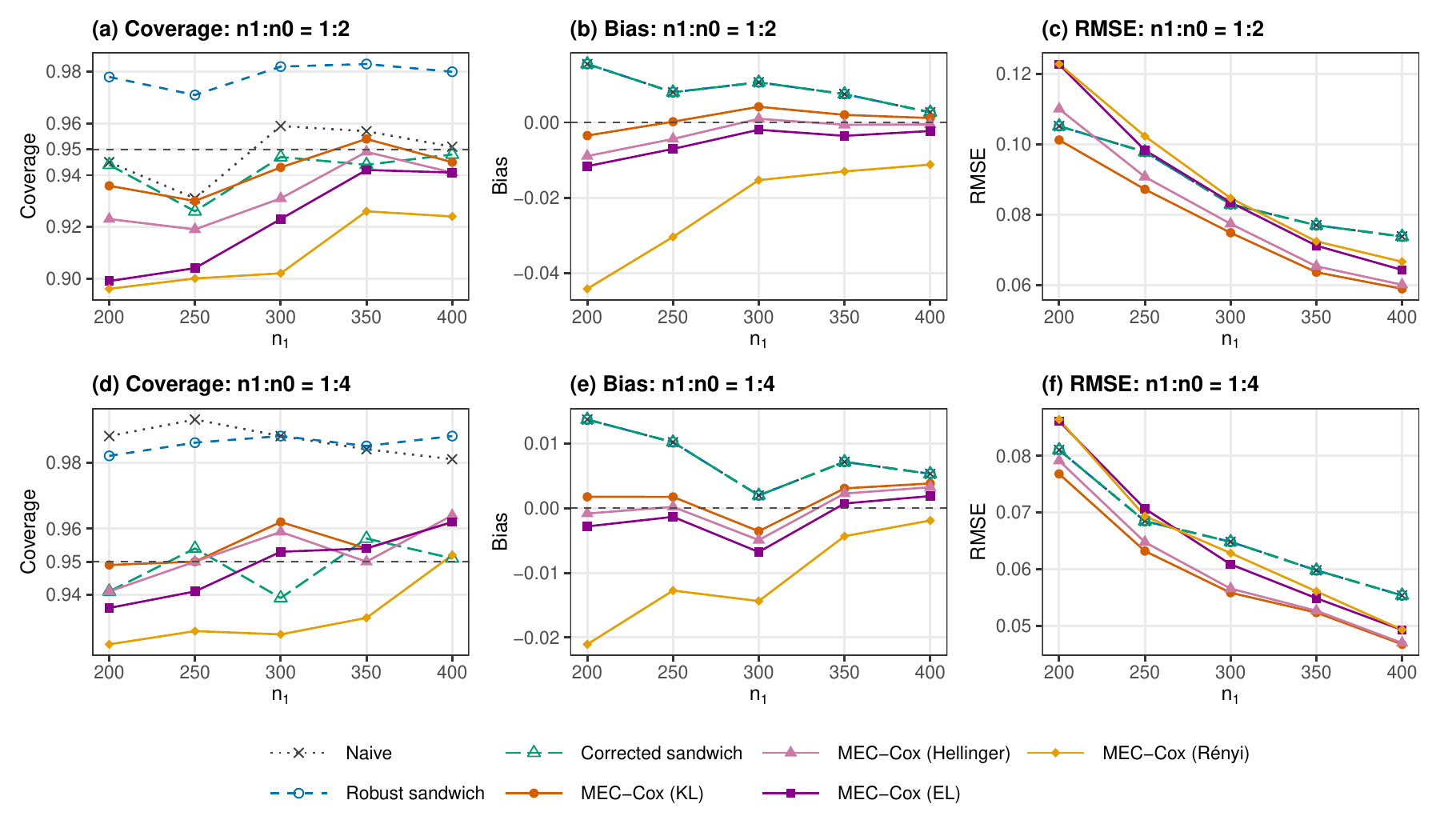}
    \caption{Sensitivity analysis for the choice of MEC-Cox generator in
    Scenario~1 of the
    main paper (\(\kappa_\pi=\kappa_m=0\), \(M=50\), and $n_1:n_0\in \{1:2, 1:4\}$). MEC-Cox uses logistic regression for source propensity-score estimation and the landmark survival calibration basis constructed from a Cox model fitted to the external-control data, with KL, EL, Hellinger, and R\'enyi Bregman generators.}
    \label{fig:sensitivity_analysis_for_generator_choice}
\end{figure}

To assess whether the empirical performance of MEC-Cox depends strongly on this
canonical generator choice, we additionally consider three alternative
Bregman generators: EL, Hellinger, and R\'enyi. These
generators provide valid Bregman calibration updates, but they perturb the
baseline ATT transport weights on different scales rather than on the
log-odds transport scale. Therefore, they are treated here as sensitivity
analyses rather than as primary implementations.

The results show that MEC-Cox with the KL generator performs well across both
external-control sample-size ratios, \(n_1:n_0=1:2\) and \(n_1:n_0=1:4\).
Its empirical coverage is close to the nominal 95\% level, especially as
\(n_1\) increases, and its bias and RMSE are consistently small. The EL and
Hellinger generators produce broadly similar qualitative patterns, suggesting
that the finite-sample behavior of MEC-Cox is not driven by a fragile or
idiosyncratic choice of generator. The R\'enyi generator also follows the same
overall decreasing RMSE trend, although it can be slightly less favorable in
some smaller-sample settings, particularly in terms of bias and coverage.

Overall, this sensitivity analysis supports the use of the KL generator as the
default and most interpretable choice for MEC-Cox in the ATT marginal
hazard-ratio setting. The alternative generators lead to broadly comparable
patterns, but they do not provide a clear empirical advantage over KL. Hence,
the main paper focuses on MEC-Cox with the KL generator, while the remaining
generators are reported here to demonstrate robustness of the proposed
calibration framework.

\subsection{Censoring-rate sensitivity analysis}
\label{subsec:sensitivity_censoring_choice}
Throughout the paper, the validity of the weighted Cox estimating equation is
studied under source-specific independent censoring. Under this condition,
censoring affects the amount of observed time-to-event information but does not
induce bias in the Cox estimating equation when the source-selection and outcome
models are correctly specified. Therefore, it is useful to examine whether the
finite-sample performance of the competing methods is stable when the censoring
rate is increased.

Figure~\ref{fig:sensitivity_analysis_for_censoring_rate} presents a sensitivity
analysis based on Scenario~1 of the main paper, considering
\(n_1:n_0\in\{1:2,1:4\}\). The data-generating mechanism is
identical to Scenario~1, except that the censoring time is generated from $C_i\sim \operatorname{Exp}(0.0016),$ whereas the default setting uses $C_i\sim \operatorname{Exp}(0.0008).$ Thus, the censoring hazard is doubled relative to the default setting, leading
to shorter censoring times and a larger amount of censoring. This creates a more
challenging finite-sample setting because fewer event times are observed,
although the independent censoring assumption remains correctly specified.

\begin{figure}[h!]
    \centering
    \includegraphics[width=0.9\linewidth]{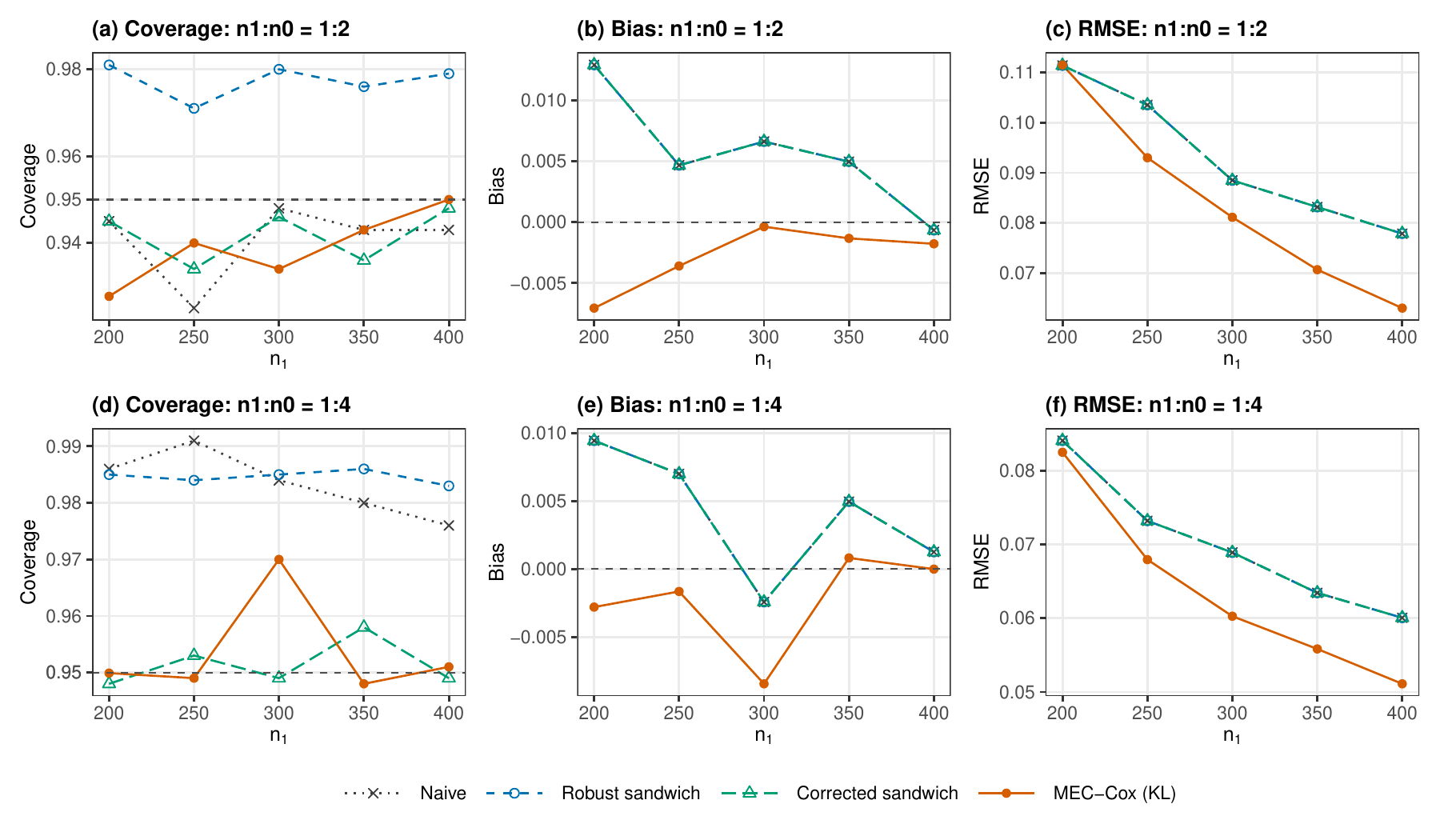}
    \caption{Sensitivity analysis for the censoring rate in Scenario~1 of the
    main paper (\(\kappa_\pi=\kappa_m=0\), \(M=50\), and $n_1:n_0\in \{1:2, 1:4\}$). The only difference from the default Scenario~1 setting in the main paper is that
    censoring times are generated from \(C_i\sim\operatorname{Exp}(0.0016)\),
    instead of the default \(C_i\sim\operatorname{Exp}(0.0008)\). MEC-Cox uses logistic regression for source propensity-score estimation and the landmark survival calibration basis constructed from a Cox model fitted to the external-control data, using the KL generator.}
    \label{fig:sensitivity_analysis_for_censoring_rate}
\end{figure}

The results show that all methods continue to behave reasonably well under this
heavier censoring setting. Because Scenario~1 uses correctly specified
source-selection and outcome models, the bias remains small across methods and
decreases toward zero as \(n_1\) increases. As expected, the RMSE values are
somewhat larger than those under the default censoring rate, reflecting the
information loss caused by additional censoring. Nevertheless, the RMSE decreases
monotonically with increasing sample size for both sample-size ratios
\(n_1:n_0=1:2\) and \(n_1:n_0=1:4\).

The coverage results also remain stable. The robust sandwich estimator tends to
be conservative, whereas the corrected sandwich estimator and MEC-Cox with the
KL generator generally achieve coverage close to the nominal 95\% level. The
MEC-Cox estimator retains small bias and favorable RMSE across the considered
sample sizes, indicating that the proposed calibration step remains stable under
increased censoring. Overall, this sensitivity analysis suggests that the main
findings of Scenario~1 in the main paper are not driven by the particular censoring rate used in the default simulation setting.

\subsection{Propensity-score learner sensitivity analyses for MEC-Cox}
\label{subsec:sensitivity_ps_ml_choice}
We further examined the sensitivity of MEC-Cox to the choice of machine-learning
method used for source propensity-score estimation. This experiment considered
both linear and nonlinear outcome-model settings while varying the degree of
nonlinearity in the true source-selection mechanism from mild to strong.
Specifically, the propensity-score nonlinearity parameter \(\kappa_\pi\) in
\eqref{eq:true_log_ps_model} was varied over \(0.5,1.0,1.5,\) and \(2.0\),
and the outcome-model nonlinearity parameter was set to
\(\kappa_m \in \{0,1\}\) in \eqref{eq:log_hazard_prognostic_function}. We fixed
\(M=50\) covariates and \(n_1:n_0=1:4\). This experiment evaluates how sensitive
MEC-Cox is to the propensity-score learner under increasing source-selection
model complexity, while keeping the remaining components of the MEC-Cox implementation fixed within each outcome-model setting.

We compared MEC-Cox using four source propensity-score estimators: logistic regression, BART \citep{chipman2010bart}, DL \citep{lecun2015deep}, and KNN \citep{cover1967nearest}. For all MEC-Cox variants, the KL Bregman generator was used, and the landmark survival calibration basis was constructed from a Cox model fitted to the external-control data. For DL, we used a feedforward multilayer perceptron with two hidden layers of sizes 32 and 16, ReLU activation, dropout regularization, and early stopping.

\paragraph{Experiment 1: Increasing source-selection nonlinearity under a linear true outcome model}

Figure~\ref{fig:sensitivity_analysis_for_PS_ML_PS_nonlinear_OR_linear}
reports the sensitivity analysis when the true outcome/prognostic model is
linear, with \(\kappa_m=0\). As the true source-selection mechanism becomes
more nonlinear, with \(\kappa_\pi\) increasing from \(0.5\) to \(2.0\) from the
upper panels (a)--(c) to the lower panels (j)--(l), the standard ATT-IPW Cox
estimators based on logistic propensity-score weights become increasingly
sensitive to propensity-score misspecification, leading to larger bias and RMSE.

\begin{figure}[h!]
    \centering
    \includegraphics[width=0.9\linewidth]{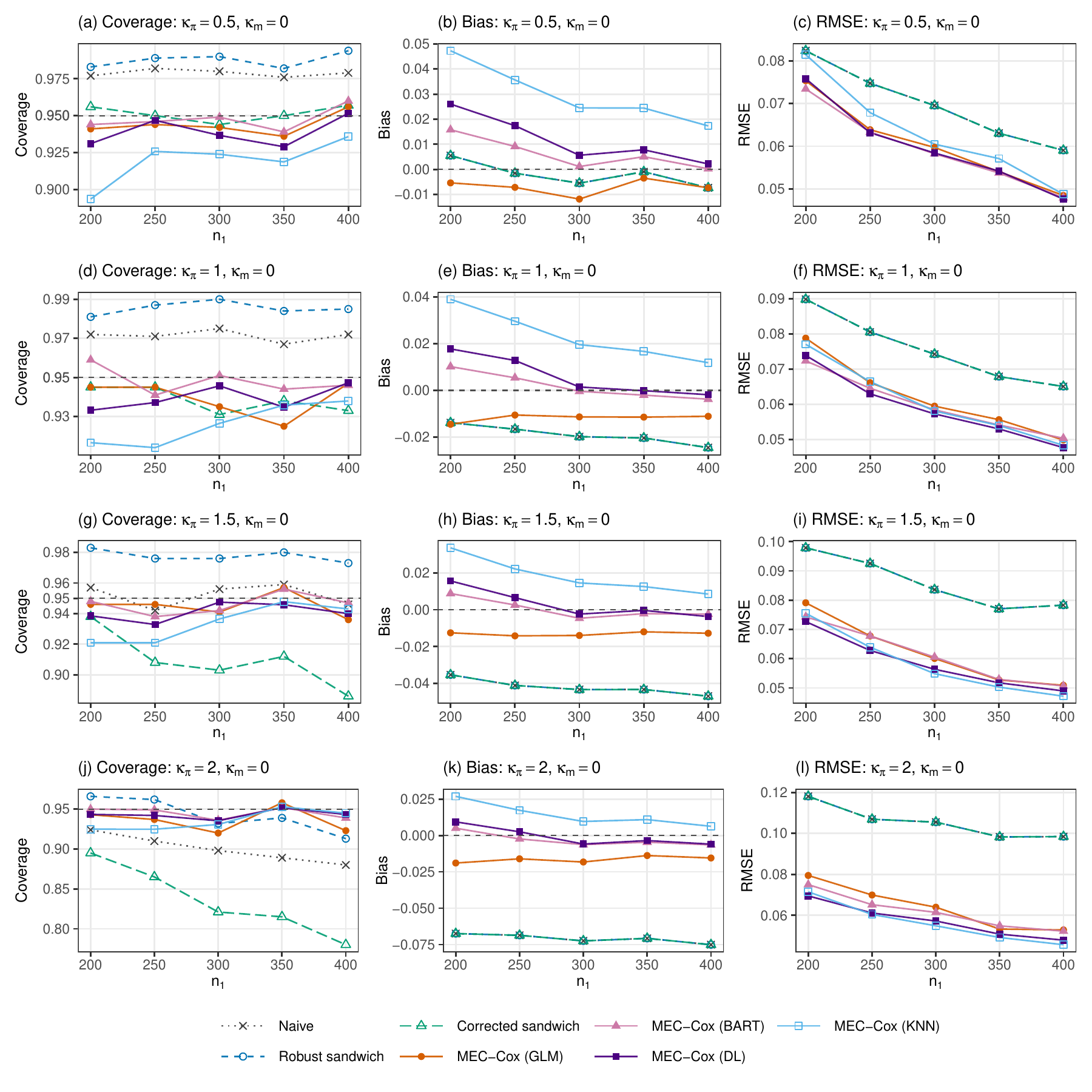}
    \caption{Sensitivity analysis for the choice of source propensity-score machine-learning method in MEC-Cox. The panels report performance metrics with \(M=50\) and \(n_1:n_0=1:4\). Rows correspond to increasing degrees of nonlinearity in the true source-selection model, with the propensity-score nonlinearity parameter \(\kappa_\pi\) in \eqref{eq:true_log_ps_model} set to \(0.5,1.0,1.5,\) and \(2.0\). The outcome/prognostic model is kept linear, with \(\kappa_m=0\) in \eqref{eq:log_hazard_prognostic_function}. MEC-Cox uses the KL Bregman generator and the landmark survival calibration basis constructed from a Cox model fitted to the external-control data, while varying the source propensity-score estimator among logistic regression, BART, DL, and KNN.}
    \label{fig:sensitivity_analysis_for_PS_ML_PS_nonlinear_OR_linear}
\end{figure}

In contrast, MEC-Cox remains comparatively stable across the range of
\(\kappa_\pi\) values. MEC-Cox with flexible propensity-score learners generally
reduces bias and improves RMSE, especially under stronger source-selection
nonlinearity. Among the MEC-Cox variants, those using BART and DL for source
propensity-score estimation appear to provide the most stable overall
performance, maintaining relatively small bias and RMSE while keeping coverage
close to the nominal level across most settings. KNN also achieves competitive
bias and RMSE, and its performance improves as \(n_1\) increases; however, its
coverage tends to be below the nominal level in this experiment, particularly in
smaller samples. Notably, MEC-Cox with logistic regression for source
propensity-score estimation also performs substantially better than the standard
ATT-IPW Cox estimators.

These findings suggest that the landmark survival calibration basis, constructed
from a Cox model fitted to the external-control data in this linear outcome
setting, helps mitigate residual prognostic imbalance after propensity-score
weighting. At the same time, flexible source propensity-score learners,
especially BART and DL, provide additional robustness against nonlinear
source-selection mechanisms.

\paragraph{Experiment 2: Increasing source-selection nonlinearity under a nonlinear true outcome model}

Figure~\ref{fig:sensitivity_analysis_for_PS_ML_PS_nonlinear_OR_nonlinear}
reports the corresponding sensitivity analysis when the true outcome/prognostic
model is nonlinear, with \(\kappa_m=1\). Overall, the performance patterns are
broadly similar to those observed in Experiment 1 in
Figure~\ref{fig:sensitivity_analysis_for_PS_ML_PS_nonlinear_OR_linear},
although the RMSE values are slightly larger. This indicates that the nonlinear
outcome/prognostic model creates a more challenging setting by adding complexity
to the survival outcome mechanism in addition to the nonlinear source-selection
mechanism.

\begin{figure}[h!]
    \centering
    \includegraphics[width=0.9\linewidth]{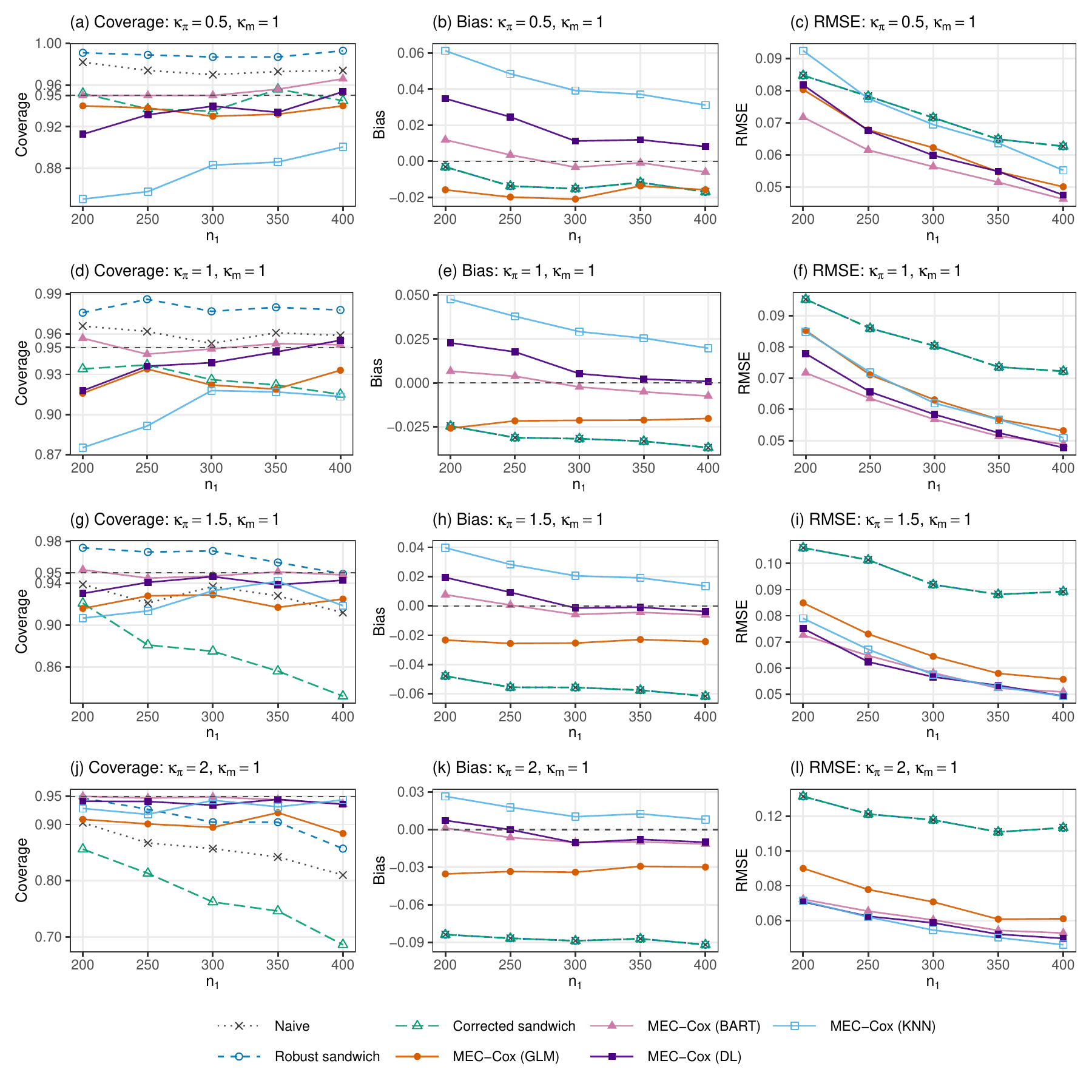}
    \caption{Sensitivity analysis for the choice of source propensity-score machine-learning method in MEC-Cox. The panels report performance metrics with \(M=50\) and \(n_1:n_0=1:4\). Rows correspond to increasing degrees of nonlinearity in the true source-selection model, with the propensity-score nonlinearity parameter \(\kappa_\pi\) in \eqref{eq:true_log_ps_model} set to \(0.5,1.0,1.5,\) and \(2.0\). The outcome/prognostic model is kept nonlinear, with \(\kappa_m=1\) in \eqref{eq:log_hazard_prognostic_function}. MEC-Cox uses the KL Bregman generator and the landmark survival calibration basis, while varying the source propensity-score estimator among logistic regression, BART, DL, and KNN.}
    \label{fig:sensitivity_analysis_for_PS_ML_PS_nonlinear_OR_nonlinear}
\end{figure}

More specifically, as the true source-selection mechanism becomes more
nonlinear, with \(\kappa_\pi\) increasing from \(0.5\) to \(2.0\) from the upper
panels (a)--(c) to the lower panels (j)--(l), the performance of the standard
ATT-IPW Cox estimators based on logistic propensity-score weights deteriorates.
The ATT-IPW Cox point estimates exhibit increasingly negative bias and larger
RMSE under stronger source-selection nonlinearity, and the corresponding
confidence intervals show substantial undercoverage when \(\kappa_\pi=2.0\).
The naive, robust sandwich, and corrected sandwich variance implementations all
become less reliable as \(\kappa_\pi\) increases.

In contrast, MEC-Cox again remains more stable than the standard ATT-IPW Cox
estimators across the range of \(\kappa_\pi\) values. Consistent with
Experiment 1, MEC-Cox variants using BART and DL for source propensity-score
estimation show the most stable overall performance, with relatively small bias
and RMSE and coverage closer to the nominal level. KNN becomes more competitive
as \(n_1\) increases, but is less stable in smaller samples. MEC-Cox with
logistic regression also improves over the standard ATT-IPW Cox estimators,
although it is less effective than BART and DL under strong source-selection
nonlinearity. Overall, these results reinforce that flexible source
propensity-score learners can improve the robustness of MEC-Cox in more complex
nonlinear outcome settings.

\newpage
\clearpage
\vskip 0.2in
\bibliography{ref1}

\end{document}